\documentclass[smallcondensed]{svjour3}

\usepackage{silence}
\usepackage{amsmath}
\usepackage{graphicx}
\usepackage{subcaption}
\usepackage{url}
\usepackage{hyperref}
\hypersetup{
 colorlinks=true,
 linkcolor=darkgreen,
 filecolor=darkgreen,
 citecolor = darkgreen,      
 urlcolor=cyan,
}

\usepackage{etoolbox}
\usepackage{amsmath,amssymb,amsfonts}

\usepackage{multirow}
\usepackage{multicol}
\usepackage{color,soul}
\usepackage[svgnames]{xcolor} % Enabling colors by their 'svgnames'
\usepackage{svg}
\usepackage{algorithm,algorithmicx,algpseudocode}
\setstcolor{red} %strikethrough in red
\usepackage{natbib}

\usepackage{tabularx}
\usepackage{makecell}

\usepackage{pdflscape} % to rotate landscape pages, use it with \begin{landscape}
\usepackage{pdfpages}
\usepackage{rotating}

\usepackage{threeparttable}

\allowdisplaybreaks

\usepackage{booktabs}
\usepackage{tikz}
\usetikzlibrary{arrows,positioning}
\usepackage{lipsum, adjustbox}
\usepackage{setspace}
\usepackage{xcolor}

\definecolor{tab-blue}{HTML}{1f77b4} % tab:blue in matplotlib
\definecolor{tab-orange}{HTML}{ff7f0e} % tab:orange in matplotlib
\definecolor{tab-green}{HTML}{2ca02c} % tab:green in matplotlib

\newcommand{\mca}{\mathcal}
\newcommand{\hhsp}{\-\hspace}

\newcommand*\colvec[3][]{
    \begin{bmatrix}\ifx\relax#1\relax\else#1\\\fi#2\\#3\end{bmatrix}
}

\newcommand*\boxSizeOfMax[1]{\makebox[\widthof{max}][c]{#1}}
\newcommand{\sthat}{\boxSizeOfMax{s.t.}}

\definecolor{wisconsin-red}{rgb}{0.6,0,0}
\definecolor{darkgreen}{rgb}{0.2,0.6,0.2}
\definecolor{maroon}{rgb}{0.5, 0.0, 0.0}
\definecolor{violet}{rgb}{0.75, 0.0, 1.0}
\definecolor{lightgray}{gray}{0.9}
\definecolor{navyblue}{rgb}{0.0, 0.0, 0.35}
\definecolor{darkmidnightblue}{rgb}{0.0, 0.2, 0.4}
\definecolor{Gray}{gray}{0.75}
\definecolor{darkgreen}{rgb}{0,0.5,0}

\newcommand{\policy}{\mu}
\newcommand{\setPolicy}{\Upsilon}
\newcommand{\bePt}{\pi}
\newcommand{\vectBePt}{\boldsymbol{\bePt}}
\newcommand{\timeHorizon}{T}

\newcommand{\genericState}{X}
\newcommand{\genericAction}{Y}

\newcommand{\setTime}{\mca{T}}
\newcommand{\setState}{\mca{S}}
\newcommand{\setAction}{\mca{A}}
\newcommand{\setObs}{\mca{O}}
\newcommand{\beSimp}{\operatorname{\Pi}(\setState)}
\newcommand{\setTRP}{\mca{P}}
\newcommand{\setOBSModel}{\mca{Z}}
\newcommand{\setReward}{\mca{R}}
\newcommand{\rewSalvage}{\phi}
\newcommand{\rewTerminal}{\varphi}

\newcommand{\indAction}{a}
\newcommand{\indAcWait}{\mathbb{W}}
\newcommand{\indAcMammo}{\mathbb{M}}
\newcommand{\indAcMMRI}{\mathbb{M}\&\mathbb{R}}
\newcommand{\indAcMUS}{\mathbb{M}\&\mathbb{U}}

\newcommand{\disuScr}{\Delta^{Scr}}
\newcommand{\disuTP}{\Delta^{TP}}
\newcommand{\disuFP}{\Delta^{FP}}

\def\Plus{\texttt{+}}
\def\Minus{\texttt{-}}

\newcommand{\indTime}{t}
\newcommand{\indState}{i}
\newcommand{\indStateNew}{j}
\newcommand{\indStateHealthy}{\texttt{0}}

\newcommand{\indStateLocalized}{\texttt{2}}
\newcommand{\indStateRegional}{\texttt{3}}

\newcommand{\indObs}{\theta}
\newcommand{\indObsNegative}{\indObs^{\Minus}}
\newcommand{\indObsPositive}{\indObs^{\Plus}}
\newcommand{\indGrid}{k}
\newcommand{\indGridNew}{\ell}

\newcommand{\parfRew}{r}
\newcommand{\parfRewGeneric}{\omega}
\newcommand{\parfLTRiskGeneric}{\gamma}
\newcommand{\parfCost}{c}
\newcommand{\parfObs}{z}
\newcommand{\parfTrp}{p}
\newcommand{\parfGridTr}{f}

\newcommand{\parfDeathCaAfterScreen}{q}
\newcommand{\parBudgetLim}{C}

\newcommand{\funcOptV}{Q}
\newcommand{\funcApxV}{\hat{\mca{\funcOptV}}}
\newcommand{\varDualLP}{{x}}
\newcommand{\parBeDistn}{\delta}
\newcommand{\varDetPol}{\nu}
\newcommand{\scaleFactor}{\varsigma}

\newcommand{\setGridIndex}{\mca{K}}
\newcommand{\setGrid}{\mca{G}}
\newcommand{\grPt}{b}
\newcommand{\ugrPt}{\grPt'}
\newcommand{\vectGrPt}{\boldsymbol{\grPt}}
\newcommand{\ugrPtvect}{\boldsymbol{\ugrPt}}
\newcommand{\resoVal}{\rho}
\newcommand{\resoThreshold}{\psi}
\newcommand{\vectResoThreshold}{\boldsymbol{\resoThreshold}}
\newcommand{\vectResoVal}{\boldsymbol{\resoVal}}

\newcommand{\indResoVal}{\xi}
\newcommand{\setAllowedBeliefComponents}{\mathcal{B}}

\newcommand{\parDiscount}{\lambda}

\newcommand{\objweight}{w}
\begin{document}\sloppy
\title{A multi-objective constrained partially observable Markov decision process model for breast cancer screening}
\titlerunning{A multi-objective CPOMDP model for breast cancer screening}
\author{Robert K. Helmeczi \and Can Kavaklioglu\and Mucahit Cevik  \and Davood Pirayesh Neghab}
% \author{author1\and author2}

\institute{
    Robert K. Helmeczi \at Toronto Metropolitan University, Toronto, Canada
    \and
    Can Kavaklioglu \at Toronto Metropolitan University, Toronto, Canada
    \and
    Mucahit Cevik \at Toronto Metropolitan University, Toronto, Canada\\ \email{mcevik@torontomu.ca}
    \and
    Davood Pirayesh Neghab \at Toronto Metropolitan University, Toronto, Canada
}

\date{Received: date / Accepted: date}

\maketitle

\vspace*{-0.3cm}
\begin{abstract}
Breast cancer is a common and deadly disease, but it is often curable when diagnosed early. 
While most countries have large-scale screening programs, there is no consensus on a single globally accepted guideline for breast cancer screening. 
The complex nature of the disease; the limited availability of screening methods such as mammography, magnetic resonance imaging (MRI), and ultrasound; and public health policies all factor into the development of screening policies. 
Resource availability concerns necessitate the design of policies which conform to a budget, a problem which can be modelled as a constrained partially observable Markov decision process (CPOMDP). 
In this study, we propose a multi-objective CPOMDP model for breast cancer screening which allows for supplemental screening methods to accompany mammography.
% Each screening action has a unique impact on QALYs and LBCMR, as well as a unique cost.
The model has two objectives: maximize the quality-adjusted life years (QALYs) and minimize lifetime breast cancer mortality risk (LBCMR).
We identify the Pareto frontier of optimal solutions for average and high-risk patients at different budget levels, which can be used by decision-makers to set policies in practice.
We find that the policies obtained by using a weighted objective are able to generate well-balanced QALYs and LBCMR values. 
In contrast, the single-objective models generally sacrifice a substantial amount in terms of QALYs/LBCMR for a minimal gain in LBCMR/QALYs.
% Additionally, we consider an expanded action space which allows for supplemental screening methods to accompany mammography. 
% Each screening action has a unique impact on QALYs and LBCMR, as well as a unique cost.
Additionally, our results show that, with the baseline cost values for supplemental screenings as well as the additional disutility that they incur, they are rarely recommended in CPOMDP policies, especially in a budget-constrained setting.
A sensitivity analysis reveals the thresholds on cost and disutility values at which supplemental screenings become advantageous to prescribe.
\end{abstract}


%%%%%%%%%%%%%%%%%%%%%%%%%%%%%%%%%%%%%%%%%%%%%%%%%%%%%%%
%%%%%%%%%%%%%%%%%%%%%%%%%%%%%%%%%%%%%%%%%%%%%%%%%%%%%%%
\section{Introduction}
%%%%%%%%%%%%%%%%%%%%%%%%%%%%%%%%%%%%%%%%%%%%%%%%%%%%%%%
%%%%%%%%%%%%%%%%%%%%%%%%%%%%%%%%%%%%%%%%%%%%%%%%%%%%%%%
Breast cancer is the most commonly occurring cancer among women and is the second leading cause of cancer death \citep{acs2019breast}. 
\citet{acs2020cancer} estimated that among US women, there would be 285,360 new deaths by cancer in 2020, 14.9 \% of which would be from breast cancer. 
While the incidence of breast cancer has increased by 30\% (62\% of which were localized breast cancers) from 1975 to 2010, its associated mortality rate has declined by 34\% \citep{narod2015mortality}.
This decline can be attributed to improvements in treatment such as surgery and radiation for non-metastatic breast cancers \citep{waks2019breast}. 
However, the pace of this decline has slowed from an annual decrease of 1.9\% during 1998 through 2011 to 1.3\% during 2011 through 2017 \citep{acs2019breast}. 
A possible reason for this is that, despite advancements in non-metastatic breast cancer treatments, treatments for metastatic cancers are rarely permanent \citep{waks2019breast}. 
Furthermore, when a patient self-diagnoses breast cancer, the cancer is typically already metastasized. 
Thus, early diagnosis could be crucial to accelerate the decline in mortality rates for this disease.

Various screening methods are used to diagnose breast cancer, with the most common approach being mammography. 
Mammography is capable of detecting breast cancer years before physical symptoms develop, and so its usage contributes to lower mortality rates. 
The five-year breast cancer survival rate in countries with population-based screening programs (e.g., North America, Sweden, and Japan) is over 80\%, while it is less than 40\% in low-income countries that do not have established screening programs \citep{coleman2008cancer}. 
Moreover, breast cancers detected with mammography (as compared to physical examination) are generally non-metastatic and more likely to be treated with breast surgery (56\% compared to 32\%), and less likely to receive adjunct chemotherapy (28\% compared to 56\%), resulting in lower overall costs of treatment \citep{barth2005detection}. 

Despite its effectiveness in early diagnosis and lowering mortality rates, mammography is not without its downsides. Firstly, mammography screenings are expensive: it was estimated that the total cost of mammography screenings in the United States in 2010 was \$7.8 billion~\citep{o2014aggregate}.  
Secondly, a significant percentage of mammography screenings result in false positives. 
\citet{fuller2015breast} conducted a simulation study which showed that out of 10,000 women undergoing annual mammography screenings, 4,970 to 6,130 of them receive at least one false positive result over 10 years, with 700 to 980 resulting in an unnecessary biopsy.
Another study estimates that the cumulative risk of a false positive after 10 mammograms is 49.1\%, as compared to 22.3\% after 10 clinical breast exams \citep{elmore1998ten}. 
False positives result in unnecessary anxiety for patients and additional costs for follow-up procedures. 

While there have been advancements in mammography technology, about 20\% of mammography screenings still result in false negatives, which are much more common among women with dense breasts~\citep{acs2019early}.
Mammographies are also less successful at detecting breast cancer in women with BRCA1 and BRCA2 mutations \citep{tilanus2002brca1}. 
Supplemental tests with a higher sensitivity compared to mammography, such as MRI, can be used to screen women with a higher risk of developing breast cancer. 
For instance, \citet{acs2019early} recommends women with a lifetime breast cancer risk of 20\% or more undergo supplemental MRI screening in the case of negative mammography results. 
However, these supplemental tests incur an additional cost and increase the likelihood of false positives, making them a potentially undesirable alternative.

Research into the design and implementation of breast cancer policies often focuses on quality-adjusted life years (QALYs). 
This metric assigns higher rewards to policies that extend the length of a patient's life while also accounting for the negative impacts associated with screening. 
Specifically, screening procedures can be time-consuming, and the risk of false positives means that patients subjected to an extensive screening regimen can also undergo periods of anxiety awaiting follow-up test results, often from invasive procedures like biopsies. 
In contrast, lifetime breast cancer mortality risk (LBCMR) is a metric that considers only the likelihood of a woman developing breast cancer and dying as a result. 
Unlike QALYs, this metric does not account for the negative effects of undergoing frequent screenings. 
In practice, policies seeking to maximize QALYs may lead to substantially higher LBCMR values. Similarly, policies which try to minimize LBCMR may decrease QALYs significantly for only a small improvement in LBCMR.
By considering both QALYs and LBCMR together, and assessing the trade-offs in between, medical professionals are afforded greater flexibility in their decision-making process, and are better able to create personalized policies for patients based on their risk tolerance.

In summary, the risks associated with false positives and false negatives, as well as the prohibitive cost of screenings, complicate the task of designing a policy for the breast cancer screening problem. 
Additionally, the risk of developing breast cancer is not equal among women, as factors including genetic mutations, family history, and breast density can all considerably increase the risk of developing breast cancer~\citep{acs2019early, sandikci2018screening}.
To address these issues, we formulate a multi-objective CPOMDP model resulting in a personalized, budget-constrained screening policy that maximizes a patient's QALYs while minimizing her LBCMR.

Our paper's contribution can be viewed from both methodological and practical standpoints. 
In terms of methodology, we extend the CPOMDP framework proposed by \citet{Cevik2018} by incorporating multiple objectives to assess the trade-off between QALYs and LBCMR. 
In addition, we empirically show that CPOMDP algorithm run times can be improved substantially by identifying the relevant belief states at each decision epoch.
From a practical point of view, we propose a more comprehensive modelling framework to increase early detection rates and QALYs while reducing LBCMR.
We accomplish this by considering the impacts of generated policies on QALYs and LBCMR simultaneously. 
The resulting policies are found to perform well in both regards, often leading to QALYs and LBCMR values that do not deviate significantly from the corresponding outcomes of single objective models. 
We also consider supplemental tests in addition to mammography and observe their impact on QALYs and LBCMR.
Our numerical results show that a limited amount of supplemental tests are recommended due to their associated costs and disutilities.
% expensive or incur too high of a disutility to be used over mammography alone.
We also perform a sensitivity analysis on both the disutility of the supplemental tests as well as the cost of screenings and identify thresholds at which supplemental screenings become favourable.
In addition, we show the Pareto frontier over the QALY maximization and LBCMR minimization objectives, which help understand the trade-offs between these two important modelling objectives.

The rest of the paper is organized as follows. 
Section~\ref{sec:litReview} provides a brief discussion of the literature surrounding cancer screening and similar modelling frameworks. 
Section~\ref{sec:Methodology} provides the mathematical background for the CPOMDP methodology, which is followed up by our proposed model. 
The results are then presented in Section~\ref{sec:Results}, followed by Section~\ref{sec:Conclusion} where we discuss our findings and future research directions.

%%%%%%%%%%%%%%%%%%%%%%%%%%%%%%%%%%%%%%%%%%%%%%%%%%%%%%%%%%%%
%%%%%%%%%%%%%%%%%%%%%%%%%%%%%%%%%%%%%%%%%%%%%%%%%%%%%%%%%%%%
\section{Literature Review}\label{sec:litReview}
%%%%%%%%%%%%%%%%%%%%%%%%%%%%%%%%%%%%%%%%%%%%%%%%%%%%%%%%%%%%
%%%%%%%%%%%%%%%%%%%%%%%%%%%%%%%%%%%%%%%%%%%%%%%%%%%%%%%%%%%%

Different methodologies have been employed to solve cancer screening problems including simulation-based \citep{Mandelblatt2016, sprague2015benefits} and optimization-based techniques~\citep{saville2019operational}.
Several studies have been conducted to determine efficient cancer-screening policies by evaluating their cost-effectiveness using Markov models~\citep{le2016structural, ralaidovy2018cost, gopalappa2018two, bansal2020analysis}. 
\citet{tollens2021cost} assessed the rate of false positive outcomes in two rounds of screening of women with extremely dense breasts and evaluated their impact on cost-effectiveness.
\citet{kaiser2021cost, kaiser2021impact} employed Markov models to evaluate cumulative costs, QALYs, false positive, and false negative results of breast cancer screening and conducted deterministic and probabilistic analyses to test the model stability.

Some other studies developed simulation-based analysis using partially observable Markov chain (POMC) models.
For example, \citet{Maillart2008} and \citet{MADADI2015} assessed breast cancer screening policies where the former focuses on balancing cost-effectiveness with lifetime mortality risk, and the latter focuses on the impact of patient adherence to screening recommendations. 
Similarly, \citet{molani2019partially} developed POMC models to examine the risks of various screening policies while considering a patient's adherence behavior. 
Later, \citet{molani2020investigating} formulated a POMC model to study the effectiveness of supplemental screening (e.g., ultrasound) while incorporating the impact of radiologists' bias on patients' screening outcomes. 
POMC models have also been used for other cancer types. 
For instance, \citet{li2014using} developed POMC models to evaluate colorectal cancer screening policies among others. 
While these studies comprehensively evaluate the impacts of various factors such as screening age and intervals, they do not attempt to find optimal policies. 
% Hence, we explore optimization-based studies in the following. 

A strand of the literature develops optimization-based techniques to find policies for various cancer screening problems. 
For instance, \citet{akhavan2017markov} optimized cervical cancer screening decisions by modelling the problem as a Markov decision process (MDP), accounting for several factors such as a patient's age and screening results. 
Similarly, \citet{chhatwal2010optimal} and \citet{alagoz2013optimal} optimized the post-mammography biopsy decisions for breast cancer using MDPs. 
\citet{vargas2015optimal} incorporated the impact of over-treatment and the potential delay in cancer detection into an MDP model of the breast cancer screening problem. 
\citet{ayvaci2018preference} developed a stochastic modelling framework using MDPs to optimize risk-sensitive diagnostic decisions after a mammography exam by considering the variations in risk preferences of the individuals. 
\citet{tunc2018new} optimized breast cancer diagnostic decisions to reduce the overdiagnosis costs in a patient's lifetime based on cancer subtypes using a large-scale MDP model with many finite states. 
\citet{imani2020markov} designed an MDP model to optimize the costs and the QALYs for breast cancer treatment plans considering prophylactic surgery, radiation therapy, chemotherapy, and their combinations.  
On the other hand, in most cases, a patient's state is not fully observable, and therefore partially observable MDPs (POMDPs), while more difficult to solve due to a generally intractable state space, are more suitable for modelling the cancer screening problem since they capture the uncertainty related to the imperfect state information. POMDPs have been used to develop policies for a wide variety of cancer screening problems such as prostate~\citep{zhang2012optimization, zhang2012optimizationb, zhang2018partially, li2022optimizing}, colorectal~\citep{erenay2014optimizing}, liver~\citep{chen2018optimal}, cervical~\citep{ebadi2021personalized}, and lung cancers~\citep{petousis2019optimizing}.
%\citep{zhang2022diagnostic}.

Numerous studies have specifically investigated the breast cancer screening problem.
\citet{ayer2012} proposed a finite-horizon POMDP to identify policies for patients based on their risk factors. 
\citet{ayer2015inverse} and \citet{ayer2016heterogeneity} formulated the breast cancer screening problem using POMDPs, where the former focuses on finding the
optimal screening policy for given sensitivity and specificity values, and the latter incorporates both the patient's adherence behavior and their breast cancer risk into the decision process. \citet{otten2017stratified} and \citet{witteveen2018risk} proposed a POMDP model for breast cancer and considered the personal risk of developing cancer to investigate the resource allocation for optimal and personal follow-up. 
Additionally, \citet{sandikci2018screening} evaluated the impact of supplemental screenings by sequentially considering the screening decisions. 
That is, their model first determines whether the patient should receive mammography, and the supplemental screening decision is then made according to the mammography outcome.
In a recent study, \citet{hajjar2022personalized} developed a model to identify the optimal screening decisions for an index disease (e.g., breast cancer) while incorporating a chronic condition (e.g. diabetes). 
In particular, they provided personalized breast cancer screening recommendations for women with diabetes and found several remarkable policy insights for this case. 
% by expanding the action space of the POMDP model of \citet{ayer2012} to include supplemental MRI and ultrasound screenings. 
Other extensions to this context include formulating a POMDP with continuous states.  
For example, \citet{otten2020stratified} extended the study by \citet{otten2017stratified} to a continuous state problem by modelling the tumor size as a continuous-valued component in the state space, which is based on the assumption that tumor growth follows an exponential distribution.
\citet{horiguchi2021approach} also evaluated the periodic screening programs for breast cancer by converting a POMDP to a fully observable MDP with a continuous state space.

The optimization approaches discussed so far assume that the decision-maker has infinite resources, which is generally not the case. 
Hence, it is often useful to find optimal policies that satisfy certain constraints, such as limited screening resources. 
Several studies have proposed using constrained MDP models for different cancer screening problems~\citep{ayvaci2012, lee2019optimal}.
\citet{poupart2015} and \citet{lee2018monte} devised solution methods for constrained POMDPs (CPOMDPs) based on approximate linear programming and Monte Carlo tree search, respectively. 
Many CPOMDP applications followed, especially in the healthcare domain.
\citet{gan2019personalized} investigated opioid use disorder using CPOMDP models, and imposed restrictions on the available budget for interventions and surveillance (e.g., using wearable devices).
To account for the imperfect patient state information and budgetary constraints in the breast cancer screening problem, \citet{Cevik2018} proposed a finite-horizon CPOMDP model that maximizes patient QALYs while limiting the number of mammographies received during the patient's lifetime.
In this study, we extend \citet{Cevik2018}'s work by introducing supplemental screening tests to the decision process and investigating multi-objective optimization approaches for the breast cancer screening problem.
Different from \citet{sandikci2018screening}, we consider only the simultaneous usage of supplemental tests with mammography, which help simplify the problem for the CPOMDP framework.
Table~\ref{tab:litrev} summarizes the relevant studies in the literature and demonstrates the relative positioning of our work.

\renewcommand{\arraystretch}{1.15}
\begin{table}[!ht]
    \centering
    \caption{Summary of the most closely related studies in the literature.} 
    \label{tab:litrev}
    \resizebox*{0.905\textwidth}{!}{
    \begin{tabular}{p{0.31\textwidth}cm{0.15\textwidth}m{0.12\textwidth}m{0.11\textwidth}c}
    \toprule
    Research& \makecell{Problem/\\Cancer\\ Type}& \multicolumn{1}{c}{Objective} &\multicolumn{2}{c}{\makecell{Model \\ Specification}}& \makecell{Solution \\ Approach} \\
    \cmidrule{4-5}
    &&&\multicolumn{1}{c}{Constraint}&\multicolumn{1}{c}{Framework}\\
    \midrule
    \cite{chhatwal2010optimal}&Breast& \centering QALYs&\makecell{None}&\centering MDP&\makecell{DP}\\
    \cite{zhang2012optimization}&Prostate& \centering QALYs&\makecell{None}&\centering PMDP&\makecell{DP}\\
    \cite{ayvaci2012}&Breast& \centering QALYs&\makecell{Budget}&\centering MDP&\makecell{MIP}\\
    \cite{ayer2012}&Breast& \centering QALYs&\makecell{None}&\centering POMDP&\makecell{DP}\\
    \cite{alagoz2013optimal}&Breast& \centering QALYs&\makecell{None}&\centering MDP&\makecell{DCL}\\
    \cite{erenay2014optimizing}&Colorectal& \centering QALYs&\makecell{None}&\centering POMDP&\makecell{LP}\\
    \cite{ayer2016heterogeneity}&Breast& \centering QALYs&\makecell{None}&\centering POMDP&\makecell{LP}\\
    \cite{akhavan2017markov}&Cervical&\centering Cost& \centering None&\centering MDP&\makecell{DP}\\
    \cite{chen2018optimal}&\makecell{Liver}& \centering \makecell{Cost\\ effectiveness}&\makecell{None}&\centering MDP&\makecell{MIP}\\
    \cite{Cevik2018}&Breast& \centering QALYs&\makecell{No. of \\screenings}&\centering POMDP&\makecell{MILP}\\
    \cite{ccauglayan2018assess}&\makecell{Breast}& \centering QALYs&\makecell{Budget}&\centering POMDP&\makecell{MILP}\\
    \cite{lee2019optimal}&\makecell{Liver}& \centering \makecell{Early-stage\\ detections}&\makecell{ Screening\\ resources}&\centering POMDP&\makecell{ BIP}\\
    \cite{gan2019personalized}&\makecell{OUD}& \centering QALD&\makecell{Budget}&\centering POMDP&\makecell{DP }\\
    \midrule
    Our Study& Breast& \makecell{QALYs and\\ LBCMR}& \centering Budget&\centering POMDP& \makecell{MILP}\\
    \bottomrule
    \end{tabular}
    }
    \begin{tablenotes}
      \footnotesize
      \item OUD: Opioid Use Disorder, DP: Dynamic Programming, MIP: Mixed-Integer Programming, DCL: Double-Control Policy, LP: Linear Programming, MILP: Mixed-Integer Linear Programming, BIP: Binary Integer Programming
    \end{tablenotes}
\end{table}

%%%%%%%%%%%%%%%%%%%%%%%%%%%%%%%%%%%%%%%%%%%%%%%%%%%%%%%%%%%%%%%%%%%%%%%%%%%%%%%%
%%%%%%%%%%%%%%%%%%%%%%%%%%%%%%%%%%%%%%%%%%%%%%%%%%%%%%%%%%%%%%%%%%%%%%%%%%%%%%%%
\section{Methodology}\label{sec:Methodology}
%%%%%%%%%%%%%%%%%%%%%%%%%%%%%%%%%%%%%%%%%%%%%%%%%%%%%%%%%%%%%%%%%%%%%%%%%%%%%%%%
%%%%%%%%%%%%%%%%%%%%%%%%%%%%%%%%%%%%%%%%%%%%%%%%%%%%%%%%%%%%%%%%%%%%%%%%%%%%%%%%

In this section, we first present the preliminaries and the POMDP model for the breast cancer screening problem. 
We then propose a multi-objective CPOMDP model that simultaneously optimizes a patient's QALYs and LBCMR, subject to a budget constraint. 
Finally, we formulate a linear programming model to approximate the multi-objective CPOMDP model.
% Finally, we solve this model using linear programming. 

%%%%%%%%%%%%%%%%%%%%%%%%%%%%%%%%%%%%%%%%%%%%%%%%%%%%%%%%%%%%%%%%%%%%%%%%%%%%%%%%
\subsection{Preliminaries}\label{sec:Preliminaries}
%%%%%%%%%%%%%%%%%%%%%%%%%%%%%%%%%%%%%%%%%%%%%%%%%%%%%%%%%%%%%%%%%%%%%%%%%%%%%%%%
A discrete-time finite horizon POMDP model is defined as a seven tuple $\langle \setTime, \setState, \setAction, \setObs, \setReward, \setTRP, \setOBSModel \rangle$.
Each model component describes one aspect of the environment in which a decision maker --- referred to as \textit{the agent} --- takes actions and makes observations. 
We use $\setTime$ to denote the set of decision epochs and $\setState$ to denote the set of core states, which represents the partially observable states that the agent can occupy. 
At each decision epoch, the agent takes an action $\indAction\in \setAction$ and makes an observation $\indObs \in \setObs$.
Observations and transitions to new states are uncertain, and are therefore described by probabilities. 
That is, $\parfTrp_{ \indState \indStateNew}^{\indTime\indAction} \in \setOBSModel$ represents the probability that the agent will transition from state $\indState \in \setState$ to $\indStateNew \in \setState$ when action $\indAction \in \setAction$ is taken at time $\indTime \in \setTime$, and $\parfObs_{\indState \indObs}^{\indTime \indAction} \in \setOBSModel$ represents the probability that the agent will observe $\indObs \in \setObs$ after taking action $\indAction \in \setAction$ in state $\indState \in \setState$ at time $\indTime \in \setTime$. 
The agent collects rewards throughout the decision process according to $\parfRew_{\indState}^{\indTime\indAction} \in \setReward$, which denotes the reward for performing action $\indAction\in\setAction$ in state $\indState\in\setState$ at time $\indTime \in \setTime$. 
Table~\ref{table:notation} summarizes the notation used in our study.
\begin{table}[!htbp]
    \centering
    \caption{Notation used for the parameters and the mathematical formulas.} 
    \label{table:notation}
    \resizebox{0.925\textwidth}{!}{
    \small
    \begin{tabularx}{\textwidth}{lX}
    \toprule
    % SETS
    Notation& \multicolumn{1}{l}{Description}\\
    \midrule
    $\indTime \in \setTime$ & a decision epoch from the set of decision epochs \\*[0.08cm] 
    $\timeHorizon$ & the total number of decision epochs (decision horizon) \\*[0.08cm] 
    % $\timeHorizon$ & Time horizon (i.e., last decision epoch)\\*[0.08cm] 
    $\indState \in \setState$ & a state in the state space\\*[0.08cm]
    $\indAction \in \setAction$ & an action in the action space\\*[0.08cm]
    $\indObs \in \setObs$ & an observation in the observation space\\*[0.08cm]
    $\vectBePt \in \beSimp$ & a belief state in the belief space\\*[0.08cm]
    $\bePt_\indState$ & the $\indState$th component of $\vectBePt$; the probability that a patient is in core state $\indState$ \\*[0.08cm]
    $\indAcWait, \indAcMammo, \indAcMMRI, \indAcMUS$ & wait, mammography, mammography + MRI, mammography + ultrasound actions\\*[0.08cm]
    $\parfTrp_{ \indState \indStateNew}^{\indTime\indAction} \in \setTRP$ & probability of making a transition from state $\indState$ to state $\indStateNew$ by taking action $\indAction$ at time $\indTime$, and the corresponding probability space\\*[0.08cm]
    $\parfObs_{\indStateNew \indObs}^{\indTime \indAction} \in \setOBSModel$ & probability of making a observation $\indObs$ at state $\indStateNew$ by taking action $\indAction$ at time $\indTime$, and the corresponding probability space\\*[0.08cm]
    $\textit{spec}^{\indTime \indAction}$, $\textit{sens}^{\indTime \indAction}$ & the specificity and sensitivity parameters for action $\indAction\in \setAction$ at time $\indTime \in \setTime$, respectively \\*[0.08cm]
    $\parfRewGeneric_{\indTime}(\genericState_{\indTime}, \genericAction_{\indTime})$ & generic QALYs reward for action $\genericAction_\indTime$ in state $\genericState_\indTime$ at time $\indTime$ \\*[0.08cm] 
    $\parfRewGeneric_{\timeHorizon}(\genericState_{\timeHorizon})$ & generic QALYs reward in the final decision epoch $\timeHorizon$ in state $\genericState_\timeHorizon$ \\*[0.08cm]
    $\parfLTRiskGeneric_{\indTime}(\genericState_{\indTime}, \genericAction_{\indTime})$ & generic LBCMR reward for action $\genericAction_\indTime$ in state $\genericState_\indTime$ at time $\indTime$ \\*[0.08cm] 
    $\parfLTRiskGeneric_{\timeHorizon}(\genericState_{\timeHorizon})$ & generic LBCMR reward in the final decision epoch $\timeHorizon$ in state $\genericState_\timeHorizon$ \\*[0.08cm]
    $\epsilon$ & upper bound on LBCMR  \\*[0.08cm]
    $\parfRew_{\indState}^{\indTime\indAction} \in \setReward$ & immediate rewards for QALYs for taking action $\indAction$ at state $\indState$ and time $\indTime$, and the corresponding reward space\\*[0.08cm]
    $\rewSalvage_{\indState}^{\indTime}, \ \rewTerminal_{\indState}$ & salvage reward at time $\indTime$, and terminal reward for state $\indState \in\setState$\\*[0.08cm]
    % $\rewSalvage_\indTime$ & Salvage rewards (i.e, utility of a possible recovery) corresponding to state $\indState$ when a patient is diagnosed with cancer\\*[0.08cm]
    $\disuScr, \disuTP, \disuFP$ & disutilities for screening, TP, and FP test results \\*[0.08cm]
    % $\valueFunction_\indTime^{\indAction}(\bePt)$ & Value function for expected reward received by taking action $\indAction$ for a given belief state $\vectBePt$ at  $\indTime$ \\*[0.08cm]  
    $\funcOptV_{\indTime}^{\indAction}(\vectBePt), \ \funcApxV_{\indTime}^{\indAction}(\vectBePt)$ & optimal and approximate value functions for belief state $\vectBePt$ at time $\indTime$\\*[0.08cm]
    $\alpha \in \Gamma$ & a piecewise linear value function and its set \\*[0.08cm]
    % $\funcApxV_{\indTime}(\vectGrPt)$ & An approximate value function for belief state $\vectGrPt$ \\*[0.08cm]
    $\policy \in \setPolicy$ & a policy and set of policies \\*[0.08cm]
    % $\parfLTRisk_{\indState}^{\indTime \indAction}$ & immediate reward for breast cancer risk corresponding to state $\indState$, action $\indAction$ and time $\indTime$\\*[0.08cm]
    $\parfCost_{\indState}^{\indTime\indAction}$ & cost of taking action $\indAction$ at state $\indState$ and time $\indTime$\\*[0.08cm] 
    $\parBudgetLim$ & budget limit\\*[0.08cm] 
    $\vectGrPt \in \setGrid$ & a grid point and set of grid points\\*[0.08cm]
    $\indGrid \in \setGridIndex$ & grid index in the index set for the grid points\\*[0.08cm] 
    $\beta_{\indGrid\indGridNew}^{\indTime \indAction \indObs}$ & convex combination weight for the grid point $\vectGrPt^{\indGridNew}$ for the updated belief state $(\vectGrPt^{\indGrid})'$ when the update is performed based on action $\indAction$ and observation $\indObs$ at time $\indTime$\\*[0.08cm]
    % $\beta_{\indGridNew}^{\indGrid \indAction \indObs}$ & Interpolation weight assigned to the $\indGridNew$th grid point in set $\setGrid$, if the updated belief state is $\grPt^\indGrid$, the action taken is $\indAction$ and the observation is $\indObs$\\*[0.08cm]
    $\resoVal, \vectResoVal = [\resoVal_1,\hdots,\resoVal_n]$ & resolution value and resolution vector used for grid construction\\*[0.08cm]
    % $n$ & Number of regions in the belief space\\*[0.08cm]
    $\vectResoThreshold = [\resoThreshold_1,\hdots,\resoThreshold_n]$ & resolution threshold vector used to partition the belief space\\*[0.08cm]
    % $\phi$ & Boundary values used to partition the belief space\\*[0.08cm]
    $\parfGridTr_{\indGrid \indGridNew}^{\indTime \indAction} $ & probability of transitioning from belief state $\vectGrPt^\indGridNew$ to the belief state $\vectGrPt^\indGrid$ by taking action $\indAction$ at $\indTime$ \\*[0.08cm]
    % $\alpha_{\indState}(\indTime)$ & $\alpha$-vector components for  core state $\indState$ at $\indTime$\\*[0.08cm]
    % $\iota(\bePt,\indAction, \indObs)$ & Index of an $\alpha$-vector associated with a decision epoch for the belief state $\bePt$, action $\indAction$, and observation $\indObs$\\*[0.08cm]
    $\varDualLP_{\indTime \indGrid \indAction}$ & state-action occupancy measures, i.e., the probability of occupying belief state $\vectGrPt^\indGrid \in \setGrid$ and taking action
    $\indAction$ at time $\indTime$\\*[0.08cm]   
    $h_1, h_2$ & expected total QALYs and expected total life-time risk\\*[0.08cm] 
    % $\varPrimalLP_{\indTime \indGrid}$ & Values assigned to each grid point $\grPt^\indGrid \in \setGrid$ at time  $\indTime$\\*[0.08cm]
    $\parBeDistn_{\indGrid}$ & probability of occupying grid point $\vectGrPt^\indGrid$ at the  first decision epoch\\*[0.08cm]
    $\varDetPol_{\indTime \indGrid \indAction}$ & binary variable with value one if action $\indAction$ is taken at belief state $\vectGrPt^\indGrid \in \setGrid$ at time $\indTime$\\*[0.08cm]
    % $\parfDeathCaBfrScreen_{\indStateLocalized\indStateRegional}^{\indTime\indAction}  $ & Probability of death by cancer before taking action $\indAction$ at $\indTime$\\*[0.08cm]
    $\parfDeathCaAfterScreen_{\indTime\indState}$ & probability of death by cancer post-diagnosis for a patient diagnosed at state $\indState$ at time $\indTime$\\*[0.08cm]
    $\parDiscount$ & the discount factor \\*[0.08cm]
    $\mathbf{\objweight}=[\objweight_1,~\objweight_2]$ & the weight vector for the objective function, giving the weights for the QALYs and LBCMR terms, respectively \\*[0.08cm]
    % $\objweight=[\objweight_1,~\objweight_2]$ & the scale factor for the objective function, giving the coefficient to scale for the QALYs and LBCMR terms, respectively. \\*[0.08cm]
    $\setAllowedBeliefComponents(\resoVal)$ & for a fixed-resolution grid, the set of allowed values for the components of a belief state\\*[0.08cm]
    \bottomrule
    \hline
    \end{tabularx}
        }
\end{table}

% By definition, POMDP models address problems with a partially observable core states.
As the state space is only partially observable, the POMDP agent cannot know its core state with certainty. Instead, the agent maintains a \textit{belief state}: a probability distribution over the set of core states $\setState$. The belief state $\vectBePt\in \beSimp$ is an $|\setState|$-dimensional vector where $\bePt_\indState$ gives the probability that the agent is in the state $\indState\in \setState$, that is,
\begin{equation*}
    \beSimp = \{\vectBePt \mid \textstyle\sum_{\indState \in \setState} \bePt_{\indState} = 1, \bePt_{\indState} \geq 0,~ \forall \indState \in \setState\}.
\end{equation*}
At each decision epoch $\indTime \in \setTime$, as the agent takes an action $\indAction \in \setAction$, and makes an observation $\indObs \in \setObs$, its belief state is updated from $\vectBePt$ to $\vectBePt'$ according to Bayes' rule, where
\begin{equation}\label{eq:beUpdate}
\bePt_{\indStateNew}' = \frac{\sum_\indState \bePt_\indState \parfObs^{\indTime\indAction}_{\indState \indObs} \parfTrp_{\indState \indStateNew}^{\indTime\indAction}}{\sum_{\indState, \indStateNew} \bePt_\indState \parfTrp^{\indTime\indAction}_{\indState \indStateNew} \parfObs^{\indTime\indAction}_{\indState \indObs}}
\end{equation}
gives the updated probability that the agent is in state $\indStateNew\in\setState$. The Bellman optimality equation (i.e., value function) corresponding to action $\indAction \in \setAction$ can be obtained as
\begin{equation}\label{eq:valueFunction-traditional}
\funcOptV_\indTime^{\indAction}(\vectBePt) = \sum_{\indState\in \setState} \bePt_\indState \big( \parfRew^{\indTime\indAction}_{\indState} +\parDiscount \sum_{\indObs\in \setObs} \sum_{\indStateNew\in \setState} \parfObs^{\indTime\indAction}_{\indState \indObs}\parfTrp^{\indTime\indAction}_{\indState \indStateNew} \funcOptV^*_{\indTime+1}(\vectBePt') \big),
\end{equation}
where $\funcOptV^*_{\indTime}(\vectBePt) = \max_{\indAction\in \setAction}\{\funcOptV^{\indAction}_{\indTime}(\vectBePt)\}$ and $\parDiscount$ is the discount factor. In finite horizon problems, $\parDiscount$ is typically set to one.
Note that, in Equation~\eqref{eq:valueFunction-traditional}, first the observation $\indObs$ is made at state $\indState$, then the state transition state $\indStateNew$ occurs.
This order of events is different than that of standard POMDP formulations (e.g., see~\citep{poupart2015}), however, it is commonly used in modeling cancer screening problems~\citep{ayer2012, Cevik2018}.

% \tcolB{
% Equation~\ref{eq:valueFunction-traditional} describes a decision process where observations are made after taking an action and transitioning to a new state. 
% For the breast cancer screening problem, however, first, a screening action is taken, which leads to an observation regarding the health of the patient, and the patient transitions to a new state based on this observation. 
% This is because, if a patient is observed to be healthy they will remain in the decision process, but if they are observed to have cancer they will exit the decision process to begin receiving treatment. 
% Accordingly, Equation~\ref{eq:valueFunction-traditional} needs to be reformulated to allow for this updated order of events. 
% Under this reformulation, the updated value function is
% }
% \begin{equation}\label{eq:valueFunction-bcs}
% 
% \funcOptV_\indTime^{\indAction}(\vectBePt) = \sum_{\indState\in \setState} \bePt_\indState \big( \parfRew^{\indTime\indAction}_{\indState} +\parDiscount \sum_{\indObs\in \setObs} \sum_{\\\in \setState} \parfObs^{\indTime\indAction}_{\indState \indObs}\parfTrp^{\indTime\indAction}_{\indState \indStateNew} \funcOptV^*_{\indTime+1}(\vectBePt') \big),
% \end{equation}
% \citet{ayer2012[Appendix A] provide a proof that Equations~\ref{eq:valueFunction-traditional} and~\ref{eq:valueFunction-bcs} are equivalent in the breast cancer screening problem.}

%%%%%%%%%%%%%%%%%%%%%%%%%%%%%%%%%%%%%%%%%%%%%%%%%%%%%%%%%%%%%%%%%%%%%%%%%%%%%%%%
\subsection{CPOMDP Model for Breast Cancer Screening}
%%%%%%%%%%%%%%%%%%%%%%%%%%%%%%%%%%%%%%%%%%%%%%%%%%%%%%%%%%%%%%%%%%%%%%%%%%%%%%%%
We formulate a CPOMDP model for the breast cancer screening problem to optimize the screening decisions made by a decision maker (e.g., a radiologist or a physician) at each age between the ages of 40 to 100 (i.e., $\setTime = \{1,2,\hdots,60\}$).
Below, we first describe the CPOMDP model components, then formally define the multi-objective CPOMDP model.

%%%%%%%%%%%%%%%%%%%%%%%%%%%%%%%%%%%%%%%%%%%%%%%%%%%%%%%%%%%%%%%%%%%%%%%%%%
% \subsubsection{Model Components}
%%%%%%%%%%%%%%%%%%%%%%%%%%%%%%%%%%%%%%%%%%%%%%%%%%%%%%%%%%%%%%%%%%%%%%%%%%
% We define the CPOMDP model components as follows:
\begin{itemize}\setlength\itemsep{0.5em}
    \item \textbf{Core states}: 
    We choose the core states of the model to represent an individual's health status.
    We consider a patient to be occupying one of the following states: healthy, having early-stage (i.e., in-situ) cancer, having late-stage (i.e., invasive) cancer, death due to breast cancer, and death from other causes. 
    Our model considers the last two states as the fully observable terminal states, where the patient leaves the decision process. 
    For simplicity, we map each core state to an integer, and represent the complete core state space as $\bar{\setState} = \{0,1,2,3,4\}$, corresponding to the five states given above. 
    Similarly, the partially observable state space is defined as $\setState = \{0,1,2\}$.
    Consequently, the belief state $\vectBePt$ defined over the partially observable states is an $|\setState|$-tuple, storing the probability of a patient being healthy, having in-situ cancer, or having invasive cancer.
    
    \item \textbf{Actions and observations}: 
    We consider three screening actions in our model, namely, mammography ($\indAcMammo$), mammography supplemented with MRI ($\indAcMMRI$), and mammography supplemented with ultrasound ($\indAcMUS$). 
    We assume that the supplemental tests are performed right after mammography, and the decision to perform these test are based on the patient's current breast cancer risk estimate that reflects both static (e.g., family history) and dynamic (e.g., age) risk factors.
    We also include a wait action ($\indAcWait$) to indicate no screening is received at a given decision epoch. 
    Accordingly, the action space is $\setAction = \{\indAcWait, \indAcMammo, \indAcMMRI, \indAcMUS\}$.
    When an action is taken, the associated observation is either positive (indicating that the patient has cancer) or negative (indicating that the patient is cancer-free), which leads to the observation space of $\setObs = \{\indObsNegative, \indObsPositive\}$.
    For the wait action, a positive observation can be attributed to self-detection or a clinical breast exam (e.g., feeling a lump in the breast tissue), whereas the observations for screening actions are the direct outcomes of the screening tests.
    The received observation can be a true positive (TP), true negative (TN), false positive (FP), or false negative (FN) result, depending on the underlying state of the patient.
    
    \item \textbf{Transition and observation probabilities}: 
    Transition probabilities capture the transitions between the core states. 
    As indicated by \citet{Maillart2008}'s study, the transition probabilities for the breast cancer screening problem possess several distinguishing properties.
    Firstly, the transition probabilities are time-dependent, as the probability of developing breast cancer increases with age~\citep{Maillart2008}.
    Secondly, the transition probability matrices are upper diagonal, as a patient cannot recover from cancer during the decision process, i.e., once the patient is diagnosed with cancer, she leaves the decision process and does not return either.
    Lastly, the transition probabilities are action-independent, that is, $\parfTrp_{\indState\indStateNew}^{\indTime\indAction} = \parfTrp_{\indState\indStateNew}^{\indTime}, \ \forall \indAction \in \setAction$.
    Observation probabilities are linked to the performance (i.e., specificity and sensitivity) of the underlying action (e.g., screening test).
    Let $\textit{spec}^{\indTime \indAction}$ and $\textit{sens}^{\indTime \indAction}$ be the specificity and sensitivity parameters for action $\indAction\in \setAction$ at time $\indTime \in \setTime$, respectively.
    Accordingly, $\parfObs_{0 \indObsNegative}^{\indTime \indAction} = \textit{spec}^{\indTime \indAction}$, $\parfObs_{0 \indObsPositive}^{\indTime \indAction} = (1-\textit{spec}^{\indTime \indAction})$, $\parfObs_{i\neq 0, \indObsNegative}^{\indTime \indAction} = (1-\textit{sens}^{\indTime \indAction})$, and $\parfObs_{i\neq 0, \indObsPositive}^{\indTime \indAction} = \textit{sens}^{\indTime \indAction}$.
    
    \item \textbf{Rewards and costs}: 
    In the breast cancer screening problem, the rewards are typically expressed in terms of QALY values.
    Immediate rewards, $\parfRew_{\indState}^{\indTime \indAction}$, correspond to the QALY values associated with action $\indAction$ for a patient in state $\indState$ at time $\indTime$. 
    These are calculated by taking into account the probability of death in a given decision epoch, as well as various disutilities associated with the actions.
    Assuming that there are no disutilities associated with the wait action, its corresponding expected immediate reward can be calculated using a half-cycle correction method, that is, $\parfRew_{\indState}^{\indTime \indAcWait} = 1\times \Pr(\textit{patient is alive at $\indTime+1$}|\textit{patient is alive at $\indTime$}) + 0.5\times \Pr(\textit{patient is dead at $\indTime+1$}|\textit{patient is alive at $\indTime$})$ \citep{gray2011applied}.    
    Immediate rewards for other actions can be derived from $\parfRew_{\indState}^{\indTime \indAcWait}$.
    Let $\disuScr, \disuTP, \ \text{and} \ \disuFP$ correspond to the screening, TP test, and FP test disutilities, respectively.
    We obtain the immediate rewards for the negative test results as $\parfRew_{\indState\indObsNegative}^{\indTime \indAction} = \parfRew_{\indState}^{\indTime \indAcWait} - \disuScr$ for $\indAction \in \setAction\setminus \indAcWait$. % for a negative observation. 
    Similarly, we obtain the immediate rewards for the positive test results, $\parfRew_{\indState\indObsPositive}$, by adjusting for the TP (i.e., when $\indState\neq 0$, $\indObs = \indObsPositive$) and FP (i.e., when $\indState = 0$, $\indObs = \indObsPositive$) test disutilities.

    The disutility values are typically low for screening (e.g., 0.5 days for undergoing screening), whereas positive test results lead to additional procedures such as biopsies, incurring much higher disutility values (e.g., 2--4 weeks)~\citep{ayer2012}. 
    Terminal rewards, $\rewTerminal_{\indState}, \ \indState \in \setState$, are received when a patient reaches time step $\timeHorizon$, whereas salvage rewards, $\rewSalvage_{\indState}, \ \indState\neq \indStateHealthy$, are received when a patient is diagnosed with cancer to account for the utility of a possible recovery.
    
    % We consider LBCMR (denoted by $\parfLTRisk_{\indState}^{\indTime \indAction}$) as another reward component, which corresponds to the risk of dying because of breast cancer if the patient is in state $\indState$ at time $\indTime$ and takes action $\indAction$.
    We consider LBCMR as another target/reward value which corresponds to the risk of dying because of breast cancer. The salvage values for LBCMR, $\parfDeathCaAfterScreen_{\indTime\indState}$, depends on patient's age (i.e., $\indTime$) and health state, $\indState$.
    Lastly, we consider screening actions to have associated monetary costs, $\indAcMammo$ being the cheapest and $\indAcMMRI$ being the most expensive option.
    Cost values, $\parfCost^{\indAction}$ for $\indAction \in \setAction$, are independent of the state and time step.
\end{itemize}

% \textbf{Optimality equations}: 
In finite horizon CPOMDPs, the objective is to maximize the total expected reward, subject to a budget constraint. 
In our problem, we consider two separate objectives: maximization of the QALYs, and minimization of LBCMR.
For a given starting belief state $\vectBePt^0$, this can be characterized as follows:
\begin{subequations}\label{eq:cpomdp}
\begin{align}
\label{eq:cpomdpObj}\hhsp{-0.4cm}\max_{\policy \in \setPolicy}  &\ \Big\lbrace \big( E_{\vectBePt^0}^{\policy} \big[\sum_{\indTime=0}^{\timeHorizon-1} \parfRewGeneric_{\indTime}(\genericState_{\indTime}, \genericAction_{\indTime}) + \parfRewGeneric_{\timeHorizon}(\genericState_{\timeHorizon}) \big], E_{\vectBePt^0}^{\policy}\big[- \sum_{\indTime=0}^{\timeHorizon-1} \parfLTRiskGeneric_{\indTime}(\genericState_{\indTime}, \genericAction_{\indTime}) - \parfLTRiskGeneric_{\timeHorizon}(\genericState_{\timeHorizon}) \big] \big)\Big \rbrace,\\
\label{eq:cpomdpCon} \text{s.t.} & \ E_{\vectBePt^0}^{\policy} \big[\sum_{\indTime=0}^{\timeHorizon-1} \parfCost_{\indTime}(\genericState_{\indTime}, \genericAction_{\indTime}) \big] \leq \parBudgetLim,
\end{align}
\end{subequations}
where $\genericState_{\indTime}$ and $\genericAction_{\indTime}$ denote the states and actions at time ${\indTime}$, respectively; $\setPolicy$ represents the set of all policies; and $\policy \in \setPolicy$ is a policy from the policy space.
In addition, $\parfRewGeneric(\cdot)$ and $\parfLTRiskGeneric(\cdot)$ functions represent the generic reward functions for QALY and LBCMR, respectively.
Similar to \eqref{eq:valueFunction-traditional}, the two optimization objectives in \eqref{eq:cpomdpObj} can be represented using Bellman optimality equations for unconstrained POMDP problems, and are typically solved using iterative optimization routines such as value iteration and backward induction methods~\citep{puterman2014markov, Cevik2018}.
However, constrained MDP and POMDP models are usually formulated as (approximate) linear programs, the latter involving a grid-based approximation mechanism.
Accordingly, we next review the grid-based approximations for POMDPs and then formulate our linear programming model to approximate the CPOMDP model presented in \eqref{eq:cpomdp}.

%%%%%%%%%%%%%%%%%%%%%%%%%%%%%%%%%%%%%%%%%%%%%%%%%%%%%%%%%%%%%%%%%%%%%%%%%%%%%%%%
\subsection{Grid-based Approximation Mechanism for CPOMDPs}
%%%%%%%%%%%%%%%%%%%%%%%%%%%%%%%%%%%%%%%%%%%%%%%%%%%%%%%%%%%%%%%%%%%%%%%%%%%%%%%%
A commonly used approach to approximate $\funcOptV_{\indTime}^*(\vectBePt)$ is to discretize $\beSimp$ into a set of grid points and calculate the values for only these grid points. 
We denote this set of grid points as $\setGrid = \{\vectGrPt^\indGrid \mid \indGrid\in \setGridIndex\}$, where $\setGridIndex = \{1,\hdots, |\setGrid|\}$ is index set of $\setGrid$.
Specifically, an approximate POMDP value function can be obtained for $\grPt \in \setGrid$ and $\indAction \in \setAction$ as~\citep{lovejoy1991computationally}:
\begin{equation}
\label{eq:q_calculation} \funcApxV_{\indTime}^{\indAction}(\vectGrPt) = \sum_{\indState \in \setState} \grPt_{\indState} \parfRew_{\indState}^{\indTime \indAction} +  \sum_{\indState \in \setState} \grPt_{\indState} \sum_{\indObs \in \setObs} \sum_{\indStateNew \in \setState} \parfObs_{\indState \indObs}^{\indTime \indAction} \ \parfTrp_{\indState \indStateNew}^{\indTime \indAction}\ {\sum_{\indGrid \in \setGridIndex} \beta_{\indGrid} \funcApxV_{\indTime+1}(\vectGrPt^{\indGrid})},
\end{equation}
where $\funcOptV_{\indTime+1}^*(\ugrPtvect) \approx {\sum_{\indGrid \in \setGridIndex} \beta_{\indGrid} \funcApxV_{\indTime+1}(\vectGrPt^{\indGrid})}$. In this approximation, $\ugrPtvect\not\in\setGrid$ is represented as a convex combination of the elements in $\setGrid$, in which $\beta_k$ gives the coefficient of $\vectGrPt^{\indGrid}, \ \indGrid \in \setGridIndex$.
There are several different strategies for calculating these $\beta$-values; in this paper, we employ the LP-based approach provided by \citet{sandikcci2010reduction}.

%%%%%%%%%%%%%%%%%%%%%%%%%%%%%%%%%%%%%%%%%%%%%%%%%%%%%%%%%%%%%%%%%%%%%%%%%%%%%%%%
\subsubsection{Grid construction}
%%%%%%%%%%%%%%%%%%%%%%%%%%%%%%%%%%%%%%%%%%%%%%%%%%%%%%%%%%%%%%%%%%%%%%%%%%%%%%%%
The grid set $\setGrid$ can have a significant impact on the approximation quality, as was shown by \citet{sandikci2018screening} in the case of an unconstrained POMDP model for another breast cancer screening problem.
In uniform grid construction methods, a resolution value, $\resoVal$, is used to construct the grid set. 
For a uniform resolution grid, the distance between grid points in each dimension is integer multiples of $1/\resoVal$. 
% Specifically, for a grid point $\vectGrPt \in \setGrid$, the distance between beliefs in each dimension is given by $\resoVal^{-1}$. 
Thus, the possible values in the grid point $\vectGrPt \in \setGrid$ for each state $\indState \in \setState$ are integer multiples of $1/\resoVal$. 
Formally, the grid set constructed by using a uniform grid approach can be obtained as~\citep{lovejoy1991computationally}
\begin{equation}\label{eq:uniform-grid}
    \setGrid=\left\{\bePt \mid \bePt_\indState=\frac{n}{\resoVal},~n\in \mathbb{Z}_{\geq 0},~\sum_\indState \bePt_\indState = 1\right\}
\end{equation}
where $\mathbb{Z}_{\geq 0}$ is the set of all integers greater than or equal to zero. 
For a problem with $|\setState|$ states and a resolution of $\resoVal$, the size of the approximate grid set is $|\setGrid|=\binom{|\setState|-1+\resoVal}{\resoVal}$.
% \begin{equation}\label{eq:num-grids-uniform-grid}
%     |\setGrid|=\binom{|\setState|-1+\resoVal}{\resoVal}
% \end{equation}
A detailed explanation of uniform resolution grid construction and sample grid sets can be found in Appendix~\ref{sec:appendix-fixed-resolution-grid-construction}.

For the breast cancer screening problem, the patients are likely to have belief states that are much closer to the corner point $\vectBePt=[1, 0, 0]$ (healthy) than to the other two corner points $\vectBePt=[0, 1, 0]$ (in-situ cancer) and $\vectBePt=[0, 0, 1]$ (invasive cancer). 
This is mainly because the probability of getting cancer on a specific time step is relatively low, as indicated by the transition probability values. 
As a result, using a uniform resolution grid to approximate the belief space would lead to a significant amount of computational overhead to be spent on belief states that are rarely visited, which is noted in previous studies as well~\citep{Cevik2018, sandikci2018screening}.
Accordingly, we employ a variable resolution uniform grid construction method for generating the grid points.
In this approach, the resolution value varies based on a set of thresholds over the belief state components. 
Specifically, the belief space is divided into distinct intervals based on the probability that a patient is healthy, $\bePt_0$ and a resolution is assigned to each interval. 
For an interval $\pi_0\in[a, b]$ with resolution $\resoVal$, the uniform grid for the resolution $\resoVal$ is constructed, and all beliefs satisfying $a\leq\pi_0\leq b$ are retained. 
Formally, threshold values $\vectResoThreshold=[\resoThreshold_1,\hdots,\resoThreshold_n]$ are defined over the probability that a woman is healthy, as specified by the first component of belief states, $\bePt_0$. 
Specific resolution values $\vectResoVal=[\resoVal_1,\hdots,\resoVal_n]$ are assigned for the intervals defined by these thresholds. 
The variable resolution grid is then constructed by generating the corresponding fixed-resolution grid sets and drawing from the relevant threshold regions~\citep{Cevik2018, sandikci2018screening}.
A detailed explanation of variable resolution grid construction and sample grid sets can be found in Appendix~\ref{sec:appendix-variable-resolution-grid-construction}.

% For instance, for $|\setState|=3$, using a resolution value of $\resoVal=4$ when $\bePt_{0} \in [1,0.5]$ and a different resolution value of $\resoVal=2$ when $\bePt_{0} \in (0.5, 0]$ leads to the grid set 
% \begin{align*}
% \setGrid^{\text{variable}} = &  \bigg\lbrace \colvec[1]{0}{0},\  \colvec[\frac{3}{4}]{\frac{1}{4}}{0},\ \colvec[\frac{3}{4}]{0}{\frac{1}{4}},\ \colvec[\frac{1}{2}]{\frac{1}{2}}{0},\ \colvec[\frac{1}{2}]{\frac{1}{4}}{\frac{1}{4}},\ \colvec[\frac{1}{2}]{0}{\frac{1}{2}}, \ \colvec[0]{1}{0},\ \colvec[0]{\frac{1}{2}}{\frac{1}{2}},\ \colvec[0]{0}{1} \bigg\rbrace.
% \end{align*}
% That is, we end up with 9 grid points using this approach. 
% On the other hand, a fixed resolution value of $\resoVal=4$ would lead to 15 grid points.
% For instance, for the grid construction method used to generate the above $\setGrid^{\text{variable}}$, the threshold value vector is $[0.5, 0]$, and the resolution value vector is $[4,2]$.

%%%%%%%%%%%%%%%%%%%%%%%%%%%%%%%%%%%%%%%%%%%%%%%%%%%%%%%%%%%%%%%%%%%%%%%%%%%%%%%%
\subsubsection{LP model for multi-objective CPOMDP}
%%%%%%%%%%%%%%%%%%%%%%%%%%%%%%%%%%%%%%%%%%%%%%%%%%%%%%%%%%%%%%%%%%%%%%%%%%%%%%%%
The grid-based approximation that we describe above reduces a POMDP to an approximate MDP~\citep{sandikcci2010reduction}.
That is, the approximate model that is characterized by the value functions in \eqref{eq:q_calculation} corresponds to an MDP model where $\vectGrPt \in \setGrid$ are the states.
The transition probabilities between the states (i.e., grid points) can be obtained from the POMDP belief states, observation probabilities, transition probabilities, and $\beta$-values as follows:
\begin{align*}
\parfGridTr_{\indGrid \indGridNew}^{\indTime \indAction} =  
\begin{cases}
\displaystyle \sum_{\indObs \in \setObs} \sum_{\indState \in \setState} \sum_{\indStateNew \in \setState} {\beta_{\indGrid\indGridNew}^{\indTime \indAction \indObs}}\
\grPt_{\indState}^{\indGrid} \parfObs_{\indState \indObs}^{\indTime \indAction} \ \parfTrp_{\indState \indStateNew}^{\indTime \indAction} & \quad \text{if} \ \indAction = \indAcWait, \indObs \in \setObs, \ \text{or} \ \indAction \neq \indAcWait, \indObs = \indObsNegative, \\[0.5em]
\displaystyle \sum_{\indObs\in \setObs} \sum_{\indStateNew \in \setState} {\beta_{\indGrid\indGridNew}^{\indTime \indAction \indObs}}\
\grPt_{\indStateHealthy}^{\indGrid} \parfObs_{\indStateHealthy \indObs}^{\indTime \indAction} \ \parfTrp_{\indStateHealthy \indStateNew}^{\indTime \indAction} & \quad \text{if} \ \indAction \neq \indAcWait \ \text{and} \ \indObs = \indObsPositive,
\end{cases}
\end{align*}
where $\beta_{\indGrid\indGridNew}^{\indTime \indAction \indObs}$ corresponds to the convex combination weight for the grid point $\vectGrPt^{\indGridNew}$ for the updated belief state $(\vectGrPt^{\indGrid})'$ when the update is performed based on action $\indAction$ and observation $\indObs$ at time $\indTime$. Note that $\parfGridTr$ values can be precalculated and stored for later usage.

By using a similar approximation mechanism, we can reduce the CPOMDP defined by \eqref{eq:cpomdp} to an approximate constrained MDP. 
Let $\varDualLP_{\indTime \indGrid \indAction}$ represent the \textit{occupancy measures}, which specify the fraction of time that the patient occupies the belief state $\vectGrPt^{\indGrid} \in \setGrid$ at time $\indTime \in \setTime$ and takes action $\indAction \in \setAction$.
We can formulate the following LP model for the approximate MDP:
\begin{subequations}\label{eq:dualLP}
\begin{align}
% \label{eq:dualLPObj} \max \ \ & \{h_1, -h_2\}\\
% \label{eq:dualLPObj} \max \ \ & \scaleFactor_1 \objweight_1 h_1 - \scaleFactor_2 \objweight_2 h_2\\
\label{eq:dualLPObj} \max_{\varDualLP} \ \ & h(\varDualLP):=(h_1(\varDualLP), h_2(\varDualLP))\\
\label{eq:dualLP_eqn1} \sthat \ \ & \sum_{\indAction \in \setAction} {\varDualLP_{0 \indGrid \indAction}} = {\parBeDistn_{\indGrid}}, && \hspace{-0.95cm} \indGrid\in \setGridIndex,\\[2.2ex]
\label{eq:dualLP_eqn2} \ & \sum_{\indAction \in \setAction}{\varDualLP_{\indTime \indGrid \indAction}} - \sum_{\indAction \in \setAction} \sum_{\indGridNew \in \setGridIndex} { \parfGridTr_{\indGridNew \indGrid}^{\indTime-1 \indAction}} {\varDualLP_{\indTime-1 \indGridNew \indAction}} = 0, && \hspace{-0.95cm}  \indGrid\in \setGridIndex, \  0 < \indTime < \timeHorizon,\\[2.2ex]
\label{eq:dualLP_eqn3} \ & {\varDualLP_{\timeHorizon \indGrid}} - \sum_{\indAction \in \setAction} \sum_{\indGridNew \in \setGridIndex} { \parfGridTr_{\indGridNew \indGrid}^{\timeHorizon-1 \indAction}} {\varDualLP_{\timeHorizon-1 \indGridNew \indAction}} = 0, && \hspace{-0.95cm}  \indGrid\in \setGridIndex,\\[2.2ex]
\label{eq:dualLP_eqnBudget}\ & \sum_{\indTime < \timeHorizon} \sum_{\indGrid \in \setGridIndex} \sum_{\indAction\in \setAction} \sum_{\indState \in \setState} \grPt_{\indState}^{\indGrid} \parfCost_{\indState}^{\indTime\indAction} \varDualLP_{\indTime \indGrid \indAction} \leq \parBudgetLim,\\[2.2ex]
\label{eq:dualLP_eqn4} \ & {\varDualLP_{\timeHorizon \indGrid}} \geq 0, \qquad {\varDualLP_{\indTime \indGrid \indAction}} \geq 0, && \hspace{-0.95cm} \indAction\in \setAction, \  \indGrid\in \setGridIndex, \  \indTime < \timeHorizon.
\end{align}
\end{subequations}

\noindent Let $\mathcal{X}$ represent the feasible solution space for the approximate CPOMDP model, as defined by the constraints~\eqref{eq:dualLP_eqn1}-\eqref{eq:dualLP_eqn4}.
Then, $h:\mathcal{X} \rightarrow \mathbb{R}^2$ maps each solution in $\varDualLP \in \mathcal{X}$ into the 2-dimensional objective vector $y:=(h_1(\varDualLP), h_2(\varDualLP))$, with $h_k: \mathcal{X} \rightarrow \mathbb{R},\ k \in \{1,2\}$.
The two optimization objectives in this model, $h_1$ and $h_2$ correspond to expected total QALYs and expected LBCMR, respectively.
% The scale of QALYs and LBCMR is substantially different, where typical values for QALYs are around 40, and LBCMR generally does not exceed 10\%. Accordingly, the two quantities must be scaled in Equation~\ref{eq:dualLPObj so that they share the same domain (i.e., so that their allowed values are the same). This scaling process is handled by $\vectScaleFactor=[\scaleFactor_1, \scaleFactor_2]$, the scale factor, and its values are chosen so that $h_1$ and $h_2$ are approximately scaled to between 0 and 1. }
The constraints specified by~\eqref{eq:dualLP_eqn1} link the occupancy measures to the initial belief distribution ${\parBeDistn_{\indGrid}}, \ \indGrid\in \setGridIndex$ over the grid points.
The initial belief distribution can be selected in a way that $\sum_{\indGrid \in \setGridIndex} \parBeDistn_{\indGrid} \grPt^{\indGrid} = \vectBePt^{0}$. 
The constraints in~\eqref{eq:dualLP_eqn2} and \eqref{eq:dualLP_eqn3} establish the connection between the occupancy measures at successive decision epochs.
Constraint~\eqref{eq:dualLP_eqnBudget} imposes a budget limit of $\parBudgetLim$ over the decision horizon.
That is, the expected total cost of the taken actions cannot exceed $\parBudgetLim$.

QALY and LBCMR objectives can be formulated, respectively, as follows:
% \begin{small}
\begin{align}
\nonumber h_1(\varDualLP) \ &= \sum_{\indTime < \timeHorizon}\sum_{\indAction\in \setAction}\sum_{\indGrid \in \setGridIndex}\sum_{\indState \in \setState} \grPt_{\indState}^{\indGrid} \parfRew_{\indState}^{\indTime \indAction} {\varDualLP_{\indTime \indGrid \indAction}} + \sum_{\indTime < \timeHorizon} \sum_{\indAction\neq \indAcWait, \indAction \in \setAction} \sum_{\indGrid \in \setGridIndex}\sum_{\indState \neq \indStateHealthy, \indState \in \setState}  \grPt_{\indState}^{\indGrid} \parfObs_{\indState\indObsPositive}^{\indTime\indAction} \rewSalvage_{\indState}^{\indTime} {\varDualLP_{\indTime \indGrid \indAction}}\\ 
\label{eq:CPOMDPQALYObj}\ &+ \sum_{\indGrid \in \setGridIndex}\sum_{\indState \in \setState} \grPt_{\indState}^{\indGrid} \rewTerminal_{\indState}{\varDualLP_{\timeHorizon\indGrid}}, \\[1.2ex]
\nonumber h_2(\varDualLP) \ &= \sum_{\indTime < \timeHorizon} \sum_{\indGrid \in \setGridIndex} \sum_{\indAction\neq\indAcWait} \sum_{\indState \neq \indStateHealthy} \grPt_{\indState}^{\indGrid} \parfObs_{\indState\indObsPositive}^{\indTime\indAction}
\parfDeathCaAfterScreen_{\indTime\indState}
{\varDualLP_{\indTime \indGrid \indAction}} 
+ \sum_{\indTime < \timeHorizon} \sum_{\indGrid \in \setGridIndex} \sum_{\indAction\neq \indAcWait} \grPt_{\indStateLocalized}^{\indGrid} \parfObs_{\indStateLocalized\indObsNegative}^{\indTime\indAction}
\parfTrp_{\indStateLocalized\indStateRegional}^{\indTime\indAction}
{\varDualLP_{\indTime \indGrid \indAction}}\\ 
\label{eq:CPOMDPRiskObj}\ & + \sum_{\indTime < \timeHorizon} \sum_{\indGrid \in \setGridIndex} \sum_{\indObs \in \setObs} \grPt_{\indStateLocalized}^{\indGrid} \parfObs_{\indStateLocalized\indObs}^{\indTime\indAcWait}
\parfTrp_{\indStateLocalized\indStateRegional}^{\indTime\indAcWait}
{\varDualLP_{\indTime \indGrid \indAcWait}}.
\end{align}
% \end{small}
The first component in \eqref{eq:CPOMDPQALYObj} corresponds to the expected total QALYs except for the case the patient is diagnosed with cancer, whereas the second component calculates the post-diagnosis QALYs.
The final component accounts for terminal rewards.
In \eqref{eq:CPOMDPRiskObj}, the first component calculates the LBCMR post-diagnosis where $\parfDeathCaAfterScreen_{\indTime\indState}$ represents the probability of a patient dying due to cancer if she is diagnosed at time $\indTime$ in state $\indState$.
The second component is for the case when the patient receives a screening test but the outcome is a false negative.
The final component corresponds to the LBCMR value when the patient takes the wait action.
The straightforward approach is to linearize these two objectives by using appropriately selected weight values, $(\objweight_1, \objweight_2)$, to obtain $\objweight_1 h_1 - \objweight_2 h_2$ as the optimization objective.
However, we also experiment with multi-objective optimization techniques to obtain the Pareto frontier.

The model presented in \eqref{eq:dualLP} does not guarantee that the resulting screening policy will be deterministic \citep{puterman2014markov}. Instead, the optimal policies may be \textit{randomized}: there may be more than one optimal action for a given grid point-time tuple. 
In circumstances where randomized policies are unacceptable, which is often the case in medical diagnosis, we can enforce deterministic policy constraints. 
We define binary variables $\varDetPol_{\indTime \indGrid \indAction}$ and incorporate the following constraints into \eqref{eq:dualLP} to ensure deterministic policies:
% We add the following constraints to produce deterministic policies, where only one action is the optimal for any given grid point, time tuple.
\begin{subequations}
\begin{align}
\ & {{\varDualLP_{\indTime \indGrid \indAction}} \leq {\varDetPol_{\indTime \indGrid \indAction}}}, && {\indTime < \timeHorizon, \ \ \indGrid\in \setGridIndex, \ \ \indAction\in \setAction},\label{eq:detforce-a}\\[1.4ex]
\ & {\sum_{\indAction\in \setAction} {\varDetPol_{\indTime \indGrid \indAction}} = 1}, \quad && {\indTime < \timeHorizon, \ \ \indGrid\in \setGridIndex},\label{eq:detforce-b}\\[1.4ex]
\ & {{\varDetPol_{\indTime \indGrid \indAction}} \in \{0,1\}}, && {\indTime < \timeHorizon, \ \ \indGrid\in \setGridIndex, \ \ \indAction\in \setAction}.\label{eq:detforce-c}
\end{align}
\end{subequations}
These constraints ensure that for each grid point-time tuple, exactly one action is recommended. \citet{Cevik2018} also propose constraints to obtain threshold-type policies, which can be useful for clinical practice.
The threshold-type policies rely on the stochastic ordering of the belief states, and ensure that policies follow the threshold levels over the belief states (e.g., take wait action if the patient is at least 99\% healthy -- $\vectBePt = [0.9954, 0.0016, 0.003]$).
However, these constraints come with an extra computational burden, and we, therefore, do not consider them in our analysis.

The approximate model defined in \eqref{eq:dualLP} can be simplified significantly by only considering the \textit{useful} grid points and eliminating those from the formulation that are guaranteed to result in $\varDualLP_{\indTime \indGrid \indAction} = 0, \ \forall \indGrid \in \setGrid, \ \forall \indAction \in \setAction$ at a given time $\indTime$. 
Specifically, some grid points are not visited in certain decision epochs, and they do not contribute to the optimization objectives. This makes it possible to set the corresponding $\varDualLP$-variable value to 0.
The following proposition formalizes this idea:
\begin{proposition}
\label{prop:unused-grids}
If at time $t\in\{1,\ldots,\timeHorizon-1\}$ and grid index $k\in\setGridIndex$ we have \begin{equation*}
    \parfGridTr_{\indGridNew \indGrid}^{\indTime-1 \indAction}\varDualLP_{\indTime \indGrid \indAction}=0,\qquad \forall \indAction \in \setAction,~\indGridNew \in \setGridIndex
\end{equation*}
then, for all $\indAction \in \setAction$, $\varDualLP_{\indTime \indGrid \indAction}=0$.
\end{proposition}
\begin{proof}
\begin{align*}
\intertext{From \eqref{eq:dualLP_eqn2} we have}
\sum_{\indAction \in \setAction}{\varDualLP_{\indTime \indGrid \indAction}} &= \sum_{\indAction \in \setAction} \sum_{\indGridNew \in \setGridIndex} { \parfGridTr_{\indGridNew \indGrid}^{\indTime-1 \indAction}} {\varDualLP_{\indTime-1 \indGridNew \indAction}} 
\intertext{by Proposition~\ref{prop:unused-grids}, this becomes}
\sum_{\indAction \in \setAction}{\varDualLP_{\indTime \indGrid \indAction}}&=0\\
\intertext{from \eqref{eq:dualLP_eqn4}, which states}
\varDualLP_{\indTime \indGrid \indAction}&\geq0\\
\intertext{the only solution is}
\varDualLP_{\indTime \indGrid \indAction}&=0,\qquad \indAction\in\setAction\tag*{\qed}
\end{align*}
\end{proof}

We take that the grid point $\indGrid$ is not useful at time $\indTime$ if it satisfies the condition in Proposition~\ref{prop:unused-grids}. 
In the model defined by \eqref{eq:dualLP}, for each unuseful grid point, we can omit the $|\setAction|$ many corresponding $\varDualLP_{\indTime\indGrid\indAction}$ variables and their associated constraints given in \eqref{eq:dualLP_eqn2}. 
In the deterministic model, we can additionally omit $|\setAction|$ many binary $\varDetPol_{\indTime \indGrid \indAction}$ variables, as well as the two corresponding deterministic policy constraints given by \eqref{eq:detforce-a} and~\eqref{eq:detforce-b}.

%%%%%%%%%%%%%%%%%%%%%%%%%%%%%%%%%%%%%%%%%%%%%%%%%%%%%%%%%%%%%%%%%%%%%%%%%%%%%%%%%%%%
%%%%%%%%%%%%%%%%%%%%%%%%%%%%%%%%%%%%%%%%%%%%%%%%%%%%%%%%%%%%%%%%%%%%%%%%%%%%%%%%%%%%
\section{Numerical Study}\label{sec:Results}
%%%%%%%%%%%%%%%%%%%%%%%%%%%%%%%%%%%%%%%%%%%%%%%%%%%%%%%%%%%%%%%%%%%%%%%%%%%%%%%%%%%%
%%%%%%%%%%%%%%%%%%%%%%%%%%%%%%%%%%%%%%%%%%%%%%%%%%%%%%%%%%%%%%%%%%%%%%%%%%%%%%%%%%%%
In this section, we discuss the parameter estimation procedures and present the results from our numerical study.  
Through detailed experiments, we first assess the performance of the grid-based approximations for the unconstrained POMDP model of the breast cancer screening problem.
Then, we present results for the multi-objective POMDP and CPOMDP models.
The optimization models are implemented in Python using Gurobi 9.0 and experiments are performed on a Linux machine with an Intel i7-8700K processor and 64 GB of system RAM.

%%%%%%%%%%%%%%%%%%%%%%%%%%%%%%%%%%%%%%%%%%%%%%%%%%%%%%%%%%%%%%%%%%%%%%%%%%%%%%%%%%%%
\subsection{Parameter Estimation}\label{sec:parameter_estimation}
%%%%%%%%%%%%%%%%%%%%%%%%%%%%%%%%%%%%%%%%%%%%%%%%%%%%%%%%%%%%%%%%%%%%%%%%%%%%%%%%%%%%
An important input for our models is the disutility of various screening and diagnostic outcomes. 
%for screening, TP, and FP results associated with mammography, mammography+MRI, and mammography+ultrasound as follows: 
We consider the following disutility values: 0.5 days for mammography, 2.5 days for mammography+MRI, 1 day for mammography+ultrasound, 14 days for a TP and 28 days for a FP, as FP has been found to have a higher disutility than TP~\citep{sandikci2018screening, ayer2012}.
Furthermore, we set the terminal rewards for a patient at age 100, $\rewTerminal_{\indState}$, to 2.5, 1.2, and 0.5 years for being healthy, having in-situ cancer, and having invasive cancer, respectively~\citep{ayer2012}.
We note that there could be many factors that contribute to an increased likelihood of developing breast cancer, which need to be considered in categorizing patients into risk groups. 
The factors that are found important for increased breast cancer risk include breast density, family history, and BRCA1 and BRCA2 gene mutations~\citep{acs2019early}. Accordingly, we group patients under one of two categories: average-risk (AR) and high-risk (HR). 
Similar patient groupings were also considered in previous studies.
For instance, \citet{ayer2012} define an HR patient as someone with a family history of breast cancer.
Overall, similar to \citet{Cevik2018}, HR patients are assumed to be 2--4 times more likely to develop breast cancer compared to their AR counterparts. 
The risk of developing breast cancer is mainly encoded by the transition probability matrices: women classified as high-risk patients are more likely to transition to unhealthy states than women who are average-risk. 
The corresponding transition probabilities for AR and HR patients are obtained from the studies of \citet{Maillart2008} and \citet{Cevik2018}, and the observation probabilities (i.e., sensitivity and specificity values for each action) are obtained from the studies of \citet{ayer2012} and \citet{sandikci2018screening}.

In terms of test performance, mammography + MRI has the highest sensitivity and the lowest specificity, whereas mammography + ultrasound has the highest specificity among these three screening modalities.
Post-cancer life expectancy values are also estimated similarly to the estimations in \citet{ayer2012}, which are based on the methods by \citet{Arias2010} and \citet{Siegel2014}. 
The cost of each screening action follows the estimates from \citet{ccauglayan2018assess}, adjusted at a 3\% rate of inflation and rounded to the nearest dollar to get the cost in 2022. 
After adjusting for inflation, the cost of mammography, mammography + MRI, and mammography + ultrasound are estimated as \$134, \$1,752, and \$243, respectively. 
In the numerical experiments with constrained POMDP models, the budget levels are set as \$350, \$850, and \$1700, for small, average and large budgets.
We note that these budget constraints represent the limits on average amount of money spent on cancer screening procedures throughout an individual patient's lifetime. 
To ensure that the weighted multi-objective optimization responds properly to both objective terms, the scale factors $\scaleFactor_1$ and $\scaleFactor_2$ are assigned such that the scaled values of QALYs and LBCMR are approximately within $[0, 1]$. The results in \citet{Cevik2018} and \citet{sandikci2018screening} reveal that QALY values are generally below or just larger than 40. Accordingly, $\scaleFactor_1$ is set to $40^{-1}$. A preliminary analysis revealed that LBCMR generally does not exceed 10\%. Accordingly, we set $\scaleFactor_2$ to 10. 
% The starting belief states for the patients are estimated by using the breast cancer risk assessment tool provided on NCI's website\footnote{http://www.cancer.gov/bcrisktool/}.
The starting belief states for AR and HR patients can be estimated by using the breast cancer risk assessment tool provided on NCI's website\footnote{http://www.cancer.gov/bcrisktool/} and they typically differ considerably based on the patient characteristics, with HR patients having significantly worse belief states at a given age than AR patients in most cases~\citep{ayer2012, sandikci2018screening}.

%%%%%%%%%%%%%%%%%%%%%%%%%%%%%%%%%%%%%%%%%%%%%%%%%%%%%%%%%%%%%%%%%%%%%%%%%%%%%%%%%%%%
\subsection{Performance Evaluation for Solution Algorithms}
\label{sec:performance_evaluation}
%%%%%%%%%%%%%%%%%%%%%%%%%%%%%%%%%%%%%%%%%%%%%%%%%%%%%%%%%%%%%%%%%%%%%%%%%%%%%%%%%%%%
% The performance of the grid based approximation algorithms varies as the grid set changes.
The performance of the grid-based approximation can be significantly impacted by the grid set composition.
Typically, as the grid size increases, policies generated by the approximation methods get closer to the ones generated by the exact solution methods (e.g., Monahan's exhaustive enumeration algorithm).
Accordingly, we first examine the performance of the grid-based approximations by using \citet{lovejoy1991computationally}'s lower bound (LB) and upper bound (UB) mechanisms over the unconstrained POMDP model for our breast cancer screening problem. 
Note that the latter provides the foundations of our CPOMDP model provided in \eqref{eq:dualLP}.
We adopt the implementations provided by \citet{kavaklioglu2022scalable} for the LB and UB methods.

Table~\ref{table:gap_table} summarizes the performance for different grid sets based on an AR patient and by using the single objective model obtained by setting $\objweight_1=1$ and $\objweight_2=0$ (i.e., QALY maximization objective).
We provide two sets of results, one for a single belief state, $\vectBePt = [0.9954, 0.0016, 0.003]$, and the other is optimality gap statistics over 100 belief states randomly sampled according to the same resolution threshold vector ($\vectResoThreshold = [0.96, 0.80, 0]$) used for constructing all the grid sets.
The optimality gap value for a specific belief state is calculated as $\text{Gap}~(\%) = (\text{UB}-\text{LB})/|\text{LB}|$.

These results show that the gap values may decrease considerably as we increase the number of grid points in the variable resolution uniform grid construction method.
In terms of CPU run times, the UB method is significantly worse than the LB method and Monahan's algorithm (ME).
Note that both the LB method and ME rely on the piecewise linear value function representation for POMDPs (i.e., $\funcOptV^*(\vectBePt) = \max_{\boldsymbol{\alpha} \in \Gamma} \{\vectBePt \cdot \boldsymbol{\alpha}\}$ where $\Gamma$ represents the set of $\alpha$ vectors), and effectively employ $\alpha$-vectors to represent the model solutions.
The LB method achieves similar values to those generated using ME even for small grid sets.
On the other hand, the UB method repeatedly calculates $\beta$-values to achieve a reduction of the POMDP to a corresponding MDP, which leads to a significant increase in computational overhead.
While the gap significantly tightens as we approach $\setGrid_3$, there is a negligible difference between the results obtained for $\setGrid_3$ and $\setGrid_4$.
As such, we set $\setGrid_3$, which uses $\vectResoVal = [500, 50, 5]$ and $\vectResoThreshold = [0.96, 0.8, 0]$, as our \textit{default} grid set in the subsequent experiments.

\begin{table}[!ht]
    \centering
    % \caption{Effect of different resolution values for the variable resolution uniform grid construction method on an average risk patient.}
    \caption{Summary performance values for different unconstrained POMDP algorithms obtained for an AR patient using different grid sets.}
    \label{table:gap_table}
    \resizebox{\columnwidth}{!}{
    
        \begin{tabular}{lrrrllrrrrr}
            \toprule
            \multicolumn{2}{c}{Grid Specification} & \multicolumn{2}{c}{LB} & \multicolumn{2}{c}{ME} & \multicolumn{2}{c}{UB} & \multicolumn{3}{c}{Gap (\%)} \\
            \cmidrule(lr){1-2}\cmidrule(lr){3-4}\cmidrule(lr){5-6}\cmidrule(lr){7-8}\cmidrule(lr){9-11}
            Grid ID and $\vectResoVal$ & $|\setGrid|$ &   QALYs &    CPU &    QALYs &    CPU &   QALYs &     CPU &   Min &  Mean &  Max \\
            \midrule
 $\setGrid_1$: (100, 25, 5) &    51 & 40.244 &    5.9 &      40.252 &  1669.9 & 40.338 &    71.0 &   0.0 & 0.082 & 0.215 \\
 $\setGrid_2$: (250, 50, 5) &   144 & 40.250 &   31.8 &       - &       - & 40.270 &   511.7 &   0.0 & 0.016 & 0.047 \\
 $\setGrid_3$: (500, 50, 5) &   309 & 40.252 &  214.6 &  - & - & 40.255 &  2118.5 &   0.0 & 0.003 & 0.009 \\
$\setGrid_4$: (1000, 50, 5) &   939 & 40.252 & 5476.4 &       - &       - & 40.252 & 19231.0 &   0.0 & 0.001 & 0.004 \\
            \bottomrule
        \end{tabular}
    }
    \begin{tablenotes}
      \footnotesize
      \item $^{*}$: QALY values are for the target belief state $\vectBePt = [0.9954, 0.0016, 0.003]$
      \item $^{\dagger}$: resolution threshold vector is $\vectResoThreshold = [0.96, 0.8, 0]$
    %   \item $^{\ddagger}$: gap values are calculated over 172 evaluation grid points
      \item $^{\ddagger}$: gap values are calculated over 100 randomly sampled evaluation grid points
    %   \item ME algorithm takes 163 seconds and leads to 40.20 QALYs for the target belief
    \end{tablenotes}
\end{table}

We next demonstrate the impact of Proposition~\ref{prop:unused-grids} for various grid set sizes.
Specifically, we consider the unconstrained POMDP model by using the corresponding integer programming model as formulated by~\eqref{eq:dualLP} with the single objective $h_1$, including the deterministic policy constraints, but excluding the budget constraints. 
Because our multi-objective CPOMDP models are constructed based on a similar formulation, this analysis sheds light on the corresponding run times as well.
Figure~\ref{fig:used_gridpoints_number} provides the percentage of useful grids at each decision epoch, and Figure~\ref{fig:used_gridpoints_time} compares the CPU run times on a logarithmic scale between the two configurations: using all of the grid points and using only the useful grid points. 
At the first decision epoch, as specified by the choice of $\parBeDistn$, all the grid points are used.
However, in the subsequent decision epochs, only a fraction of grid points are used/visited.
We also find that the percentage of useful grid points decreases with $|\setGrid|$. 
Across all the grid set sizes, we observe a significant performance boost in terms of run times, which tends to increase as $|\setGrid|$ increases.

\begin{figure}[!ht]
	%\vspace{-.5em}
	\centering
	%
	% -cb extension on images specifies the versions more easily accessile to colorblind. Remove -cb for originals.
	%
	\subfloat[\% of grid points used \label{fig:used_gridpoints_number}]{\includegraphics[width=0.5\textwidth]{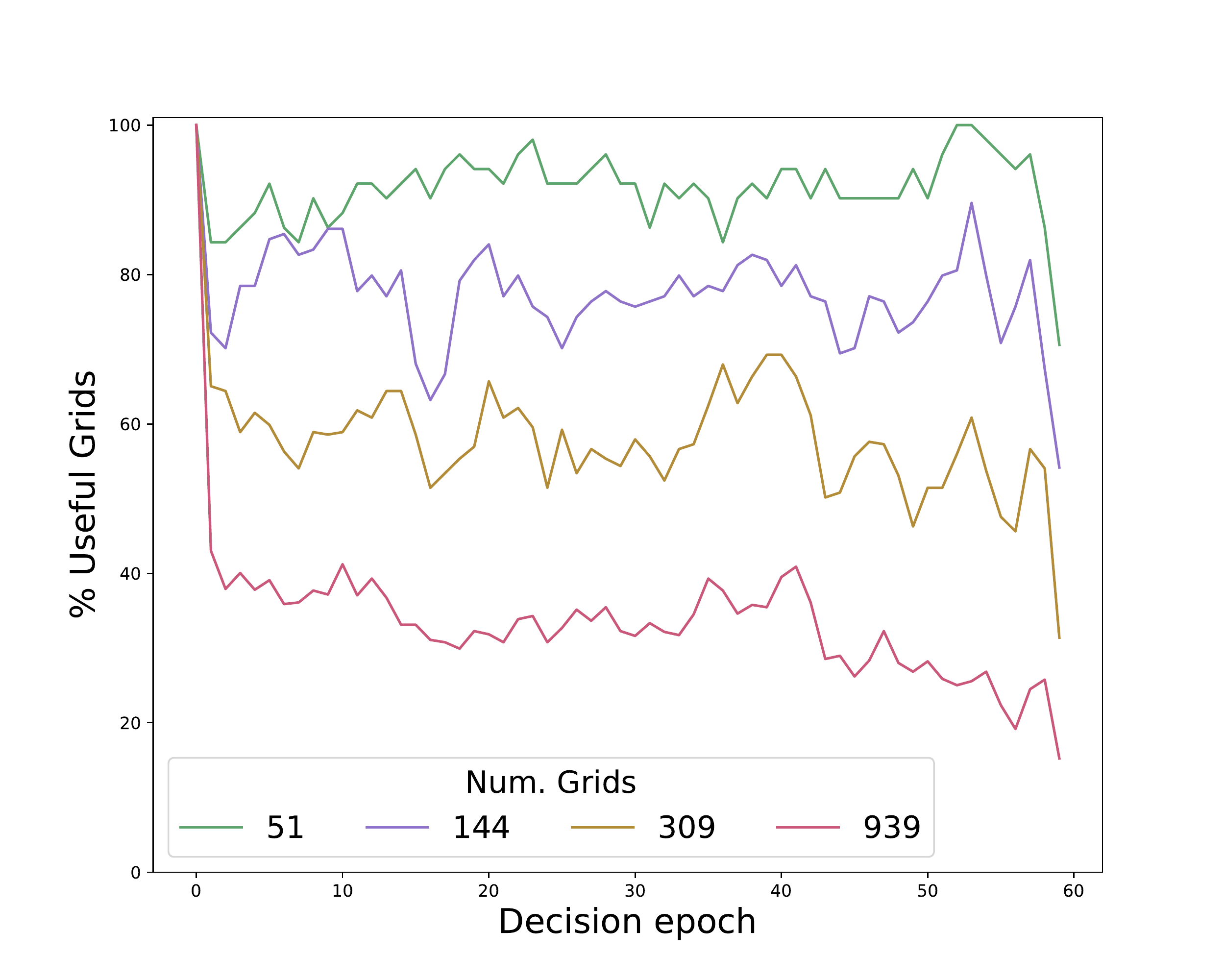}}\hfill
	\subfloat[CPU run time comparison \label{fig:used_gridpoints_time}]{\includegraphics[width=0.5\textwidth]{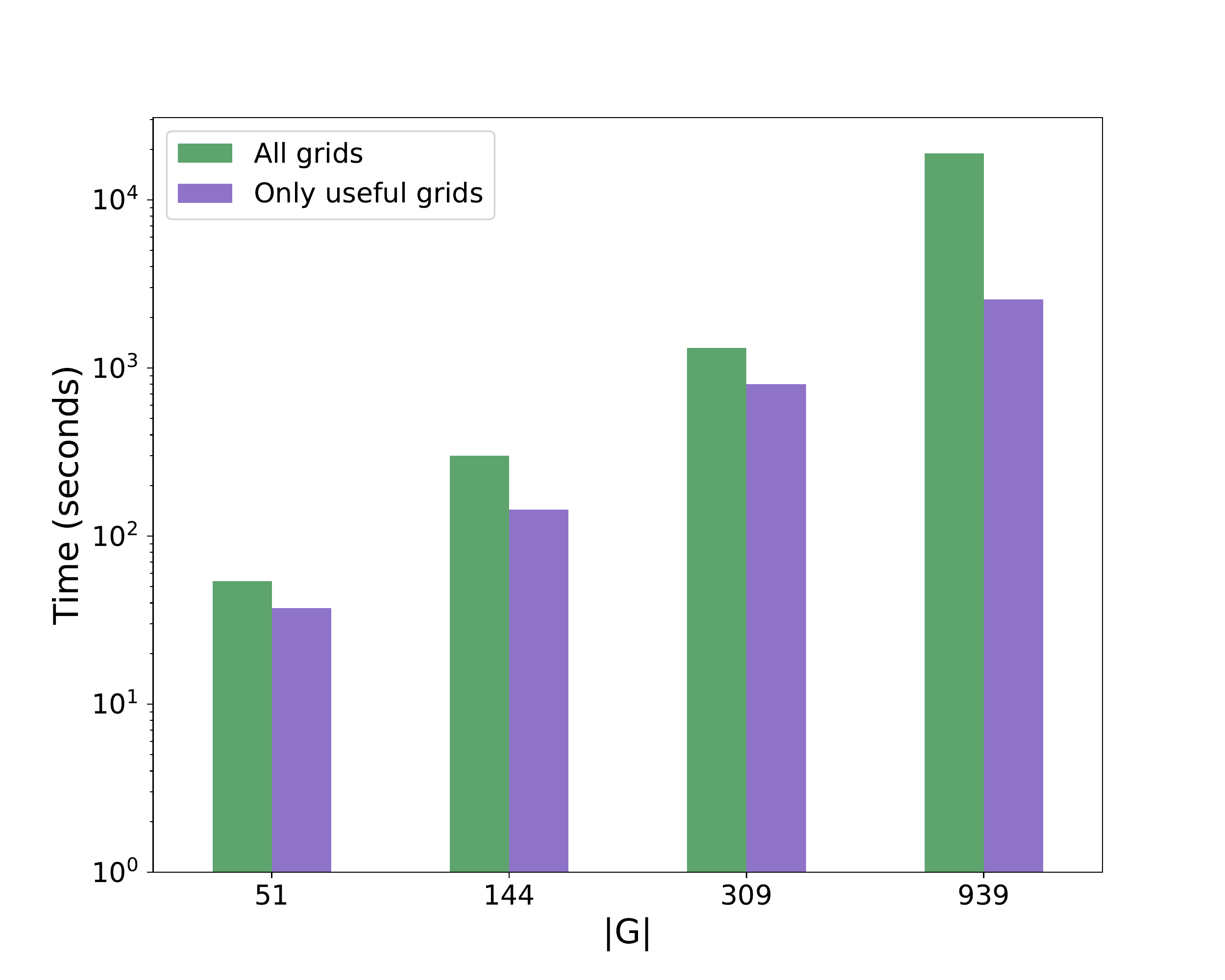}}
	\caption{Grid point elimination analysis results for unconstrained POMDP model solved using the corresponding integer programming formulation.}
	\label{fig:used_gridpoints}
\end{figure}

%%%%%%%%%%%%%%%%%%%%%%%%%%%%%%%%%%%%%%%%%%%%%%%%%%%%%%%%%%%%%%%%%%%%%%%%%%%%%%%%%%%%%%%%%%%%%%%%
%%%%%%%%%%%%%%%%%%%%%%%%%%%%%%%%%%%%%%%%%%%%%%%%%%%%%%%%%%%%%%%%%%%%%%%%%%%%%%%%%%%%%%%%%%%%%%%%
% \subsection{Assessing the QALYs - lifetime risk trade-off}
\subsection{Multi-objective CPOMDP Results}
\label{sec:mop_results}
%%%%%%%%%%%%%%%%%%%%%%%%%%%%%%%%%%%%%%%%%%%%%%%%%%%%%%%%%%%%%%%%%%%%%%%%%%%%%%%%%%%%%%%%%%%%%%%%
%%%%%%%%%%%%%%%%%%%%%%%%%%%%%%%%%%%%%%%%%%%%%%%%%%%%%%%%%%%%%%%%%%%%%%%%%%%%%%%%%%%%%%%%%%%%%%%%
In our experiments with multi-objective models, we first assess the trade-offs between the optimization objectives.
Specifically, three representative weight combinations are considered to test the multi-objective optimization model.
Intuitively, both extremes are considered: the QALY maximization objective, obtained by setting $\objweight = [1,~0]$, as well as the LBCMR minimization objective, obtained by setting  $\objweight = [0,~1]$.
The final weight vector involves a combination of maximizing QALYs while minimizing LBCMR, which is referred to as the \textit{weighted} objective. 
Our preliminary analysis indicates that an arbitratily selected weight combination can perform similarly to QALY maximization or LBCMR minimization objectives. 
We identify $\objweight = [0.933,~0.067]$ as the weight combination that leads to noticeably different QALYs and LBCMR values than the single objective counterparts.
Below, after establishing baseline results for our multi-objective CPOMDP model by solving its unconstrained version (i.e., without budget restrictions), we investigate the impact of the budget constraints on the resulting policies. 
We then provide a sensitivity analysis for important model parameters such as the disutility of supplemental screenings and the screening costs.

% We find that, as the belief state worsens (i.e., as the probability of being cancer-free decreases), QALY values decrease and the risk values increase, which is intuitive.
Table~\ref{tab:MOP_UB_ress} shows the baseline results for the unconstrained breast cancer screening problem considering AR and HR patients.
For each patient type, we consider five different belief states to report the QALYs and LBCMR values. 
These belief states only specify the starting health status estimates of the patients, and some extreme cases are included (e.g., see the belief state in the last row of Table~\ref{tab:MOB_UB_table_AR}) to assess the impact of the starting belief state on screening outcomes.
We find that there is a noticeable trade-off between the QALY maximization policies (i.e., the ones obtained by solving the CPOMP model using only QALY maximization objective) and LBCMR minimization policies (i.e., the ones obtained by solving the CPOMP model using only LBCMR minimization objective).
Specifically, across different belief states for AR patients, we observe that aggressive screening strategies lead to a decrease in LBCMR by approximately 0.51\%, but these strategies also result in a drop of 0.36 QALYs (4.3 months).
This can be explained by the various disutilities (e.g., FP test results) associated with screening.
We observe similar trends for HR patients as well.
As expected, HR patients experience lower QALYs and a higher LBCMR due to the elevated likelihood of developing breast cancer.
On average, HR patients have a slightly lower QALY loss (3.7 months) and a larger LBCMR reduction (0.58\%) when compared to AR patients.
This is mainly because the disutilities associated with screening are suppressed by the increased benefit of screening in terms of breast cancer mortality reduction, which leads to screenings having a more positive impact on QALYs for HR patients.

\begin{table}[!htb]
\centering
\caption{Multi-objective POMDP results obtained for fixed objective weights (analysis exclude budget limits; LBCMR values are in \%).}
\label{tab:MOP_UB_ress}
    \subfloat[AR patient\label{tab:MOB_UB_table_AR}]{
    \resizebox{0.959\textwidth}{!}{
    
    \begin{tabular}{crrrrrr}
        \toprule
        & \multicolumn{2}{c}{Min. LBCMR} & \multicolumn{2}{c}{Weighted} & \multicolumn{2}{c}{Max. QALYs} \\
        \cmidrule(lr){2-3}\cmidrule(lr){4-5}\cmidrule(lr){6-7}
        Belief state &                    QALYs & LBCMR &                    QALYs & LBCMR &                    QALYs & LBCMR \\
    \midrule
        1.0000, 0.0000, 0.0000 &                  39.949 &    3.780 &                  40.278 &    3.897 &                  40.317 &    4.301 \\
        0.9954, 0.0016, 0.0030 &                  39.886 &    3.960 &                  40.216 &    4.079 &                  40.255 &    4.478 \\
        0.9850, 0.0060, 0.0090 &                  39.783 &    4.342 &                  40.110 &    4.463 &                  40.149 &    4.854 \\
        0.9750, 0.0120, 0.0130 &                  39.715 &    4.653 &                  40.032 &    4.766 &                  40.072 &    5.164 \\
        0.9500, 0.0180, 0.0320 &                  39.426 &    5.715 &                  39.733 &    5.826 &                  39.770 &    6.200 \\
    \bottomrule
    \end{tabular}
    }}\\*[0.25em]
    \subfloat[HR patient. \label{tab:MOB_UB_table_HR}]{
    \resizebox{0.959\textwidth}{!}{
    
    \begin{tabular}{crrrrrr}
        \toprule
        & \multicolumn{2}{c}{Min. LBCMR} & \multicolumn{2}{c}{Weighted} & \multicolumn{2}{c}{Max. QALYs} \\
        \cmidrule(lr){2-3}\cmidrule(lr){4-5}\cmidrule(lr){6-7}
        Belief state &                    QALYs & LBCMR &                    QALYs & LBCMR &                    QALYs & LBCMR\\
        \midrule
            0.9890, 0.0030, 0.0080 &                  39.428 &    7.441 &                  39.720 &    7.523 &                  39.754 &    8.050 \\
            0.9755, 0.0085, 0.0160 &                  39.306 &    7.900 &                  39.587 &    7.974 &                  39.623 &    8.508 \\
            0.9500, 0.0180, 0.0320 &                  39.065 &    8.794 &                  39.337 &    8.867 &                  39.369 &    9.370 \\
            0.9300, 0.0280, 0.0420 &                  38.908 &    9.415 &                  39.173 &    9.487 &                  39.205 &    9.978 \\
            0.9200, 0.0330, 0.0470 &                  38.829 &    9.725 &                  39.092 &    9.797 &                  39.123 &   10.283 \\
        \bottomrule
        \end{tabular}
    }}\\
\end{table}

We find that the weighted objective manages to balance the QALYs and LBCMR values.
Specifically, for all of the considered belief states for both AR and HR patients, the QALYs and LBCMR for the weighted objective tend to be very close to their optimal values, as given by the maximize QALYs and minimize LBCMR objectives, respectively.
For AR patients, following the policy from the weighted objective model results in an average increase of just 0.12\% to the LBCMR over the policy from the LBCMR minimization model, and with a decrease in QALYs on average of just 14 days compared to the policy by QALY maximization model.
On the other hand, the LBCMR minimization policy would decrease QALYs on average by almost four months, and the QALY maximization policy would increase the risk on average by 0.39\%. 
The results for HR patients show an even bigger improvement, with the policy by the weighted objective model adding just 0.07\% to the LBCMR, and decreasing the QALYs by just 12 days compared to the single objective counterparts. 
In contrast, for the HR patients, switching to the QALY maximization policy would increase the risk value on average by 0.51\%, and switching to the LBCMR minimization policy would decrease the QALYs by about 3.3 months. 
This highlights the benefits of the multi-objective approach, as it provides an opportunity to balance the different objectives. 
That is, when a CPOMDP model only considers the QALY maximization objective, it is willing to accept a large increase in risk for minimal QALY gains. 
Similarly, policies which only aim to minimize LBCMR are willing to sacrifice a great deal of QALYs for minimal improvement in LBCMR.
The policies obtained by using a weighted objective achieve neither maximal QALYs nor minimal LBCMR values, but they provide a balance between the two objectives.

We next experiment with the budget-constrained multi-objective CPOMDP models for AR and HR patients. 
Specifically, a small budget (\$350), an average budget (\$850), and a large budget (\$1,700) are each considered and their impact on the resulting policies is compared. 
% These budget constraints represent the average amount of money spent on cancer screening procedures throughout a patient's lifetime.
We adopt two different approaches to solve the resulting approximate model: weighted-sum method by selecting a large number (100) of distinct weight values to solve the model using the weighted objective, and applying the $\epsilon$-constraint method~\citep{mavrotas2009effective}.
In the weighted-sum method, for easier and more meaningful selection of weight values, we first bring both objective terms to same scale by using the scale factors $\scaleFactor_1$ and $\scaleFactor_2$.
These are determined in a way that the scaled values of QALYs and LBCMR are approximately within the range $[0, 1]$.
For instance, previous studies show that the QALY values are typically around 40 years and LBCMR values are around 0.1 (i.e., 10\%), so 
$\scaleFactor_1 = 40$ and $\scaleFactor_2 = 10$ can be used in our analysis.
In the $\epsilon$-constraint method, we maximize the single objective $h_1$, subject to the constraint that $-h_2$ is no less than a particular value (i.e., we maximize QALYs subject to a pre-calculated maximum tolerable LBCMR).
Upper limit for $-h_2$ is changed iteratively to generate nondominated solutions.
Figures~\ref{fig:epsilon_constraint_vs_weighted_AR} and \ref{fig:epsilon_constraint_vs_weighted_HR} show the results for AR and HR patients, respectively.
As expected, repeatedly solving the model using the weighted objective --- a naive approach to multi-objective optimization --- fails to identify the majority of the values on the Pareto frontier.
On the other hand, the $\epsilon$-constraint method performs well in this regard and helps to quantify the trade-offs between the QALY maximization and LBCMR minimization objectives.

\begin{figure}[!ht]
    \centering
    \subfloat[Budget: \$350. \label{fig:AR-epsilon-350}]{\includegraphics[width=0.33\textwidth]{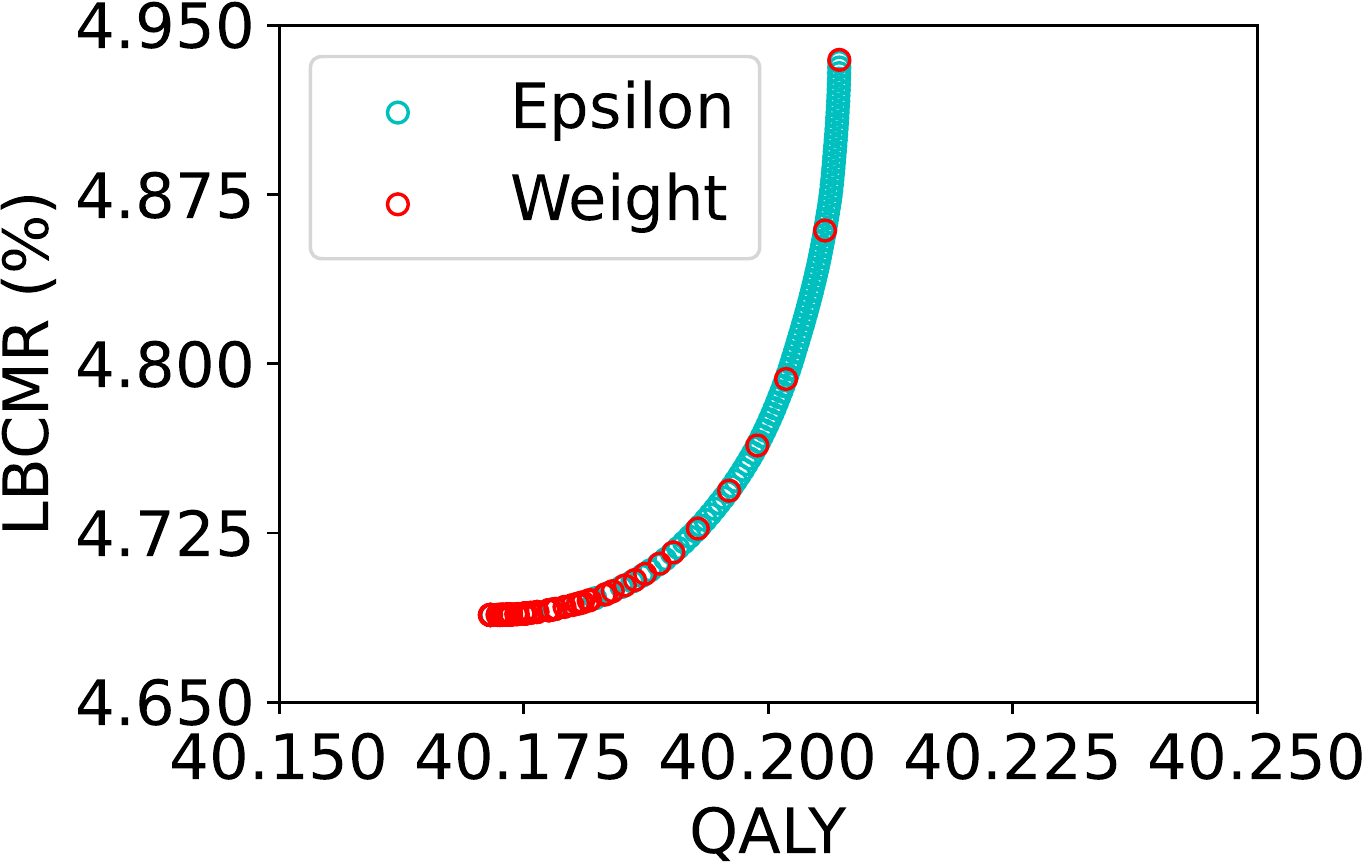}}
    \hfill
    \subfloat[Budget: \$850]{\includegraphics[width=0.33\textwidth]{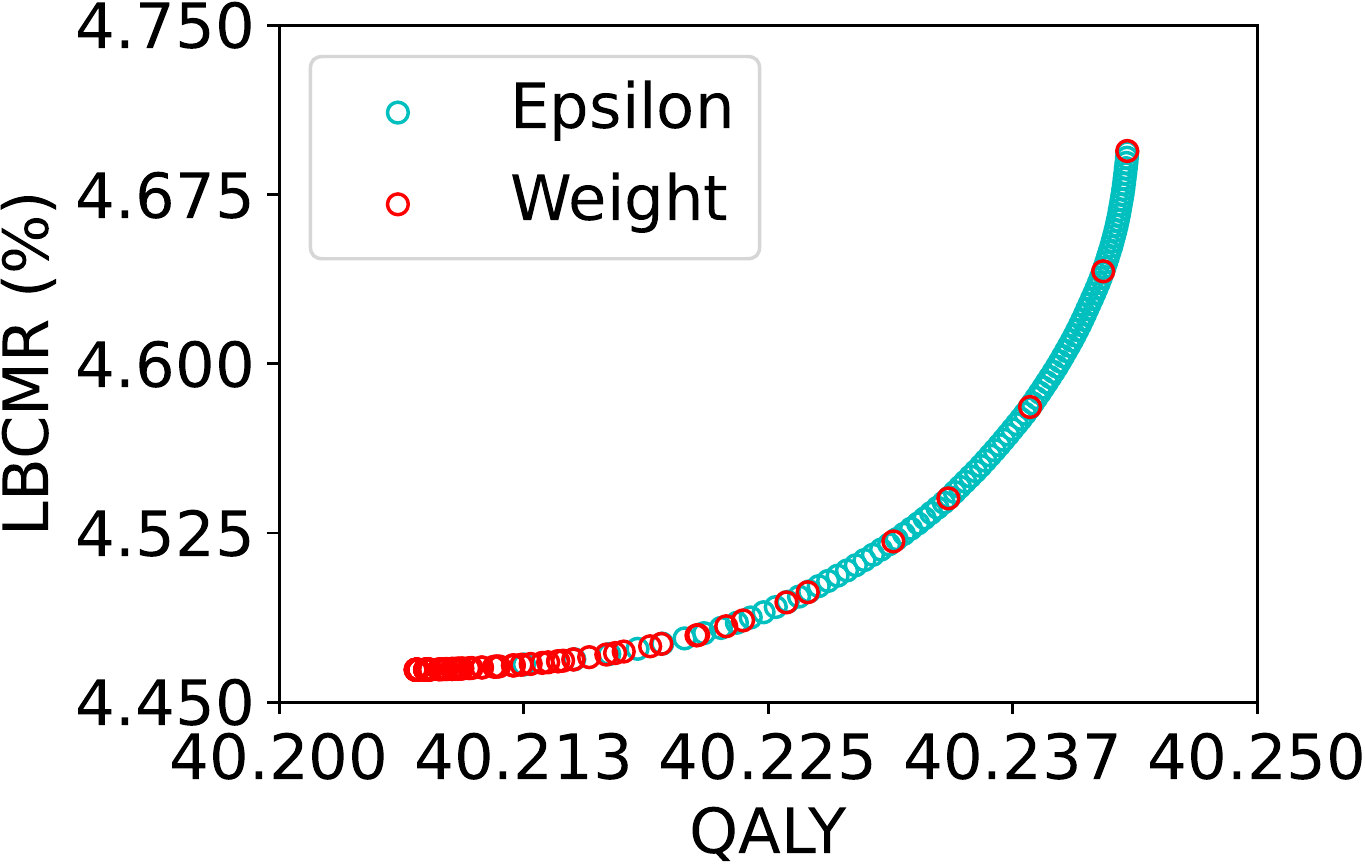}}
    \subfloat[Budget: \$1700]{\includegraphics[width=0.33\textwidth]{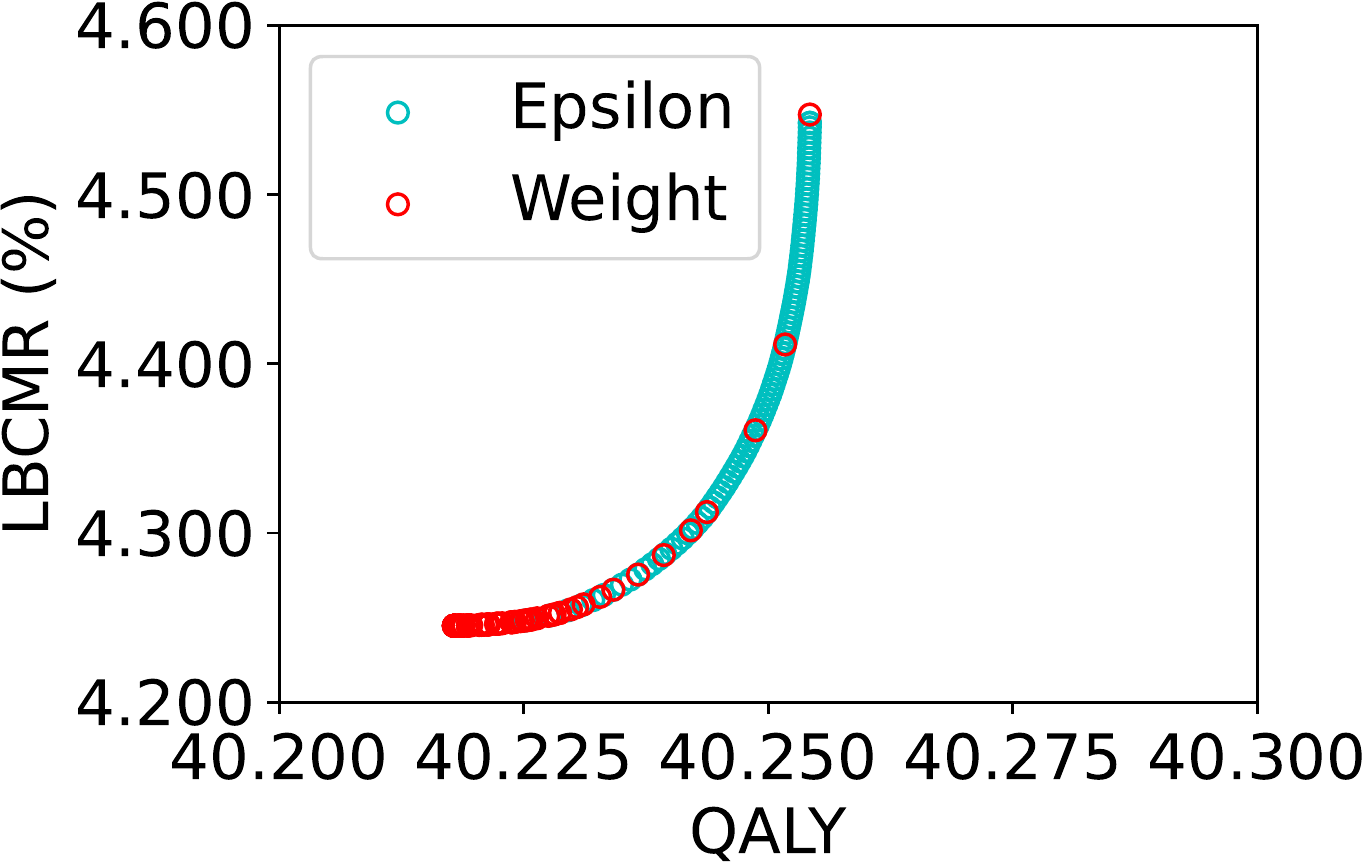}}
    \caption{Results with weighted objective and $\epsilon$-constraint methods for various budget levels for AR patients (patient starting belief state is [0.9954, 0.0016, 0.003]).}
    \label{fig:epsilon_constraint_vs_weighted_AR}
\end{figure}

\begin{figure}[!ht]
    \centering
    \subfloat[Budget: \$350\label{fig:HR-epsilon-350}]{\includegraphics[width=0.33\textwidth]{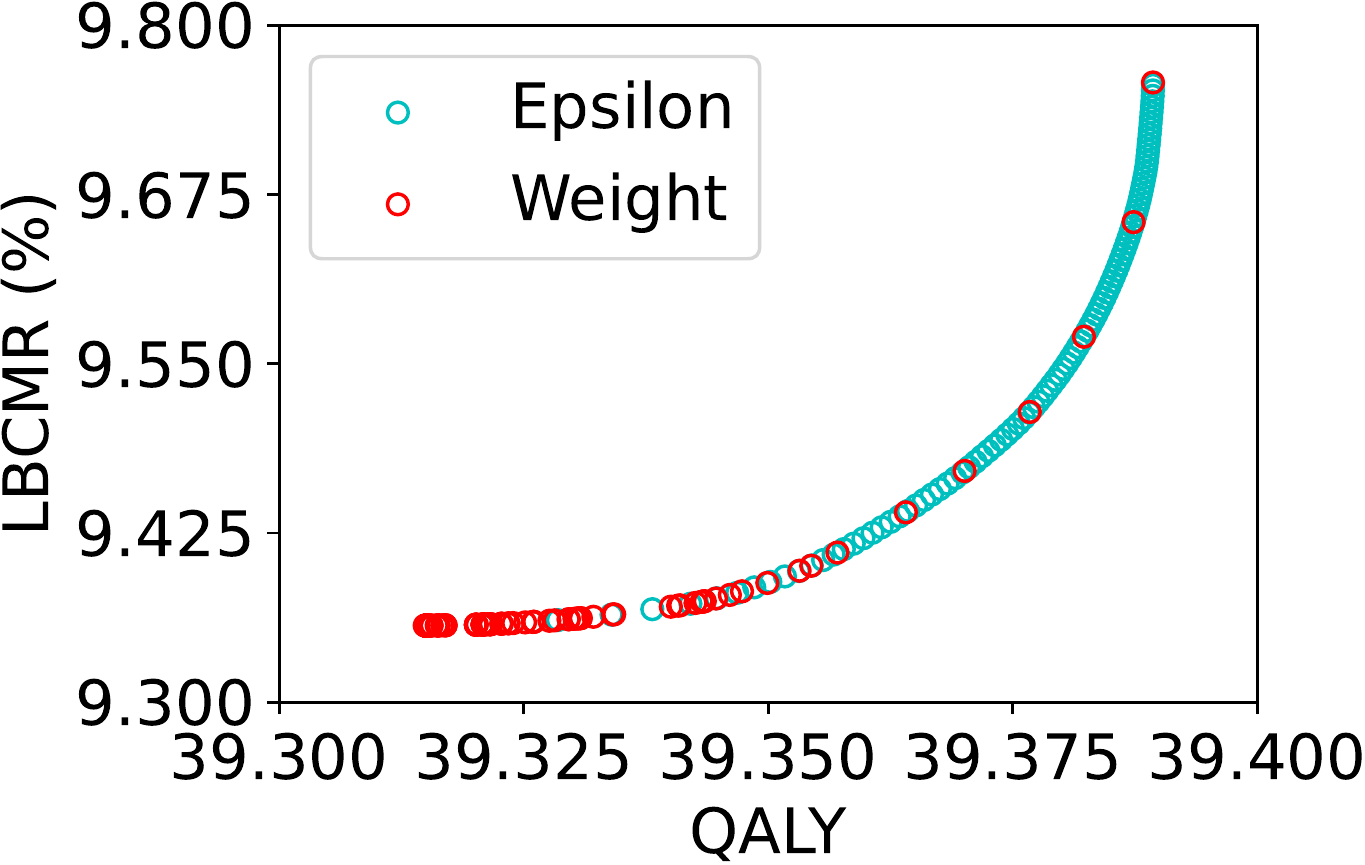}}
    \hfill
    \subfloat[Budget: \$850]{\includegraphics[width=0.33\textwidth]{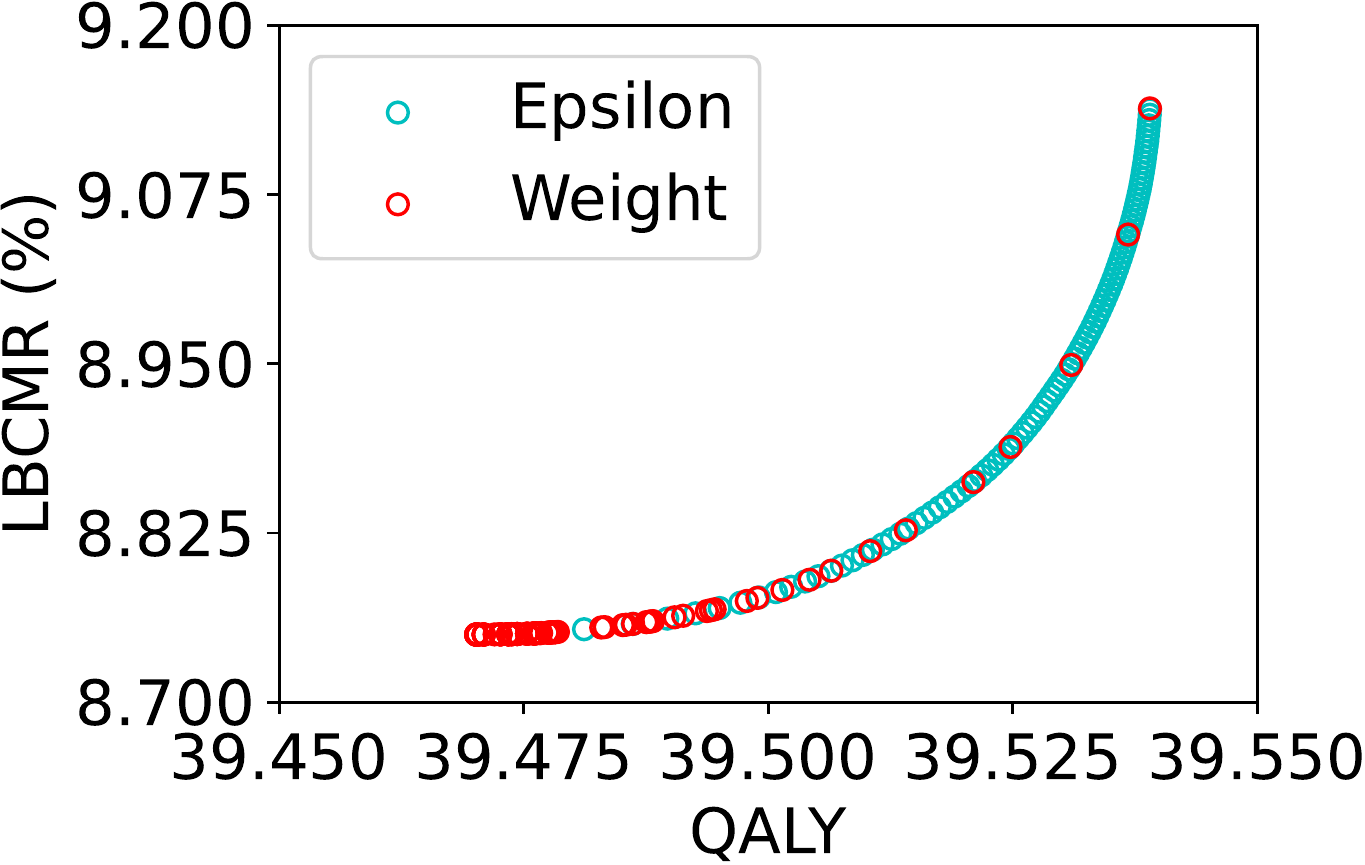}}
    \hfill
    \subfloat[Budget: \$1700]{\includegraphics[width=0.33\textwidth]{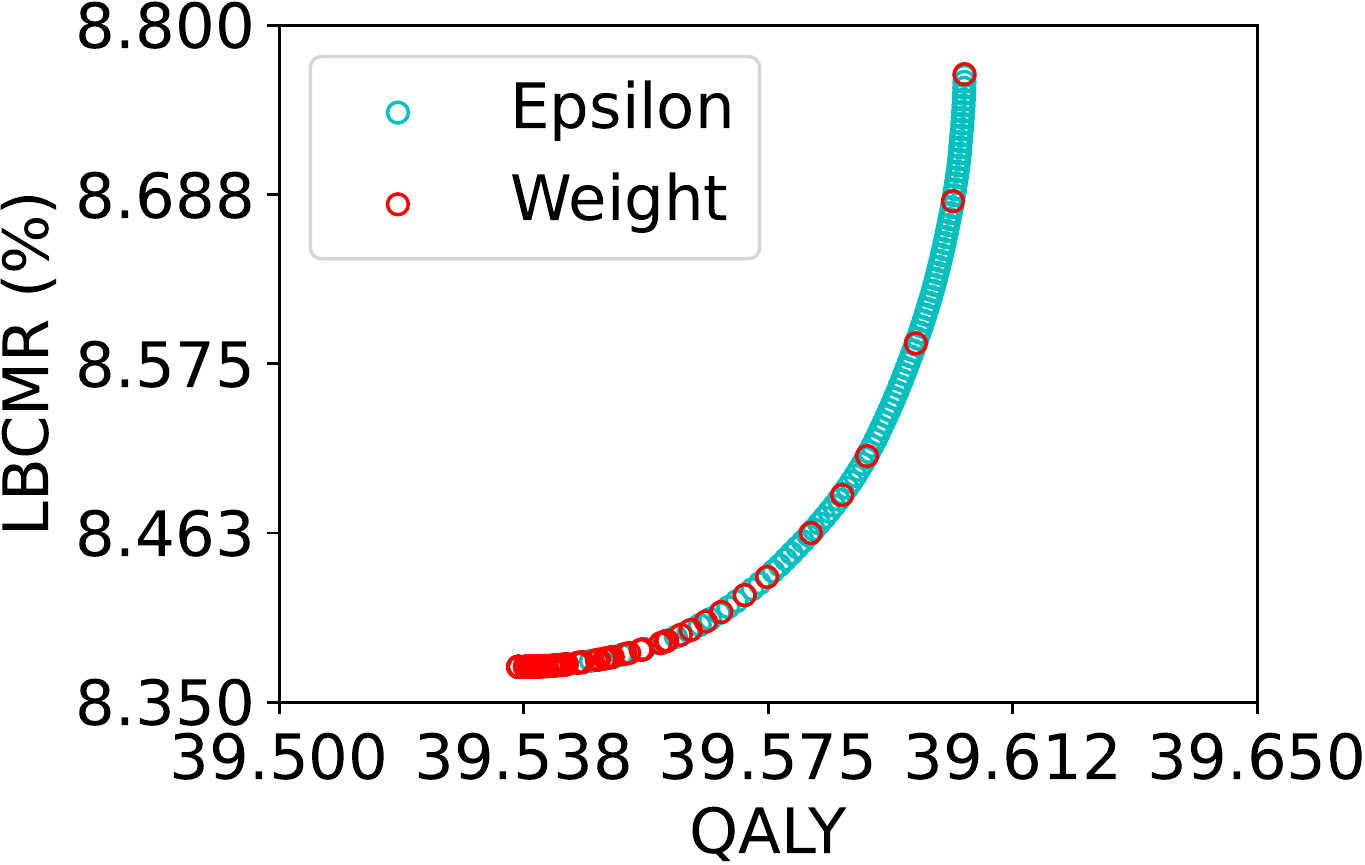}}
    \caption{Results with weighted objective and $\epsilon$-constraint methods for various budget levels for HR patients (patient starting belief state is [0.9755, 0.0085, 0.016]).}
    \label{fig:epsilon_constraint_vs_weighted_HR}
\end{figure}

The Pareto fronts in figures~\ref{fig:epsilon_constraint_vs_weighted_AR} and~\ref{fig:epsilon_constraint_vs_weighted_HR} reveal that, for each of the considered budgets, the selected weight vector has a substantial impact on LBCMR, but minimal effect on QALYs. 
% This is especially true for AR patients. 
Figure~\ref{fig:AR-epsilon-350} shows that, for AR patients with $\parBudgetLim = \$350$, QALY values range between 40.17 and 40.21 (i.e., $\pm$ 0.04 years/14 days) and risk values range between 4.69\% and 4.94\% (i.e., $\pm$ 0.25\%).
Similarly, for $\parBudgetLim = \$1,700$, QALY values range between 40.22 and 40.25 (i.e., $\pm$ 0.03 years/11 days) and risk values range between 4.25\% and 4.55\% (i.e., $\pm$ 0.30\%).
For HR patients, the impact on QALYs is more pronounced, where, in Figure~\ref{fig:HR-epsilon-350}, the difference in QALYs is about 27 days. 
This trend continues for the higher budgets, where the average gap in QALYs is about 25 days, and the average gap in risk is about 0.39\%.
This can again be attributed to the elevated likelihood of developing breast cancer for HR patients, that is, additional screenings can be especially helpful in reducing the LBCMR, even for relatively healthy patients, whereas QALY maximization often involves avoiding screening procedures for patients with healthier belief states. 
% Policy evaluations, explored in Section~\ref{sec:screening-recommendation-scenarios}, reveal the different screening regimes prescribed across the varying weight vectors.

%%%%%%%%%%%%%%%%%%%%%%%%%%%%%%%%%%%%%%%%%%%%%%%%%%%%%%%%%%%%%%%%%%%%%%%%%%%%%%%%%%
% \subsection{Simulations}
\subsection{Policy Evaluation Results}\label{sec:policy-evaluation-results}
%%%%%%%%%%%%%%%%%%%%%%%%%%%%%%%%%%%%%%%%%%%%%%%%%%%%%%%%%%%%%%%%%%%%%%%%%%%%%%%%%%

Figures~\ref{fig:ar-cpomdp-vs-rule} and~\ref{fig:hr-cpomdp-vs-rule} compare the performance of multi-objective CPOMDP policies obtained for three distinct objective weight configurations against the following fixed-interval (i.e., rule-based) policies: annual mammography screenings (A), biannual mammography screenings (B), biannual mammography screenings supplemented with MRI (B+R), and biannual mammography screenings supplemented with ultrasound (B+U). 
Note that these results are collected by using a discrete-event simulation model that involves simulating the lifetime of 100,000 patients.

\begin{figure}[!ht]
    \centering
    \resizebox{.4\textwidth}{!}{\fbox{\begin{tabular}{ccc}
\textcolor{tab-blue}{\tikz\draw[tab-blue,fill=tab-blue] (0,0) circle (.7ex);}  Budget: \$350 &
\textcolor{tab-orange}{$\large\blacktriangledown$}  Budget: \$850  &
\textcolor{tab-green}{$\blacksquare$} Budget: \$1700 \\
% \multicolumn{3}{c}{A: Annual Screening} \\
% \multicolumn{3}{c}{B: Biannual Screening}\\
% \multicolumn{3}{c}{B + R: Biannual Screening + MRI} \\
% \multicolumn{3}{c}{B + U: Biannual Screening + Ultrasound}
\end{tabular}}} \\[2pt]
    \subfloat[\centering QALY gains vs no screening (39.4 years).\label{fig:ar-QALYs-gain}]{\includegraphics[width=.33\textwidth]{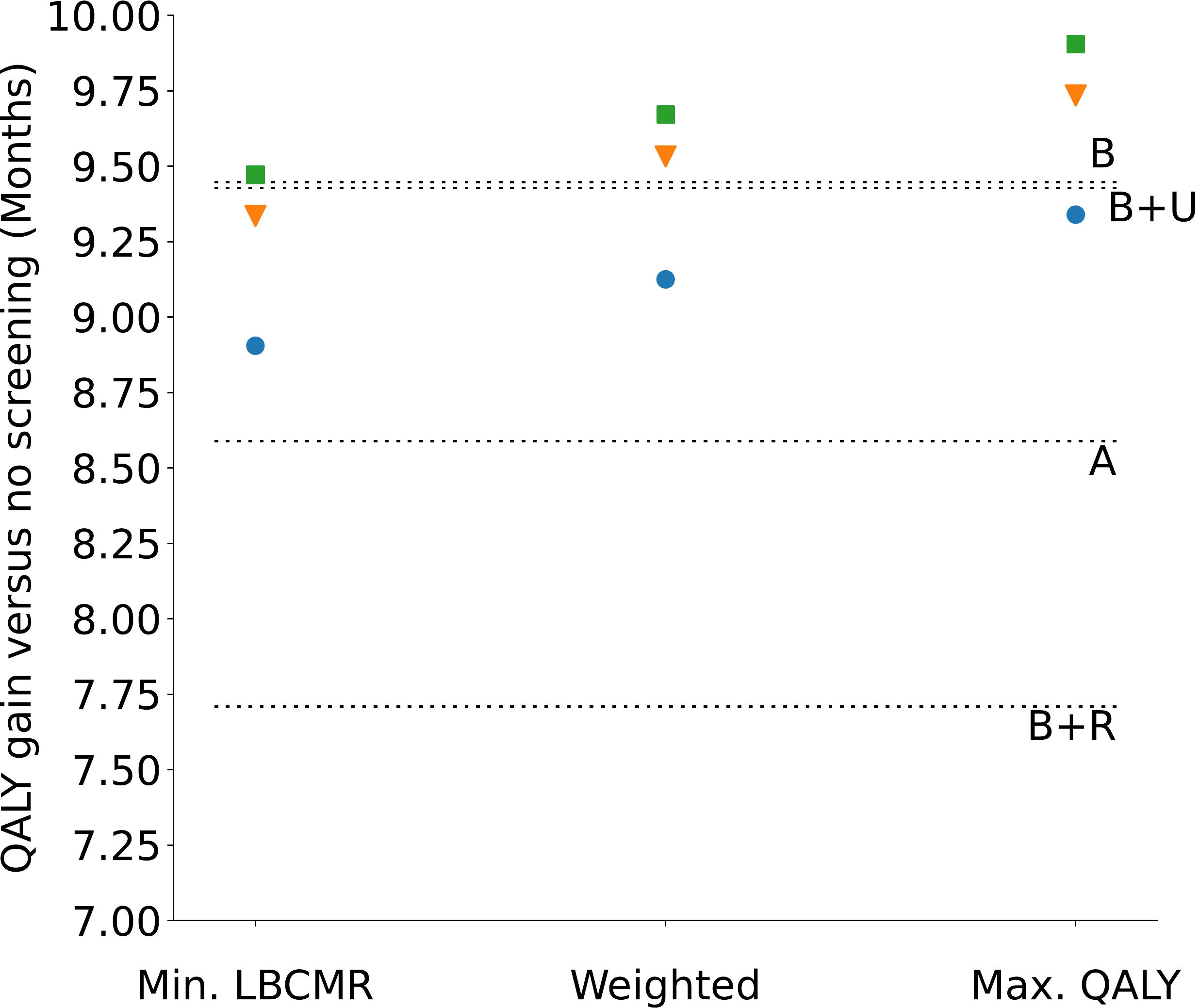}}
    \hfill
    \subfloat[\centering Decrease in LBCMR vs no screening (8.0\%).\label{fig:ar-risk-loss}]{\includegraphics[width=.33\textwidth]{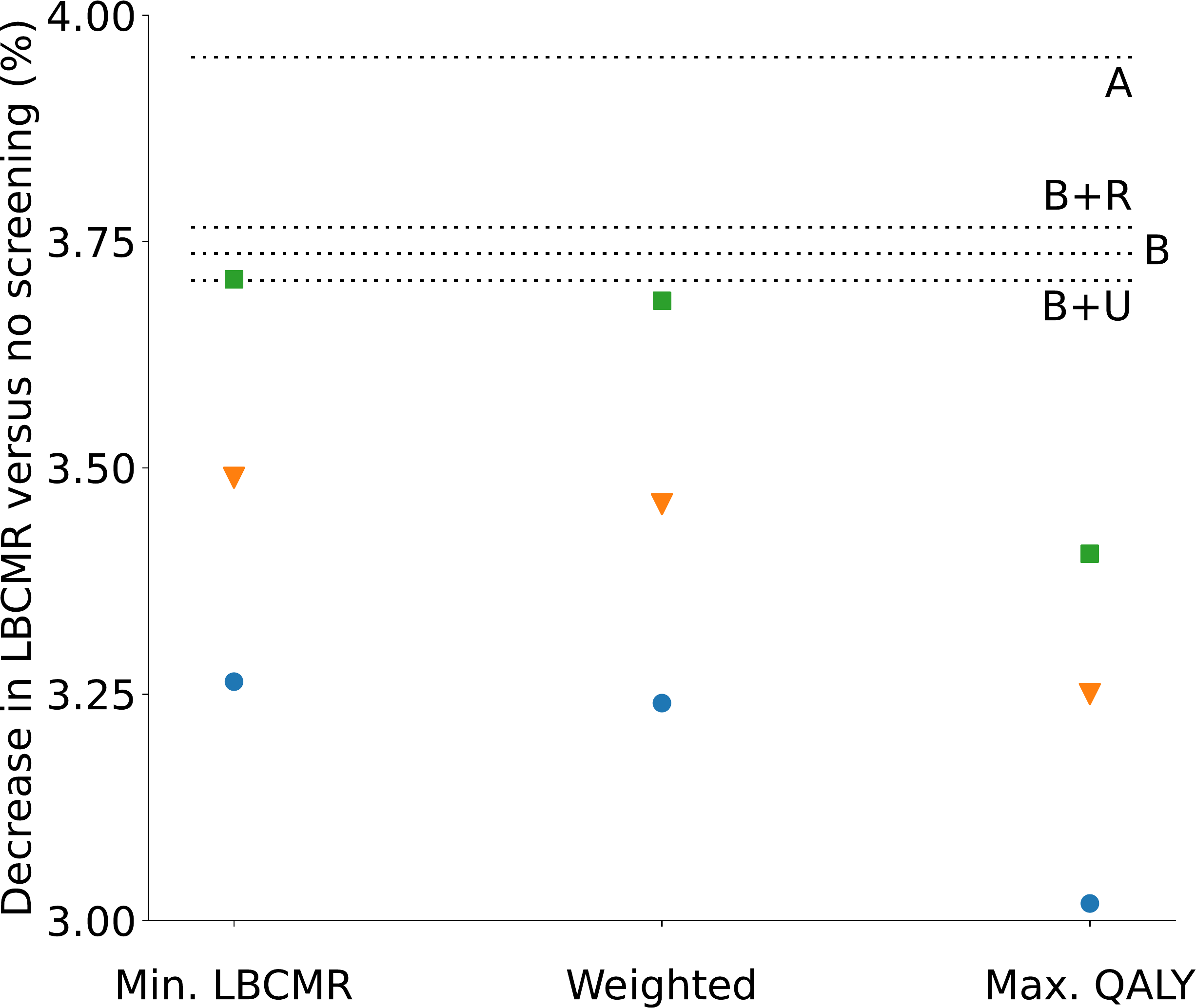}}
    \hfill
    \subfloat[\centering Budget usage ($B+R=\$30,046$).~\label{fig:ar-budget-usage}]{\includegraphics[width=.33\textwidth]{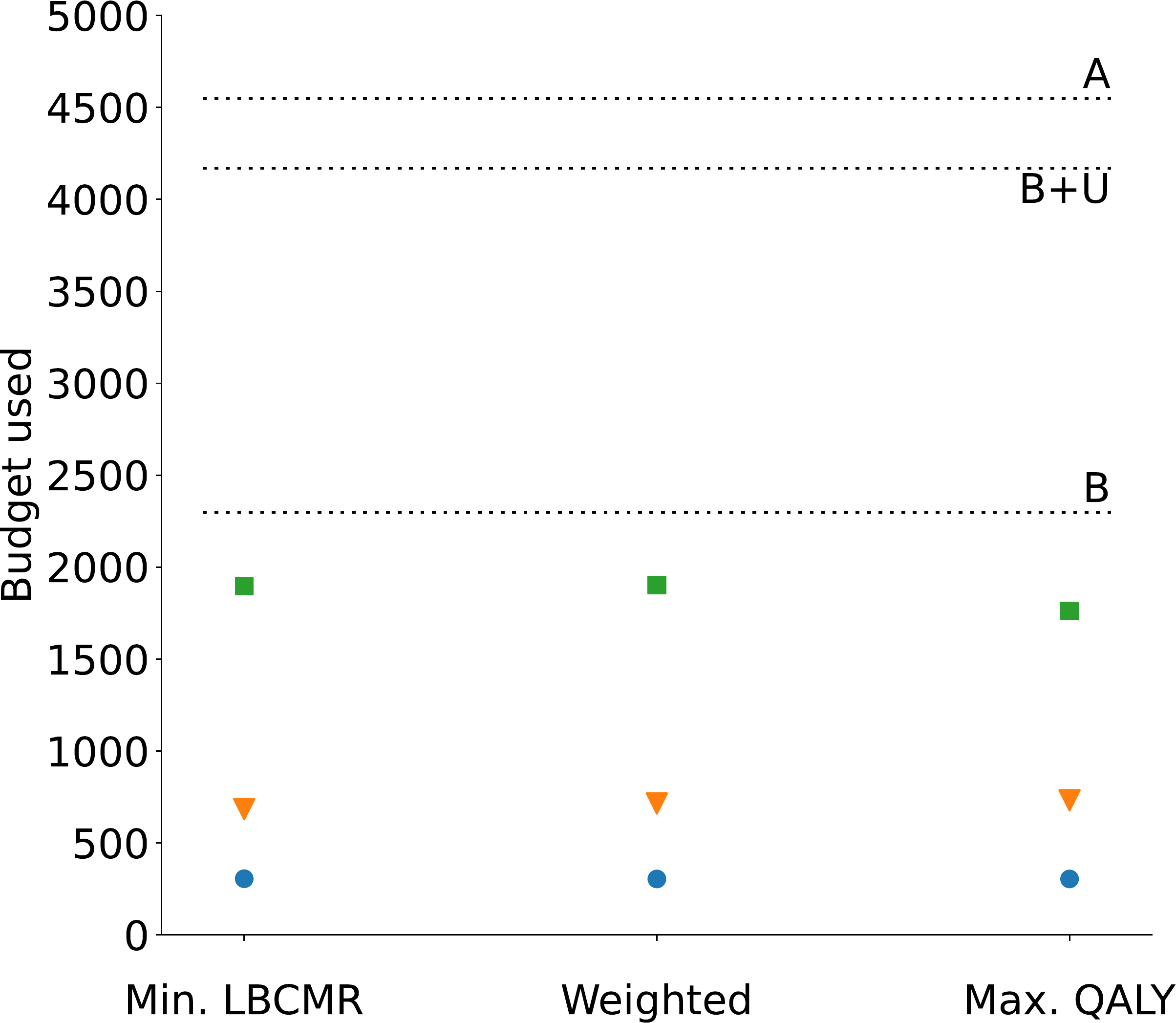}}
    \caption{Comparisons of screening outcomes against fixed-interval policies for HR patients (results for no screening policy are in parenthesis in (a) and (b)).}
    \label{fig:ar-cpomdp-vs-rule}
\end{figure}

Figure~\ref{fig:ar-QALYs-gain} shows the relative QALY gain for each policy compared to no screening for AR patients. 
We observe that a budget of just \$350 can be used to increase the QALYs above 
several rule-based screening strategies, with a budget of just \$850 outperforming every rule-based policy in terms of QALYs. 
This demonstrates that using fewer screenings as prescribed by the CPOMDP model can be deemed sufficient in maximizing the QALY values for the AR patients. 
However, Figure~\ref{fig:ar-risk-loss} shows that the frequent screenings associated with rule-based policies tend to outperform budgeted policies for LBCMR minimization.
This result is expected, as there is no detriment to overscreening when attempting to minimize the risk of dying due to cancer. 
Figure~\ref{fig:ar-budget-usage} shows a significant difference in budgets used between the budgeted policies and the rule-based ones, with annual mammographies using approximately 2.5 times the maximum budget (i.e., \$1,700) with 
only a 0.25\% improvement in LBCMR. 
Figure~\ref{fig:ar-budget-usage} also reveals that CPOMDP policies may exceed their budget. 
Specifically, each of the policies generated using a budget of \$1,700 exceeds this budget in simulations.
Note that exceeding the imposed budget limits in a simulation environment can be attributed to the employed grid-based approximation scheme, as the approximate LP/MIP model only optimizes over the pre-defined grid points.
Similar observations regarding the constraint violation of grid-based approximations for CPOMDPs were also made in other studies (e.g., see \citep{poupart2015}).

% In contrast to the results for the AR patients, 
Figure~\ref{fig:hr-QALYs-gain} shows that the rule-based policies perform similarly to one another with respect to maximizing QALYs in the case of HR patients.
Additionally, the maximum gain in QALYs is more than double that of AR patients.
Unlike in the case of AR patients, a budget of \$350 does not produce any competitive policies, and, as can be seen in both Figures~\ref{fig:hr-QALYs-gain} and~\ref{fig:hr-risk-loss}, the results for this budget level is well separated from the other policies. 
However, a budget of \$350 is still a significant improvement over no screening, and it is the most valuable policy per dollar spent. 
For HR patients, the gap between the best budget-constrained policy and the best rule-based policy with respect to risk minimization is about 0.35\%, which is larger than the gap observed for AR patients. 
Figure~\ref{fig:hr-budget-usage} shows that rule-based policies for HR patients tend to cost slightly less than for AR patients.
This observation can be attributed to the fact that HR patients are more likely to develop cancer and are therefore more likely to leave the decision process due to diagnosis or death.
In contrast, the budgeted policies tend to fully utilize the available budget for both AR and HR patients, as they are able to account for this consideration. 
%difference between AR and HR patients.

\begin{figure}[htb]
    \centering
    \resizebox{.4\textwidth}{!}{\fbox{\begin{tabular}{ccc}
\textcolor{tab-blue}{\tikz\draw[tab-blue,fill=tab-blue] (0,0) circle (.7ex);}  Budget: \$350 &
\textcolor{tab-orange}{$\large\blacktriangledown$}  Budget: \$850  &
\textcolor{tab-green}{$\blacksquare$} Budget: \$1700 \\
% \multicolumn{3}{c}{A: Annual Screening} \\
% \multicolumn{3}{c}{B: Biannual Screening}\\
% \multicolumn{3}{c}{B + R: Biannual Screening + MRI} \\
% \multicolumn{3}{c}{B + U: Biannual Screening + Ultrasound}
\end{tabular}}} \\[0pt]
    \subfloat[\centering QALY gains vs no screening (37.7 years).\label{fig:hr-QALYs-gain}]{\includegraphics[width=.33\textwidth]{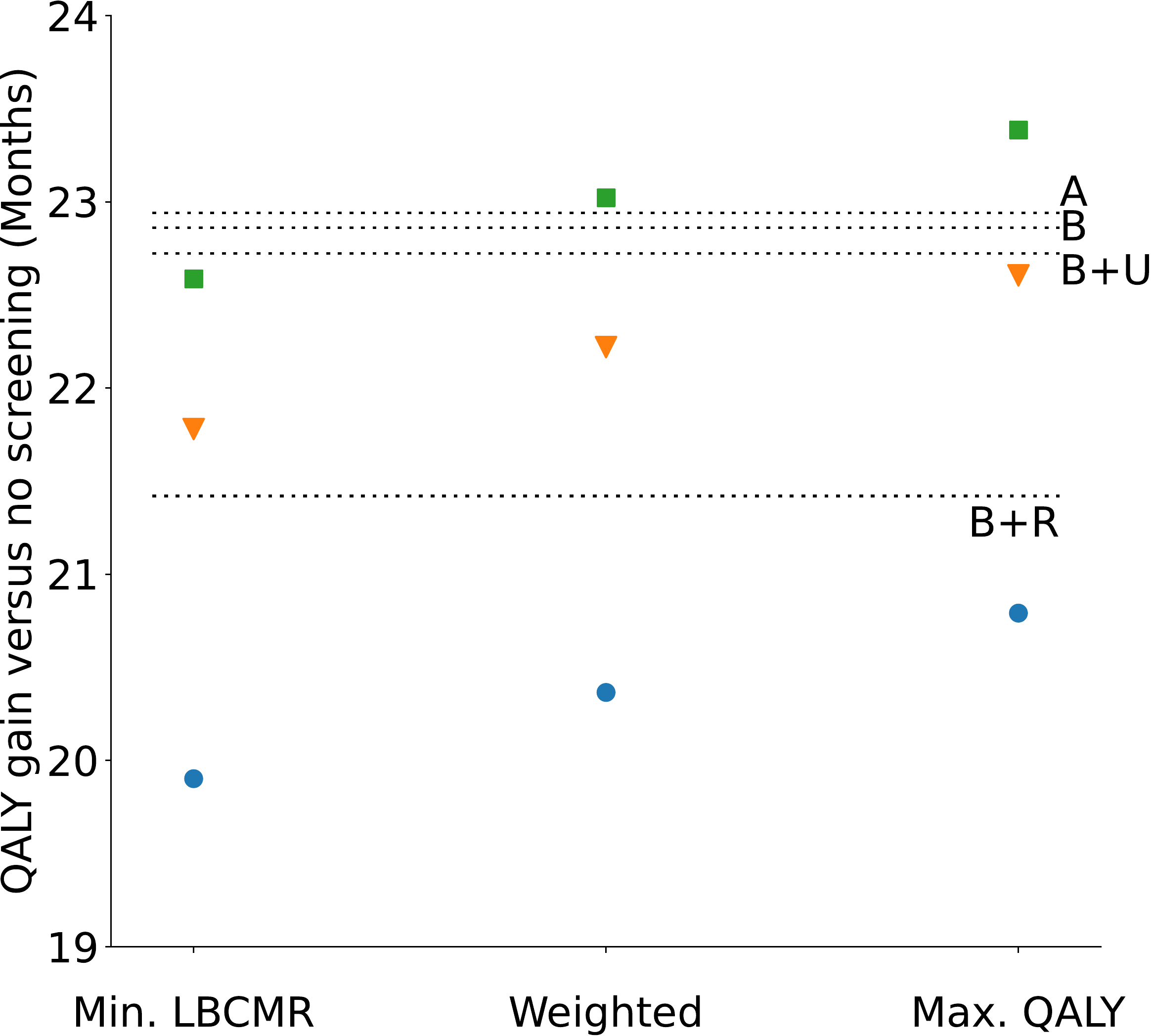}}
    \hfill
    \subfloat[\centering Decrease in risk vs no screening (16.3\%).\label{fig:hr-risk-loss}]{\includegraphics[width=.33\textwidth]{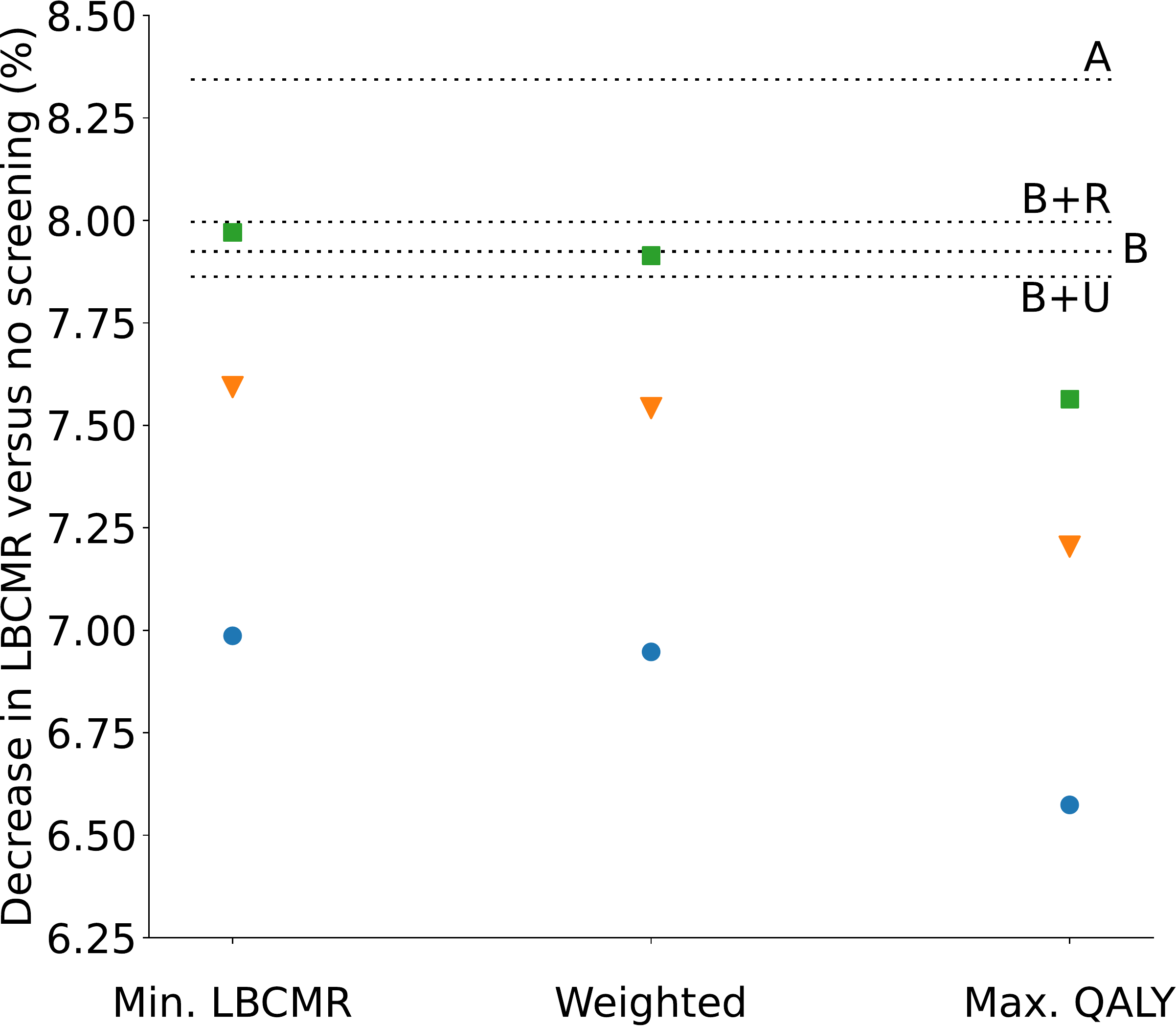}}
    \hfill
    \subfloat[\centering Budget usage ($B+R=\$ 28,295$).\label{fig:hr-budget-usage}]{\includegraphics[width=.33\textwidth]{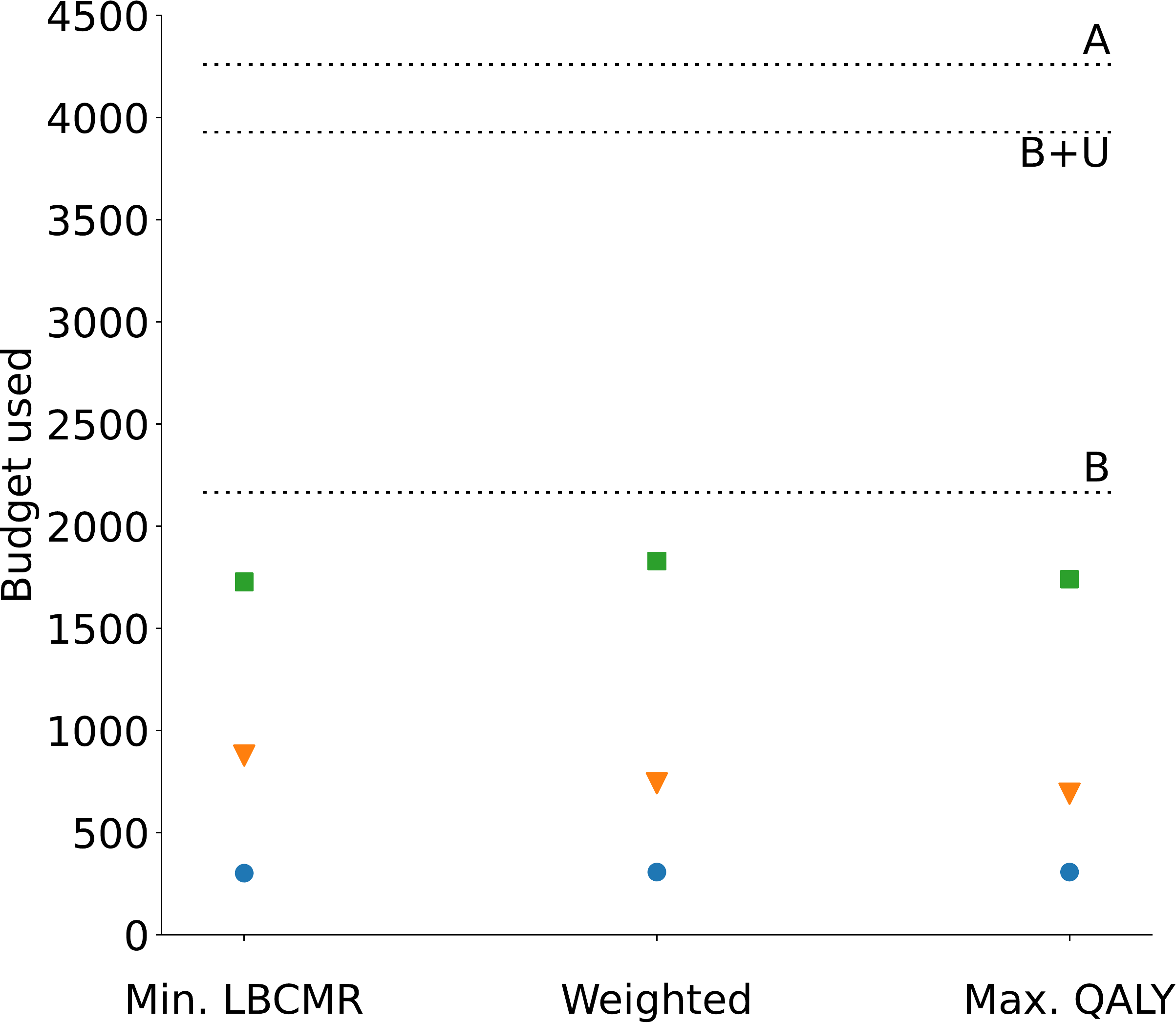}}
    \caption{Comparisons of screening outcomes against fixed-interval policies for HR patients (results for no screening policy are in parenthesis in (a) and (b)).}
        \label{fig:hr-cpomdp-vs-rule}
\end{figure}

Table~\ref{tab:simulated-actions-taken} summarizes the screening actions prescribed by various CPOMDP policies based on a discrete-event simulation.
Specifically, we simulate the lifetime of 100,000 patients using our CPOMDP model data (e.g., starting belief states and transition/observation probabilities) and collect the relevant statistics, namely, QALY, LBCMR and cost estimates, to evaluate the implications of different screening policies. 
Note that ``M/R/U'' column shows, respectively, the percentage of decision epochs with mammography only, mammography + MRI and mammography + ultrasound screening actions.
As expected, with an infinite budget, LBCMR minimization policies employ aggressive screening using the highest sensitivity modality, i.e., mammography supplemented with MRI. 
In contrast, the QALY maximization policies tend to recommend fewer screenings and, even without a budget constraint, they prescribe mammography screenings conservatively. 
Some supplemental tests with ultrasounds are prescribed for these patients, but as ultrasound adds an extra half day of disutility, it is often the case that the decreased likelihood of false positives does not outweigh the additional disutility of adding the supplemental test. 
% A sensitivity analysis is conducted in Section~\ref{sec:sensitivity-analysis} which provides more insight into when supplementing with ultrasounds becomes the dominant screening type.
Policies from the weighted objective CPOMDP model lead to screening more aggressively than the QALY maximization policies, but they are also more likely to recommend mammography supplemented with ultrasound when screening is necessary. 
The budgeted policies all tend to use mammography alone, preserving their budget to allow for additional screenings.
Remarkably, there are several instances like the AR patient with $\parBudgetLim=\$350$ case where, for each objective, the percentage of mammography actions taken is very similar, but the results from Table~\ref{tab:simulated-actions-taken} clearly show substantial differences in QALYs and LBCMR.
These results show that even with access to a limited number of screening actions, QALYs and LBCMR can vary considerably depending on to whom the screenings are prescribed (e.g., at different ages).

\begin{table}[!ht]
\centering

\caption{Simulation results for three optimization objectives, budgets, and two patient types (M/R/U: screening action percentages; LBCMR values are in \%).}
\label{tab:simulated-actions-taken}
\resizebox{\textwidth}{!}{
\begin{tabular}{lrrrrrrrrrrrrr}
\toprule
      &  & \multicolumn{4}{c}{Min. LBCMR} & \multicolumn{4}{c}{Weighted} & \multicolumn{4}{c}{Max. QALYs} \\
      \cmidrule(lr){3-6}\cmidrule(lr){7-10}\cmidrule(lr){11-14}
Patient & $\parBudgetLim$ &      QALYs & LBCMR & Cost &         M/R/U &     QALYs & LBCMR & Cost &          M/R/U &      QALYs & LBCMR & Cost &          M/R/U \\
\midrule
AR & 350 &     40.17 &      4.69 &         306 &   6.7/0.0/0.0 &    40.19 &      4.71 &         304 &    6.6/0.0/0.0 &     40.21 &      4.94 &         304 &    6.6/0.0/0.0 \\
      & 850 &     40.21 &      4.46 &         683 &  14.9/0.0/0.0 &    40.22 &      4.49 &         714 &   15.6/0.0/0.0 &     40.24 &      4.70 &         732 &   16.0/0.0/0.0 \\
      & 1700 &     40.22 &      4.25 &        1896 &  41.6/0.0/0.0 &    40.23 &      4.27 &        1903 &   41.7/0.0/0.0 &     40.25 &      4.55 &        1762 &   28.4/0.0/5.6 \\
      & --- &     39.88 &      3.96 &       54810 &  2.1/92.0/0.0 &    40.22 &      4.08 &        3753 &  57.2/0.0/13.9 &     40.25 &      4.50 &        2235 &  26.8/0.0/12.2 \\
\midrule
HR & 350 &     39.31 &      9.36 &         302 &   7.0/0.0/0.0 &    39.35 &      9.40 &         307 &    7.1/0.0/0.0 &     39.39 &      9.77 &         307 &    7.1/0.0/0.0 \\
      & 850 &     39.47 &      8.75 &         876 &  20.4/0.0/0.0 &    39.51 &      8.80 &         741 &   17.2/0.0/0.0 &     39.54 &      9.14 &         690 &   16.0/0.0/0.0 \\
      & 1700 &     39.54 &      8.37 &        1727 &  40.3/0.0/0.0 &    39.57 &      8.43 &        1829 &   42.7/0.0/0.0 &     39.60 &      8.78 &        1741 &   40.6/0.0/0.0 \\
      & --- &     39.31 &      7.90 &       51770 &  1.5/92.7/0.0 &    39.59 &      7.97 &        6558 &  64.9/3.1/26.4 &     39.62 &      8.52 &        3431 &  55.3/0.0/13.8 \\
\bottomrule
\end{tabular}}
% \resizebox{\textwidth}{!}{\input{tables/simulations/actions-used}}
\end{table}

As with the unconstrained case, Table~\ref{tab:simulated-actions-taken} shows that the weighted objective offers a reasonable trade-off between QALYs and LBCMR.
That is, our multi-objective CPOMDP formulation can be used to identify a policy that balances these two objectives, as the QALY maximization (LBCMR minimization) objective incurs substantial sacrifices to LBCMR (QALYs) for a minimal gain in QALYs (LBCMR).
In contrast, the weighted objective tends to achieve near maximal/minimal values in terms of QALYs/LBCMR.

%%%%%%%%%%%%%%%%%%%%%%%%%%%%%%%%%%%%%%%%%%%%%%%%%%%%%%%%%%%%%%%%%%%%%%%%%%%%%%%%%%
% \subsection{Simulations}

% \subsection{Policy Evaluation Results}\label{sec:policy-evaluation}
\subsection{Sample Screening Recommendation Scenarios}\label{sec:screening-recommendation-scenarios}
%%%%%%%%%%%%%%%%%%%%%%%%%%%%%%%%%%%%%%%%%%%%%%%%%%%%%%%%%%%%%%%%%%%%%%%%%%%%%%%%%%

In the previous experiment, the investigation into POMDP/CPOMDP policies focused on the summary statistics obtained by a discrete-event simulation model with a large number of patients.
In this experiment, we consider how each policy impacts the screening decisions made by a physician throughout a patient's lifetime.
Here, we make two assumptions.
Firstly, we assume that the patient never develops any form of breast cancer, as this could result in them exiting the decision process early.
Secondly, we assume that every screening action yields a negative result, including those taken during a prescribed wait period (i.e., the result of a clinical breast exam/self-examination is always negative). 
We observe the decision process for varying policies and patient types according to these two assumptions.

Table~\ref{tab:ar-patient-avg-budget-sim} compares the constrained policies using an average budget for each objective for an AR patient with a starting belief of $\vectBePt = [0.9954, 0.0016, 0.003]$ (i.e., in-situ cancer risk of 0.16\% and invasive cancer risk of 0.3\%). 
Rows corresponding to a wait action recommendation are omitted from the table. These results reveal that regardless of the objective, the policies are generally the same in this particular scenario. 
There is one noticeable difference, as the QALY maximization objective takes its final screening action at age 64 while the other two objectives take their final screening at age 73. 
This difference in screening ages can be attributed to the fact that as a patient gets older, they are more likely to develop cancer. 
Additionally, the sensitivity of screening tests increases as the patients get older. Accordingly, by saving this screening action until a later age, the risk of cancer death can be reduced. 
In the case of the QALY maximization objective however, a main consideration is the collection of salvage rewards, which represent the lump-sum QALYs associated with a patient being diagnosed at a particular age and leaving the screening process. 
These salvage rewards are lower for older patients. 
Accordingly, the QALY maximization objective attempts to balance the risk of a patient dying (insofar as it affects QALYs) with the reward for diagnosing a patient.
% AR 850 single patient simulation
\begin{table}[htb]

\caption{Sample screening recommendations for an AR patient considering average budget limit ($\parBudgetLim=\$850$; in-situ and invasive cancer risks are in \%).}
\label{tab:ar-patient-avg-budget-sim}
\subfloat[Min. LBCMR~\label{tab:ar-patient-min-risk-sim}]
{\resizebox{0.31\textwidth}{!}{\begin{tabular}{rrrl}
\toprule
 age &  in-situ & invasive & action \\
\midrule
  51 &     0.32 &      0.19 &         Mam. \\
  61 &     0.55 &      0.32 &         Mam. \\
  73 &     0.89 &      0.49 &         Mam. \\
\bottomrule
\end{tabular}
}}
\hfill\subfloat[Weighted\label{tab:ar-patient-weighted-sim}]
{\resizebox{0.31\textwidth}{!}{\begin{tabular}{rrrl}
\toprule
 age &  in-situ &  invasive & action \\
\midrule
  51 &     0.32 &      0.19 &         Mam. \\
  61 &     0.55 &      0.32 &         Mam. \\
  73 &     0.89 &      0.49 &         Mam. \\
\bottomrule
\end{tabular}
}}
\hfill
\subfloat[Max. QALY~\label{tab:ar-patient-max-QALYs-sim}]
{\resizebox{0.31\textwidth}{!}{\begin{tabular}{rrrl}
\toprule
 age &  in-situ &  invasive & action \\
\midrule
  51 &     0.32 &      0.19 &         Mam. \\
  60 &     0.52 &      0.31 &         Mam. \\
  64 &     0.61 &      0.33 &         Mam. \\
\bottomrule
\end{tabular}
}}
\end{table}

Unlike for AR patients, the screening recommendations for HR patients vary considerably across different objectives.
HR patients are more likely to develop breast cancer and, accordingly, they tend to begin the decision process at a higher likelihood of having breast cancer.
For this particular scenario, we consider a starting belief of $\vectBePt =[0.9755,0.0085,0.0160]$. 
Accordingly, each policy prescribes a screening at age 40, the earliest available time to do so in the decision process, as shown in Table~\ref{tab:hr-patient-avg-budget-sim}. 
However, the next screening is not recommended until age 57 for the LBCMR minimization policy, and age 64 for the weighted and QALY maximization policies.
% This can be due to the fact that older patients are more likely to develop and die from breast cancer. 
Tables~\ref{tab:hr-patient-min-risk-sim} and~\ref{tab:hr-patient-max-QALYs-sim} reveal that the number of screenings taken for a patient can vary considerably across different objectives.
In the case of the AR patient, each objective yielded just three mammography actions, whereas, for HR patients with the LBCMR minimization policy, the number of mammographies is double that of the QALY maximization policy.
In the case of QALYs this is intuitive: recall that these simulations assume a negative screening result, even during a wait action period, meaning that this particular patient is able to stay in a healthier belief state during their lifetime compared to other patients (i.e., scenarios). 
Accordingly, the QALY maximization objective will allocate an additional budget for patients that are in worse belief/health states due to positive observations (i.e., $\indObs = \indObsPositive$).
Notably, the minimize LBCMR policy evaluation comes in under budget, costing just \$804 of the \$850 budget limit. In contrast, results in Table~\ref{tab:simulated-actions-taken} shows that the average cost for patients subjected to this policy is \$876. This result is to be expected, as the assumptions made in this policy evaluation tend to keep the patient in a healthier belief state. In particular, positive results during wait actions will tend to make the patient appear more likely to have cancer, but as the assumptions made in this evaluation only allow negative results during wait actions, this patient tends to stay in a healthier belief state. Accordingly, the policy has allocated additional budget to those patients that are more likely to have breast cancer.
Furthermore, observing a high number of screenings in this all-negative observations scenario indicates that LBCMR minimization is affected at a smaller scale by the perceived belief/health states of the patients, and screenings are allocated more uniformly across different scenarios that a patient might follow.

% HR 850 single patient simulation
\begin{table}[htb]

\caption{Sample screening recommendations for an HR patient considering average budget limit ($\parBudgetLim=\$850$; in-situ and invasive cancer risks are in \%).}
\label{tab:hr-patient-avg-budget-sim}
\subfloat[Min. LBCMR~\label{tab:hr-patient-min-risk-sim}]
{\resizebox{0.31\textwidth}{!}{
\begin{tabular}{rrrl}
\toprule
 age &  in-situ &  invasive & action \\
\midrule
  40 &     0.85 &      1.60 &         Mam. \\
  57 &     0.88 &      0.52 &         Mam. \\
  64 &     1.26 &      0.72 &         Mam. \\
  70 &     1.63 &      0.89 &         Mam. \\
  72 &     1.41 &      0.69 &         Mam. \\
  78 &     1.85 &      0.96 &         Mam. \\
\bottomrule
\end{tabular}
}}
\hfill
\subfloat[Weighted~\label{tab:hr-patient-weighted-sim}]
{\resizebox{0.31\textwidth}{!}{
\begin{tabular}{rrrl}
\toprule
 age &  in-situ &  invasive & action \\
\midrule
  40 &     0.85 &      1.60 &         Mam. \\
  64 &     1.27 &      0.73 &         Mam. \\
  70 &     1.63 &      0.89 &         Mam. \\
  72 &     1.41 &      0.69 &         Mam. \\
\bottomrule
\end{tabular}
}}
\hfill
\subfloat[Max. QALY~\label{tab:hr-patient-max-QALYs-sim}]
{\resizebox{0.31\textwidth}{!}{
\begin{tabular}{rrrl}
\toprule
 age &  in-situ &  invasive & action \\
\midrule
  40 &     0.85 &      1.60 &         Mam. \\
  64 &     1.27 &      0.73 &         Mam. \\
  72 &     1.75 &      0.96 &         Mam. \\
\bottomrule
\end{tabular}
}}
\end{table}

Sample screening recommendation scenarios for an HR patient with no budget restrictions are also considered for the three objective functions to demonstrate the impact of budget limits on the screening recommendations. 
Table~\ref{tab:hr-unconstrained-max-QALYs-sim} shows that, in this scenario, screening recommendations generally favour supplemental ultrasound for earlier ages, but once the patient reaches 50 years old, the recommended screenings turn to mammography alone. 
Mammography supplemented with ultrasound is assumed to have a higher specificity but a lower sensitivity than mammography alone, meaning that while the ultrasound as a supplemental test reduces the risk of false positives, it also leads to fewer true positives.
At earlier ages, when patients are less likely to develop breast cancer, this reduced chance of false positives is substantial enough to justify the supplemental test, even with the additional disutility associated with the test.
However, as the patient ages, the reduced risk of false positives is outweighed by the additional disutility of the supplemental test, as well as the increased chance of having cancer go unnoticed. 
% Section~\ref{sec:sensitivity-analysis} explores the impact of various disutility combinations on the screening actions prescribed, and offers insight into how the disutility of a supplemental ultrasound impacts the screenings in Table~\ref{tab:hr-unconstrained-max-QALYs-sim}.
% HR unconstrained single patient simulation
\begin{table}[htb]

\caption{Sample screening recommendations for an HR patient considering unlimited budget (in-situ and invasive cancer risks are in \%)}
\label{tab:hr-patient-unconstrained-sim}
\subfloat[Min. LBCMR~\label{tab:hr-unconstrained-min-risk-sim}]
{\resizebox{0.31\textwidth}{!}{
\begin{tabular}{rrrl}
\toprule
 age &  in-situ &  invasive & action \\
\midrule
  40 &     0.85 &      1.60 &   Mam. + MRI \\
  41 &     0.31 &      0.25 &   Mam. + MRI \\
  42 &     0.23 &      0.12 &   Mam. + MRI \\
  43 &     0.23 &      0.12 &   Mam. + MRI \\
  44 &     0.26 &      0.12 &   Mam. + MRI \\
  45 &     0.28 &      0.14 &   Mam. + MRI \\
  46 &     0.30 &      0.14 &   Mam. + MRI \\
  47 &     0.32 &      0.15 &   Mam. + MRI \\
  48 &     0.33 &      0.16 &   Mam. + MRI \\
  49 &     0.35 &      0.17 &   Mam. + MRI \\
  50 &     0.36 &      0.17 &   Mam. + MRI \\
  51 &     0.37 &      0.18 &   Mam. + MRI \\
  52 &     0.38 &      0.18 &   Mam. + MRI \\
  53 &     0.40 &      0.19 &   Mam. + MRI \\
  54 &     0.42 &      0.20 &   Mam. + MRI \\
  55 &     0.46 &      0.21 &   Mam. + MRI \\
  56 &     0.48 &      0.23 &   Mam. + MRI \\
  57 &     0.52 &      0.25 &   Mam. + MRI \\
  58 &     0.55 &      0.26 &   Mam. + MRI \\
  59 &     0.58 &      0.27 &   Mam. + MRI \\
  60 &     0.61 &      0.29 &   Mam. + MRI \\
  61 &     0.64 &      0.30 &   Mam. + MRI \\
  62 &     0.67 &      0.31 &   Mam. + MRI \\
  63 &     0.69 &      0.32 &   Mam. + MRI \\
  64 &     0.73 &      0.33 &   Mam. + MRI \\
  65 &     0.76 &      0.35 &   Mam. + MRI \\
  66 &     0.80 &      0.36 &   Mam. + MRI \\
  67 &     0.82 &      0.37 &   Mam. + MRI \\
  70 &     1.46 &      0.75 &   Mam. + MRI \\
  71 &     1.05 &      0.46 &   Mam. + MRI \\
  72 &     1.00 &      0.44 &   Mam. + MRI \\
  73 &     1.00 &      0.44 &   Mam. + MRI \\
  74 &     1.01 &      0.44 &   Mam. + MRI \\
  75 &     1.02 &      0.44 &   Mam. + MRI \\
  76 &     1.03 &      0.44 &   Mam. + MRI \\
  77 &     1.03 &      0.44 &   Mam. + MRI \\
  79 &     1.44 &      0.66 &         Mam. \\
\bottomrule
\end{tabular}
}}
\hfill\subfloat[Weighted~\label{tab:hr-unconstrained-weighted-sim}]
{\resizebox{0.31\textwidth}{!}{
\begin{tabular}{rrrl}
\toprule
 age &  in-situ &  invasive & action \\
\midrule
  40 &     0.85 &      1.60 &   Mam. + MRI \\
  41 &     0.31 &      0.25 &    Mam. + US \\
  42 &     0.26 &      0.15 &    Mam. + US \\
  43 &     0.26 &      0.14 &    Mam. + US \\
  44 &     0.28 &      0.14 &    Mam. + US \\
  45 &     0.30 &      0.15 &    Mam. + US \\
  46 &     0.33 &      0.16 &    Mam. + US \\
  47 &     0.35 &      0.17 &    Mam. + US \\
  48 &     0.36 &      0.18 &    Mam. + US \\
  49 &     0.37 &      0.18 &    Mam. + US \\
  50 &     0.38 &      0.19 &         Mam. \\
  51 &     0.38 &      0.19 &         Mam. \\
  52 &     0.40 &      0.19 &         Mam. \\
  53 &     0.41 &      0.20 &         Mam. \\
  54 &     0.44 &      0.21 &         Mam. \\
  55 &     0.47 &      0.22 &         Mam. \\
  56 &     0.50 &      0.24 &         Mam. \\
  57 &     0.53 &      0.26 &         Mam. \\
  58 &     0.56 &      0.27 &         Mam. \\
  59 &     0.60 &      0.28 &         Mam. \\
  60 &     0.62 &      0.30 &         Mam. \\
  61 &     0.65 &      0.30 &         Mam. \\
  62 &     0.68 &      0.32 &         Mam. \\
  63 &     0.70 &      0.33 &         Mam. \\
  64 &     0.74 &      0.34 &         Mam. \\
  65 &     0.77 &      0.36 &         Mam. \\
  66 &     0.80 &      0.37 &         Mam. \\
  67 &     0.83 &      0.38 &         Mam. \\
  70 &     1.46 &      0.75 &         Mam. \\
  71 &     1.06 &      0.46 &         Mam. \\
  72 &     1.00 &      0.44 &         Mam. \\
  73 &     1.00 &      0.44 &         Mam. \\
  74 &     1.00 &      0.44 &         Mam. \\
  75 &     1.02 &      0.44 &         Mam. \\
  76 &     1.03 &      0.44 &         Mam. \\
  77 &     1.02 &      0.44 &         Mam. \\
  78 &     1.03 &      0.44 &         Mam. \\
  79 &     1.04 &      0.43 &         Mam. \\
  % 80 &     1.03 &      0.43 &         Mam. \\
\bottomrule
\end{tabular}
}}
\hfill
\subfloat[Max. QALY~\label{tab:hr-unconstrained-max-QALYs-sim}]
{\resizebox{0.31\textwidth}{!}{
\begin{tabular}{rrrl}
\toprule
 age &  in-situ &  invasive & action \\
\midrule
  40 &     0.85 &      1.60 &         Mam. \\
  41 &     0.37 &      0.40 &    Mam. + US \\
  43 &     0.34 &      0.21 &    Mam. + US \\
  45 &     0.39 &      0.22 &    Mam. + US \\
  47 &     0.46 &      0.25 &    Mam. + US \\
  48 &     0.38 &      0.19 &    Mam. + US \\
  50 &     0.50 &      0.27 &         Mam. \\
  52 &     0.54 &      0.29 &         Mam. \\
  54 &     0.59 &      0.31 &         Mam. \\
  55 &     0.50 &      0.24 &         Mam. \\
  57 &     0.71 &      0.38 &         Mam. \\
  59 &     0.81 &      0.43 &         Mam. \\
  60 &     0.66 &      0.31 &         Mam. \\
  61 &     0.66 &      0.31 &         Mam. \\
  62 &     0.68 &      0.32 &         Mam. \\
  63 &     0.70 &      0.33 &         Mam. \\
  64 &     0.74 &      0.34 &         Mam. \\
  65 &     0.77 &      0.36 &         Mam. \\
  66 &     0.80 &      0.37 &         Mam. \\
  67 &     0.83 &      0.38 &         Mam. \\
  68 &     0.87 &      0.40 &         Mam. \\
  69 &     0.91 &      0.41 &         Mam. \\
  70 &     0.95 &      0.42 &         Mam. \\
  71 &     0.98 &      0.43 &         Mam. \\
  73 &     1.37 &      0.67 &         Mam. \\
  74 &     1.06 &      0.46 &         Mam. \\
  75 &     1.02 &      0.44 &         Mam. \\
  77 &     1.42 &      0.67 &         Mam. \\
  79 &     1.47 &      0.67 &         Mam. \\
\bottomrule
\end{tabular}
}}
\end{table}

%%%%%%%%%%%%%%%%%%%%%%%%%%%%%%%%%%%%%%%%%%%%%%%%%%%%%%%%%%%%%%%%%%%%%%%%%%%%%%%%%%
\subsection{Sensitivity Analysis}\label{sec:sensitivity-analysis}
%%%%%%%%%%%%%%%%%%%%%%%%%%%%%%%%%%%%%%%%%%%%%%%%%%%%%%%%%%%%%%%%%%%%%%%%%%%%%%%%%%

We conclude our numerical study with a sensitivity analysis for important model parameters. 
We first consider the impact of differing disutility values on the unconstrained model. 
These results are quantified using QALYs and LBCMR, as well as the percentage of decision epochs where a screening action was taken under ``M/R/U'' columns.
The various disutilities impact QALYs only and have no effect on the LBCMR.
Accordingly, the LBCMR minimization objective is excluded from the sensitivity analysis, as all combinations of disutilities will produce the same LBCMR with this objective function. 
Preliminary experiments in the sensitivity analysis on disutilities focus only on those cases where TP and FP are the same, except for the baseline disutility values obtained from the literature. 
This is because experiments showed that varying TP while keeping FP constant had a negligible impact on QALYs, LBCMR, and screening composition. 
Accordingly, the TP and FP disutilities are amalgamated into a single term, PT, representing the disutility of a positive test result.

Table~\ref{tab:ar-disutility-sens} shows that when there is no disutility associated with any of the screenings nor for TP or FP, MRI is employed extensively. 
This is intuitive, as what prevents the general use of MRI is the disutility associated with the screening test (2.5 days) as well as the risk of false positives. 
The addition of PT disutilities has a substantial impact on the screenings recommended. 
When the disutility of mammography and ultrasound are the same, the introduction of PT disutilities --- even as little as 7 days --- causes the QALY maximization policies to prefer ultrasound for nearly every screening, as this supplemental screening action reduces the likelihood of false positives. 
However, when ultrasound has a higher disutility than mammography, the model prefers to use the mammography action, even when false positives incur a 28-day disutility. 
Similarly, a supplemental MRI is often recommended when MRI has the same disutility as mammography and the disutility of a positive test is zero, but when positive tests are assigned a disutility value, the MRI action is no longer favourable for maximizing QALYs as it raises the risk of false positives.

% AR disutility sensitivity analysis
\begin{table}[htb]
\centering
% \caption{AR unconstrained disutility sensitivity analysis.}
\caption{Sensitivity analysis results for disutility values considering AR patients with unlimited budgets (M/R/U: screening action percentages; LBCMR values are in \%).}
\label{tab:ar-disutility-sens}
\resizebox{\textwidth}{!}{
\begin{tabular}{rrrrrrrrrr}
\toprule
    &     &     &    & \multicolumn{3}{c}{Weighted} & \multicolumn{3}{c}{Max. QALYs} \\
    \cmidrule(lr){5-7} \cmidrule(lr){8-10}
M & M+R & M+U & PT  &     QALYs & LBCMR &          M/R/U &      QALYs & LBCMR & M/R/U \\
\midrule
\multicolumn{4}{c}{Baseline} & 40.22 &      4.08 &  57.2/0.0/13.9 &     40.25 &      4.50 &  26.8/0.0/12.2 \\
\midrule
0 & 0 & 0 & 0  &    40.46 &      3.98 &  16.7/72.6/8.5 &     40.48 &      4.17 &  2.7/73.9/18.9 \\
    &     &     & 7  &    40.40 &      4.00 &   9.3/2.9/85.6 &     40.42 &      4.18 &  0.0/0.0/100.0 \\
    &     &     & 14 &    40.35 &      4.00 &   0.1/0.0/97.6 &     40.36 &      4.20 &   0.0/0.0/84.9 \\
    &     &     & 28 &    40.26 &      4.01 &   0.3/0.2/79.9 &     40.29 &      4.38 &   0.1/0.0/60.1 \\
    &     & 0.5 & 0  &    40.46 &      3.99 &   5.3/73.8/0.0 &     40.48 &      4.17 &   5.0/71.5/0.0 \\
    &     &     & 7  &    40.39 &      3.98 &   94.3/3.0/0.0 &     40.41 &      4.17 &   93.5/2.9/0.0 \\
    &     &     & 14 &    40.33 &      4.00 &   76.6/4.1/0.0 &     40.34 &      4.23 &   78.4/0.4/0.0 \\
    &     &     & 28 &    40.23 &      4.04 &  59.2/0.2/14.1 &     40.27 &      4.42 &  33.5/0.0/14.7 \\
    & 2 & 0 & 0  &    40.45 &      3.98 &  70.9/0.0/15.0 &     40.47 &      4.19 &  76.4/0.0/13.0 \\
    &     &     & 7  &    40.40 &      4.00 &   9.3/0.0/88.5 &     40.42 &      4.19 &  0.0/0.0/100.0 \\
    &     &     & 14 &    40.35 &      4.00 &   0.1/0.0/97.6 &     40.36 &      4.20 &   0.1/0.0/83.4 \\
    &     &     & 28 &    40.26 &      4.01 &   0.4/0.0/80.2 &     40.29 &      4.37 &   0.3/0.0/58.7 \\
    &     & 0.5 & 0  &    40.45 &      3.98 &   78.9/0.0/0.0 &     40.47 &      4.19 &   78.1/0.0/0.0 \\
    &     &     & 7  &    40.39 &      4.00 &   97.8/0.0/0.0 &     40.40 &      4.17 &   96.2/0.0/0.0 \\
    &     &     & 14 &    40.33 &      3.99 &   85.5/0.0/0.0 &     40.34 &      4.23 &   79.0/0.0/0.0 \\
    &     &     & 28 &    40.23 &      4.04 &  61.5/0.0/14.3 &     40.27 &      4.42 &  32.8/0.0/14.9 \\
0.5 & 0.5 & 0.5 & 0  &    40.41 &      3.96 &   2.4/87.9/0.0 &     40.43 &      4.15 &  0.0/100.0/0.0 \\
    &     &     & 7  &    40.35 &      3.99 &  16.7/2.9/78.1 &     40.37 &      4.21 &   0.4/3.7/79.0 \\
    &     &     & 14 &    40.31 &      4.01 &   0.5/1.1/79.3 &     40.32 &      4.28 &   0.7/0.0/74.5 \\
    &     &     & 28 &    40.23 &      4.04 &   0.5/0.2/76.7 &     40.26 &      4.44 &   0.4/0.0/50.8 \\
    &     & 1.0 & 0  &    40.41 &      3.96 &   0.0/97.8/0.0 &     40.43 &      4.16 &   2.3/97.7/0.0 \\
    &     &     & 7  &    40.34 &      3.98 &   90.9/6.0/0.0 &     40.36 &      4.20 &   73.9/4.1/0.0 \\
    &     &     & 14 &    40.29 &      4.00 &   73.6/4.1/0.0 &     40.31 &      4.34 &   56.2/3.5/0.0 \\
    &     &     & 28 &    40.21 &      4.08 &  56.6/0.4/13.9 &     40.25 &      4.49 &  28.5/0.4/11.7 \\
    & 2.5 & 0.5 & 0  &    40.40 &      3.98 &   97.8/0.0/0.0 &     40.42 &      4.17 &  100.0/0.0/0.0 \\
    &     &     & 7  &    40.35 &      3.99 &  13.8/0.0/84.0 &     40.37 &      4.20 &   0.8/0.0/82.5 \\
    &     &     & 14 &    40.31 &      4.01 &   1.5/0.0/79.6 &     40.32 &      4.29 &   0.7/0.0/72.9 \\
    &     &     & 28 &    40.23 &      4.04 &   0.5/0.0/76.9 &     40.26 &      4.44 &   0.6/0.0/49.7 \\
    &     & 1.0 & 0  &    40.40 &      3.98 &   97.8/0.0/0.0 &     40.42 &      4.17 &  100.0/0.0/0.0 \\
    &     &     & 7  &    40.34 &      3.99 &   97.3/0.0/0.0 &     40.36 &      4.21 &   81.8/0.0/0.0 \\
    &     &     & 14 &    40.29 &      4.00 &   78.7/0.0/0.0 &     40.31 &      4.34 &   62.2/0.0/0.0 \\
    &     &     & 28 &    40.21 &      4.08 &  57.1/0.0/13.9 &     40.25 &      4.49 &  27.6/0.0/11.8 \\
\bottomrule
\end{tabular}}
\end{table}

Table~\ref{tab:hr-disutility-sens} explores the impact of varying disutilities on HR patients.
The results show that the composition of screenings follows a similar pattern to that of AR patients: as false positive disutilities are introduced, supplemental ultrasounds begin to be prescribed until the disutility of a supplemental ultrasound differs from that of mammography, in which case mammography alone becomes the favoured test.
Supplemental MRIs are never used when their disutility differs from that of mammography, even when false positive disutilities are zero. 
These results suggest that, while supplemental MRI offers increased sensitivity and supplemental ultrasound offers increased specificity, the gain in QALYs from these advantages is outweighed by the disutility of undergoing a supplemental test.

In general, we observe that, when the disutility of a supplemental screening test is greater than the disutility of mammography alone, the supplemental test is not used. 
One exception to this rule is supplemental ultrasound when the disutility of a positive test is 28 days. 
It turns out that the reduced risk of false positives outweighs the additional disutility of the supplemental ultrasound to some extent, although in this case, the majority of screening actions are still mammography alone. 
For the AR patients, in the case of weighted objective, the additional disutility of the MRI action prevents MRIs from being prescribed when its disutility is different from mammography. 
In contrast, for the HR weighted objective, the MRI actions are used about 3.1\% of the time when the disutility of mammography + MRI is 2 days more than the disutility of mammography alone.
It is also observed from these sensitivity analysis results that having no additional disutility for these supplemental tests would favour a higher usage of the supplemental screenings.
% It should be noted that the main experiments in this research assume that a supplemental screening is done independently of the mammography.
% Accordingly, the disutilities are summed together.
% If the disutilities could be combined by scheduling both tests for the same appointment, the added disutility could be reduced or removed entirely.
% The results of such a scenario would be favourable for the usage of supplemental screening actions.

% HR disutility sensitivity analysis
\begin{table}[htb]

\centering
% \caption{HR unconstrained disutility sensitivity analysis.}
\caption{Sensitivity analysis results for disutility values considering HR patients with unlimited budgets (M/R/U: screening action percentages; LBCMR values are in \%).}
\label{tab:hr-disutility-sens}
\resizebox{\textwidth}{!}{\begin{tabular}{rrrrrrrrrr}
\toprule
    &     &     &    & \multicolumn{3}{c}{Weighted} & \multicolumn{3}{c}{Max. QALYs} \\
    \cmidrule(lr){5-7} \cmidrule(lr){8-10}
M & M+R & M+U & PT &     QALYs & LBCMR &           M/R/U &      QALYs & LBCMR &          M/R/U \\
\midrule
 \multicolumn{4}{c}{Baseline} & 39.59 &      7.97 &  64.9/3.1/26.4 &     39.62 &      8.52 &  55.3/0.0/13.8 \\
 \midrule
0 & 0 & 0 & 0  &    39.89 &      7.97 &    6.4/86.7/4.9 &     39.92 &      8.26 &  0.8/86.0/11.2 \\
    &     &     & 7  &    39.81 &      7.96 &  44.7/19.1/34.1 &     39.85 &      8.31 &   9.2/5.7/85.0 \\
    &     &     & 14 &    39.76 &      7.98 &    4.1/3.1/90.5 &     39.79 &      8.29 &   0.0/3.1/95.9 \\
    &     &     & 28 &    39.66 &      7.99 &    0.2/3.1/93.6 &     39.69 &      8.30 &   0.0/3.1/92.6 \\
    &     & 0.5 & 0  &    39.89 &      7.97 &    0.6/88.2/0.0 &     39.92 &      8.26 &   1.1/89.0/0.0 \\
    &     &     & 7  &    39.81 &      7.94 &   80.7/17.2/0.0 &     39.84 &      8.27 &   94.3/5.7/0.0 \\
    &     &     & 14 &    39.74 &      7.95 &    90.8/5.8/0.0 &     39.77 &      8.25 &   95.7/3.2/0.0 \\
    &     &     & 28 &    39.62 &      7.96 &   67.6/3.2/26.1 &     39.65 &      8.36 &  52.2/3.3/29.8 \\
    & 2 & 0 & 0  &    39.87 &      7.97 &    84.1/3.1/5.9 &     39.90 &      8.30 &  85.2/0.0/11.1 \\
    &     &     & 7  &    39.81 &      7.97 &   57.1/3.1/37.6 &     39.84 &      8.32 &   7.4/0.0/92.6 \\
    &     &     & 14 &    39.76 &      7.98 &    4.2/3.1/90.5 &     39.79 &      8.31 &   0.0/0.0/99.3 \\
    &     &     & 28 &    39.66 &      7.99 &    0.2/3.1/93.6 &     39.68 &      8.31 &   0.0/0.0/96.2 \\
    &     & 0.5 & 0  &    39.87 &      7.97 &    81.6/3.1/0.0 &     39.90 &      8.30 &   86.1/0.0/0.0 \\
    &     &     & 7  &    39.80 &      7.95 &    94.7/3.1/0.0 &     39.84 &      8.28 &  100.0/0.0/0.0 \\
    &     &     & 14 &    39.74 &      7.95 &    93.6/3.1/0.0 &     39.77 &      8.28 &   96.1/0.0/0.0 \\
    &     &     & 28 &    39.62 &      7.95 &   67.7/3.1/26.1 &     39.65 &      8.38 &  58.5/0.0/27.1 \\
0.5 & 0.5 & 0.5 & 0  &    39.84 &      7.95 &    0.0/97.8/0.0 &     39.87 &      8.26 &   4.9/95.1/0.0 \\
    &     &     & 7  &    39.77 &      7.97 &  30.6/21.8/25.6 &     39.80 &      8.29 &   3.6/5.8/89.5 \\
    &     &     & 14 &    39.72 &      7.99 &    2.4/3.2/91.4 &     39.74 &      8.28 &   0.0/3.1/95.9 \\
    &     &     & 28 &    39.62 &      7.99 &    0.2/3.2/93.6 &     39.64 &      8.37 &   0.1/3.1/86.9 \\
    &     & 1.0 & 0  &    39.84 &      7.95 &    0.0/97.8/0.0 &     39.87 &      8.25 &  0.0/100.0/0.0 \\
    &     &     & 7  &    39.76 &      7.94 &   86.4/11.4/0.0 &     39.79 &      8.26 &   93.0/5.8/0.0 \\
    &     &     & 14 &    39.70 &      7.96 &    90.8/5.8/0.0 &     39.73 &      8.26 &   92.9/3.2/0.0 \\
    &     &     & 28 &    39.58 &      7.97 &   64.6/3.2/27.2 &     39.62 &      8.51 &  48.7/3.6/13.7 \\
    & 2.5 & 0.5 & 0  &    39.82 &      7.94 &    94.7/3.1/0.0 &     39.86 &      8.29 &  100.0/0.0/0.0 \\
    &     &     & 7  &    39.77 &      7.97 &   62.3/3.1/32.4 &     39.79 &      8.30 &   5.5/0.0/93.7 \\
    &     &     & 14 &    39.72 &      7.99 &    2.4/3.1/91.8 &     39.74 &      8.29 &   3.1/0.0/96.1 \\
    &     &     & 28 &    39.61 &      7.99 &    1.5/3.1/92.3 &     39.64 &      8.39 &   0.1/0.0/89.2 \\
    &     & 1.0 & 0  &    39.82 &      7.96 &    94.7/3.1/0.0 &     39.86 &      8.29 &  100.0/0.0/0.0 \\
    &     &     & 7  &    39.76 &      7.96 &    94.7/3.1/0.0 &     39.79 &      8.28 &   99.0/0.0/0.0 \\
    &     &     & 14 &    39.69 &      7.96 &    93.6/3.1/0.0 &     39.72 &      8.28 &   96.2/0.0/0.0 \\
    &     &     & 28 &    39.57 &      7.96 &   66.0/3.1/26.1 &     39.61 &      8.51 &  54.9/0.0/13.9 \\
\bottomrule
\end{tabular}}
\end{table}

Table~\ref{tab:cost-sens-analysis} investigates the impact of varying costs for supplemental screenings on the QALYs and LBCMR values, as well as their impact on the different types of screenings used. 
In designing this experiment, we considered the results in Table~\ref{tab:simulated-actions-taken}, which revealed that, when subjected to a budget constraint, screening policies typically utilize only the mammography action since the supplemental tests tend to be prohibitively expensive. 
Even in the case of QALY maximization, where the risk of false positives must be taken into account, it is observed that using mammography alone is more preferable, which has a higher false positive rate than mammography plus ultrasound, both due to the disutility of supplemental ultrasound (1 day vs 0.5 days for mammography alone) and the higher sensitivity of the mammography action. 
Accordingly, in this sensitivity analysis, the cost of mammography is considered to be constant. 
As all other screenings involve mammography, the cost of the supplemental screening tests cannot be below the cost of a mammography.

% Cost sensitivity analysis
\begin{table}[htbp]
% \caption{Cost sensitivity analysis ($\parBudgetLim=\$850$)}
\caption{Sensitivity analysis results for cost values considering average budget limit ($\parBudgetLim=\$850$; M/R/U: screening action percentages; LBCMR values are in \%).}
\label{tab:cost-sens-analysis}
\resizebox{\textwidth}{!}{
\begin{tabular}{lrrrrrrrrrrrr}
\toprule
      &     &      &     & \multicolumn{3}{c}{Min. LBCMR} & \multicolumn{3}{c}{Weighted} & \multicolumn{3}{c}{Max. QALYs} \\
      \cmidrule(lr){5-7}\cmidrule(lr){8-10}\cmidrule(lr){11-13}
Patient & M & M+R & M+U & QALYs &  LBCMR & M/R/U & QALYs & LBCMR & M/R/U & QALYs &  LBCMR & M/R/U \\
\midrule
AR & 134 & 134  & 134 &     40.16 &  4.44 &  0.0/14.2/1.9 &    40.22 &  4.49 &  14.9/0.8/0.0 &     40.24 &  4.71 &  14.5/0.0/3.4 \\
      &     & 161  & 161 &     40.20 &  4.46 &  14.1/2.3/0.1 &    40.22 &  4.49 &  15.4/0.2/0.0 &     40.24 &  4.70 &  15.8/0.0/0.0 \\
      &     & 188  & 188 &     40.21 &  4.46 &  16.2/0.4/0.0 &    40.22 &  4.49 &  15.5/0.0/0.0 &     40.24 &  4.70 &  16.1/0.0/0.0 \\
      &     & 538  & 161 &     40.21 &  4.46 &  16.8/0.0/0.1 &    40.22 &  4.49 &  15.8/0.0/0.0 &     40.24 &  4.70 &  15.9/0.0/0.0 \\
      &     & 943  & 188 &     40.21 &  4.46 &  16.8/0.0/0.0 &    40.22 &  4.49 &  15.5/0.0/0.0 &     40.24 &  4.70 &  16.2/0.0/0.0 \\
      \cmidrule(lr){3-13}
      &     & 1752 & 243 &     40.21 &  4.46 &  14.9/0.0/0.0 &    40.22 &  4.49 &  15.6/0.0/0.0 &     40.24 &  4.70 &  16.0/0.0/0.0 \\
\midrule
HR & 134 & 134  & 134 &     39.44 &  8.67 &  1.8/16.1/1.6 &    39.51 &  8.74 &  13.5/5.8/0.1 &     39.54 &  9.09 &  12.0/4.0/0.0 \\
      &     & 161  & 161 &     39.47 &  8.73 &  10.6/5.8/0.2 &    39.51 &  8.78 &  12.9/3.9/0.0 &     39.54 &  9.12 &  12.9/3.1/0.0 \\
      &     & 188  & 188 &     39.47 &  8.75 &  20.0/0.2/0.0 &    39.51 &  8.80 &  17.2/0.0/0.0 &     39.54 &  9.14 &  16.0/0.0/0.0 \\
      &     & 538  & 161 &     39.47 &  8.75 &  20.1/0.0/0.2 &    39.51 &  8.80 &  17.2/0.0/0.0 &     39.54 &  9.14 &  16.0/0.0/0.0 \\
      &     & 943  & 188 &     39.47 &  8.75 &  20.4/0.0/0.0 &    39.51 &  8.80 &  17.2/0.0/0.0 &     39.54 &  9.14 &  16.0/0.0/0.0 \\
      \cmidrule(lr){3-13}
      &     & 1752 & 243 &     39.47 &  8.75 &  20.4/0.0/0.0 &    39.51 &  8.80 &  17.2/0.0/0.0 &     39.54 &  9.14 &  16.0/0.0/0.0 \\
\bottomrule
\end{tabular}}
\end{table}

The results in Table~\ref{tab:cost-sens-analysis} show that even when the supplemental tests are subsidized greatly, they are still not overwhelmingly utilized.
This is particularly the case for the weighted and QALY maximization objectives, as the disutilities of these supplemental screenings can often outweigh their benefits. 
For AR patients, supplemental screenings are effectively unused when their cost exceeds 1/4 of the cost of an ultrasound. 
We observe that when supplemental screenings are free, AR patients still tend to use the mammography action as opposed to using the ultrasound action. 
This is due to the higher disutility of the supplemental test, as well as the fact that the supplemental test is assumed to have a lower sensitivity than mammography alone. 
For HR patients, the risk of developing cancer is much greater. 
Accordingly, the higher specificity of the ultrasound is not as useful as the higher sensitivity of the other tests. 
However, as MRI has a much higher disutility compared to mammography (2.5 days versus half a day) the mammography action is still used for the majority of screenings for HR patients, even when supplemental MRI is cheap or free.

The results of the sensitivity analysis show that, in general, supplemental screenings are too expensive, and have too high of a disutility to justify their usage in a constrained setting. 
This is particularly the case for supplemental MRIs, which raise the price of screening by an order of magnitude. 
Supplemental screenings would be vastly more beneficial if they were scheduled in such a way as to minimize the overhead disutility that they incur, for example by performing them in the same appointment, and by subsidizing their cost.

% \clearpage

%%%%%%%%%%%%%%%%%%%%%%%%%%%%%%%%%%%%%%%%%%%%%%%%%%%%%%%%%%%%%%%%%%%%%%%%%%%%%%%%%%%%%%%%%%%%%%%%
%%%%%%%%%%%%%%%%%%%%%%%%%%%%%%%%%%%%%%%%%%%%%%%%%%%%%%%%%%%%%%%%%%%%%%%%%%%%%%%%%%%%%%%%%%%%%%%%
\section{Conclusion and Future Work}\label{sec:Conclusion}
%%%%%%%%%%%%%%%%%%%%%%%%%%%%%%%%%%%%%%%%%%%%%%%%%%%%%%%%%%%%%%%%%%%%%%%%%%%%%%%%%%%%%%%%%%%%%%%%
%%%%%%%%%%%%%%%%%%%%%%%%%%%%%%%%%%%%%%%%%%%%%%%%%%%%%%%%%%%%%%%%%%%%%%%%%%%%%%%%%%%%%%%%%%%%%%%%
In this study, we present a multi-objective constrained partially observable Markov decision process model for the breast cancer screening problem, which involves three partially observable health states and three screening modalities.
We employ grid-based approximation methods to formulate an approximate mixed-integer linear programming model, which enables solving the multi-objective CPOMDP model using a weighted combination of the objectives, as well as the $\epsilon$-constraint method. 
We investigate the impact of supplemental tests on QALYs and LBCMR in a constrained environment and show that many policies opt to use mammography alone due to the additional cost and disutility of the supplemental tests. 
We also compare budget-constrained policies with several rule-based ones and show that many constrained policies can achieve similar results to rule-based screenings while incurring a considerably reduced cost. 
We then provide policy evaluations for a patient subject to several assumptions in the decision process using various policies. 
These policy evaluations reveal that, especially for high-risk patients, prescribed policies can vary considerably between the LBCMR minimization and QALY maximization objectives. 
We conduct a sensitivity analysis on the screening disutility values and the cost of screening test and identify conditions under which the addition of supplemental tests becomes favourable.
We improve the solvability of this approximate model by eliminating unuseful grid points from the formulation, which leads to substantial gains in terms of CPU run times.
We conduct detailed experiments to assess the trade-off between the QALY and LBCMR optimization objectives.
Being able to quantify such trade-offs can help health policymakers in designing more informed personalized screening policies.
% Because solving this approximate model is computationally challenging, we 

Our research has certain limitations which can be, in part, addressed in future studies.
First, because of the computational challenges, our baseline POMDP model is limited to only three health states.
Patients' adherence behaviours as well as dynamic risk factors such as breast density can be incorporated into the state space to have a more representative model for the breast cancer screening problem.
Secondly, imposing budget limits for personalized screening involves setting constraints over the lifetime of the patients.
This approach might limit the adoption of these policies in practice because it would be difficult to keep track of the patients over their lifetimes.
On the other hand, our models also guide the policymakers in terms of which screening ages are more important when there are limited available resources.
Such insights can be particularly useful in settings where governments or institutions would need to carefully allocate screening resources to maximize early detection of cancer.
% Thirdly, we assume that supplemental screening actions are handled independently of mammograms. Accordingly, the costs and the disutilities of screenings are summed together for supplemental screening actions. However, it is conceivable that in some cases, supplemental tests could be performed in the same appointment as a mammogram. In such a case, both the cost and disutility of the combined screening could be reduced. This idea is addressed in part through a sensitivity analysis, but a more thorough exploration into such combined screenings would be extremely beneficial.
We also note that emerging screening modalities (e.g., tomosynthesis) might be incorporated into our models.
Similarly, a sensitivity analysis around the screening test performance values can be used to understand at what performance levels the supplemental screenings become more beneficial.
Lastly, the budget constraints imposed in our formulations only account for screening and disregard various other costs including those of diagnostic tests and treatment.
A more complex CPOMDP formulation can be developed in the future to account for these costs.

\section*{Acknowledgements}

This research was enabled in part by support provided by the Digital Research Alliance of Canada (\url{alliancecan.ca}).
This work is also in part supported by NSERC Discovery Grant RGPIN-2019-05588.

\bibliographystyle{spbasic}
\bibliography{main_bcs_springer}

\appendix

\clearpage

\section{Grid construction}
\label{sec:appendix-grid-construction}
%%%%%%%%%%%%%%%%%%%%%%%%%%%%%%%%%%%%%%%%%%%%%%%%%%%%%%%%%%%%%%%%%%%
%%%%%%%%%%%%%%%%%%%%%%%%%%%%%%%%%%%%%%%%%%%%%%%%%%%%%%%%%%%%%%%%%%%

The belief simplex $\beSimp$ is the continuous, infinite set of possible belief values for a state space $\setState$.
In this study, finite subsets of the belief simplex are used as approximations.
Such approximate grid sets are generated using a variable-resolution uniform grid construction approach, a modification of the commonly used fixed-resolution grid construction approach.

%%%%%%%%%%%%%%%%%%%%%%%%%%%%%%%%%%%%%%%%%%%%%%%%%%%%%%%%%%%%%%%%%%%
\subsection{Fixed-resolution grid construction approach}
%%%%%%%%%%%%%%%%%%%%%%%%%%%%%%%%%%%%%%%%%%%%%%%%%%%%%%%%%%%%%%%%%%%
\label{sec:appendix-fixed-resolution-grid-construction}
In the fixed-resolution grid approach, the grid construction process involves sampling beliefs at equidistant intervals along the dimensions of the state space. The intervals are set according to the resolution parameter, $\resoVal$, where in each dimension the allowed values must be an integer multiple of $\resoVal^{-1}$. These values are further constrained by the fact that for any state, the belief can be no less than zero and no greater than one, as all belief states represent a probability distribution. Accordingly, the allowed values for the belief component $\grPt_\indState$ of the grid point $\vectGrPt$ are given by
\begin{subequations}
\begin{equation}
    \label{apx-eq:grid-component-allowed-values}
    \setAllowedBeliefComponents(\resoVal)=\left\{1, \frac{\resoVal-1}{\resoVal}, \frac{\resoVal-2}{\resoVal}, \ldots, \frac{2}{\resoVal}, \frac{1}{\resoVal},0 \right\}
\end{equation}
As $\vectGrPt$ is a probability distribution, it follows that 
\begin{equation}
\label{apx-eq:grid-component-normalization}
    \sum_\indState \vectGrPt_\indState = 1
\end{equation}
\end{subequations}
Together, Equations~\ref{apx-eq:grid-component-allowed-values} and~\ref{apx-eq:grid-component-normalization} provide the full set of constraints on the beliefs in a fixed-resolution grid set. The grid set can be generated in two steps: take the Cartesian product of $\setAllowedBeliefComponents$ with itself $|\setState|$ times and then filter off all elements that do not satisfy Equation~\ref{apx-eq:grid-component-normalization}.

\paragraph{Fixed-resolution grid set example}

In this example, a fixed-resolution grid set is generated for a model with two states and a resolution value of $\resoVal=2$.
According to Equation~\ref{apx-eq:grid-component-allowed-values}, the allowed values for each component of the approximate grids are given by
\begin{subequations}
\begin{equation}
    \setAllowedBeliefComponents(2)=\left\{1, \frac{1}{2}, 0\right\}
\end{equation}
Let $\mathcal{C}$ denote the Cartesian product of $\setAllowedBeliefComponents$ with itself.
In this case, there are 2 states so $\mathcal{C}$ is given by the Cartesian product of $\setAllowedBeliefComponents$ with itself twice.
Thus
\begin{align}
    \mathcal{C}&= \setAllowedBeliefComponents(2) \times \setAllowedBeliefComponents(2)\\
    \mathcal{C}&=\left\{[1, 1], [1, \tfrac{1}{2}], [1, 0], [\tfrac{1}{2}, 1], [\tfrac{1}{2}, \tfrac{1}{2}], [\tfrac{1}{2}, 0], [0, 1], [0, \tfrac{1}{2}], [0, 0]\right\}
\end{align}
$\setGrid$ is then given by all of the elements of $\mathcal{C}$ that satisfy Equation~\ref{apx-eq:grid-component-normalization}, giving
\begin{equation}
    \setGrid=\left\{[1, 0], [\tfrac{1}{2}, \tfrac{1}{2}], [0, 1]\right\}
\end{equation}
\end{subequations}

%%%%%%%%%%%%%%%%%%%%%%%%%%%%%%%%%%%%%%%%%%%%%%%%%%%%%%%%%%%%%%%%%%%
\subsection{Variable resolution uniform grid construction approach}
%%%%%%%%%%%%%%%%%%%%%%%%%%%%%%%%%%%%%%%%%%%%%%%%%%%%%%%%%%%%%%%%%%%
\label{sec:appendix-variable-resolution-grid-construction}
When approximating the belief simplex, the approximation tends to improve as the resolution value $\resoVal$, and therefore the number of grids, is increased. 
However, the number of grids has a considerable impact on run times.
Accordingly, the grids in the grid set must be chosen carefully to ensure that the best approximation is achieved while run times remain feasible.
\citet{Cevik2018} and~\citet{sandikci2018screening} found that, in the breast cancer screening problem, the majority of patients are much closer to being cancer free (i.e., $\vectBePt=[1, 0, 0]$) than to any of the cancer states, and found that a variable resolution uniform grid construction approach could be beneficial for this problem.
A variable resolution uniform grid set is a grid set $\setGrid$ that uses grids from different $\setGrid_\resoVal$ to approximate different regions of the belief simplex.

In this research, the belief simplex is divided into three regions based on the probability that a patient is healthy.
A given threshold vector $\vectResoThreshold=[\resoThreshold_1, \resoThreshold_2, \resoThreshold_3]$ divides the belief simplex into the regions $[1,~\resoThreshold_1)$, $[\resoThreshold_1,~\resoThreshold_2)$, and $[\resoThreshold_2,~\resoThreshold_3]$.
To ensure the entire belief simplex is accounted for, $\resoThreshold_3$ must be equal to zero.
These three regions define thresholds on the probability that a patient is healthy, $\bePt_0$.
For each region, a different grid set is used to select grid points from.
These grid sets are given by $\vectResoVal=[\resoVal_1, \resoVal_2, \resoVal_3]$, and accordingly the first, second, and third regions select their grids from $\setGrid_{\resoVal_1}$, $\setGrid_{\resoVal_2}$, and $\setGrid_{\resoVal_3}$, respectively.
Specifically, in the region given by $[1,~\resoThreshold_1)$, all grid points from $\setGrid_{\resoVal_1}$ that satisfy $\bePt_0 \in [1,~\resoThreshold_1)$ are added to the final approximate grid set, $\setGrid$.
Likewise, the second region draws grids from $\setGrid_{\resoVal_2}$ that satisfy $\bePt_0 \in [\resoThreshold_1,~\resoThreshold_2)$.
The third region draws from $\setGrid_{\resoVal_3}$ for $\bePt_0 \in [\resoThreshold_2,~\resoThreshold_3]$.
Unlike the preceding regions, the final region has inclusive boundaries on both ends (i.e., in the third region, belief points where $\bePt_0=\resoThreshold_3$ are included).
The grid set $\setGrid$ is given by
\begin{equation}
    \setGrid = \left(\bigcup_{\indResoVal=1}^{|\vectResoThreshold| - 1} \{\vectBePt\mid \vectBePt\in \setGrid_{\resoVal_{\indResoVal}},~\bePt_0\in[\resoThreshold_{\indResoVal-1},~\resoThreshold_{\indResoVal})\}\right) \cup \{\vectBePt\mid \vectBePt\in \setGrid_{\resoVal_{|\vectResoThreshold|}},~\bePt_0\in[\resoThreshold_{|\vectResoThreshold|-1},~\resoThreshold_{|\vectResoThreshold|}]\}
\end{equation}
Where $|\vectResoThreshold|$ gives the size of the resolution threshold vector, $\vectResoThreshold$, and $\resoThreshold_0=1$.
Algorithm~\ref{alg:variable-resolution-grid-construction} gives the procedure for constructing the grid set.
\begin{algorithm}[!ht]

\caption{ Variable resolution uniform grid-construction approach.}
\label{alg:variable-resolution-grid-construction}
%\\ \textbf{Output}: $\{\alpha\}$
\begin{algorithmic}[1] %[1] enables line numbers
\Procedure{\texttt{Get\_Grid\_Set}}{$\vectResoVal$, $\vectResoThreshold$}
    \State{$\setGrid\gets \emptyset$}
    \For{$\indResoVal \in 1\ldots |\vectResoThreshold|-1$}
        \State{$\setGrid \gets \setGrid \cup \{\vectBePt\mid \vectBePt\in \setGrid_{\resoVal_\indResoVal}, \bePt_0\in[\resoThreshold_{\xi-1},~\resoThreshold_{\indResoVal})\}$}
        \Comment{Here, $\resoThreshold_0=1$.}
    \EndFor
    \State{$\setGrid \gets \setGrid \cup \{\vectBePt\mid \vectBePt\in \setGrid_{\resoVal_|\vectResoThreshold|}, \bePt_0\in[\resoThreshold_{|\vectResoThreshold|-1},~\resoThreshold_{|\vectResoThreshold|}]\}$}\Comment{Inclusive right boundary.}
    \State{\Return{$\setGrid$}}
\EndProcedure
\end{algorithmic}
\end{algorithm}

\paragraph{Variable resolution uniform grid construction example}
Consider the breast cancer screening problem, which uses three states, and a resolution vector of $\vectResoVal=[3,~2]$ and a threshold vector of $\vectResoThreshold=[\tfrac{1}{2},~0]$. 
The fixed resolution grid sets needed for this problem are $\setGrid_3$ and $\setGrid_2$, which are
\begin{subequations}
\begin{align}
    \setGrid_3&=
        \resizebox{0.8\textwidth}{!}{$\{[1, 0, 0], [\tfrac{2}{3}, \tfrac{1}{3}, 0], [\tfrac{2}{3}, 0, \tfrac{1}{3}], [\tfrac{1}{3}, \tfrac{2}{3}, 0], [\tfrac{1}{3}, \tfrac{1}{3}, \tfrac{1}{3}], [\tfrac{1}{3}, 0, \tfrac{2}{3}], [0, 1, 0], [0, \tfrac{2}{3}, \tfrac{1}{3}], [0, \tfrac{1}{3}, \tfrac{2}{3}], [0, 0, 1]\}$}\\
    \setGrid_2&=\{[1, 0, 0], [\tfrac{1}{2}, \tfrac{1}{2}, 0], [\tfrac{1}{2}, 0, \tfrac{1}{2}], [0, 1, 0], [0, \tfrac{1}{2}, \tfrac{1}{2}], [0, 0, 1]\}
\end{align}
Following along with Algorithm~\ref{alg:variable-resolution-grid-construction}, $\setGrid$ is initialized as the empty set.
In the first iteration over the for loop, $\indResoVal=1$. 
Here, the threshold region is $[\resoThreshold_{0},~\resoThreshold_{1}]=[1,~\tfrac{1}{2}])$.
The selected grid points have $\bePt_0\in[1,~\tfrac{1}{2})$ and, for $\setGrid_3$, this is all regions where the probability of being healthy is either $1$ or $\tfrac{2}{3}$.
Notably, if a belief point from $\setGrid_3$ had $\bePt_0=\tfrac{1}{2}$, this belief point would not be included as the right boundary is exclusive in the for loop.
After this step, the grid set is
\begin{equation}
    \setGrid=\{[1, 0, 0], [\tfrac{2}{3}, \tfrac{1}{3}, 0], [\tfrac{2}{3}, 0, \tfrac{1}{3}]\}
\end{equation}
This concludes the for loop in Algorithm~\ref{alg:variable-resolution-grid-construction} for this example.
The next threshold region is $[\tfrac{1}{2}, 0]$ corresponding to the fixed resolution grid $\setGrid_0$ and, unlike the regions in the for loop, the right boundary is inclusive, meaning that states containing $\bePt_0=0$ in $\setGrid_2$ are included in the final grid set.
All points except for $[1,~0~,0]$ lie within this threshold region.
Thus
\begin{align}
    \setGrid&=\setGrid \cup \{[\tfrac{1}{2}, \tfrac{1}{2}, 0], [\tfrac{1}{2}, 0, \tfrac{1}{2}], [0, 1, 0], [0, \tfrac{1}{2}, \tfrac{1}{2}], [0, 0, 1]\}\\
    \setGrid&=\{[[1, 0, 0], [\tfrac{2}{3}, \tfrac{1}{3}, 0], [\tfrac{2}{3}, 0, \tfrac{1}{3}], [\tfrac{1}{2}, \tfrac{1}{2}, 0], [\tfrac{1}{2}, 0, \tfrac{1}{2}], [0, 1, 0], [0, \tfrac{1}{2}, \tfrac{1}{2}], [0, 0, 1]\}
\end{align}
\end{subequations}

\end{document}